%% file: main.tex
\documentclass[final]{elsarticle}

\makeatletter
\def\ps@pprintTitle{%
 \let\@oddhead\@empty
 \let\@evenhead\@empty
 \def\@oddfoot{}%
 \let\@evenfoot\@oddfoot}
\makeatother

\usepackage{lineno}
\modulolinenumbers[5]

\usepackage{hyperref}

\usepackage{amsmath}
\usepackage{amssymb,amsthm}

\usepackage{stmaryrd}
\usepackage[noend]{algpseudocode}
\usepackage[dvipsnames,table,xcdraw]{xcolor}
\usepackage{xspace}
\usepackage[frozencache=true,cachedir=minted-cache,newfloat]{minted}
\usepackage{caption}
\usepackage{listings}
\usepackage{nicefrac}
\usepackage{multirow}
\usepackage{textcomp} 
\usepackage{textgreek}
\usepackage{changepage}
\usepackage{cancel}
\usepackage{subcaption}
\usepackage{tikz}
\usetikzlibrary{arrows, calc,shapes.geometric,shapes.multipart,positioning,tikzmark}
\usetikzlibrary{arrows,shapes,trees}
\usepackage{enumitem}
\usepackage{multicol}
\usepackage{bm}

\usepackage[T1]{fontenc}
\usepackage{txfonts}
\usepackage{mdframed}
\usepackage{scrextend}
\usepackage{blindtext}
\usepackage[section]{placeins}

\usepackage{xassoccnt}
\usepackage{chngcntr}
\usepackage{thm-restate}
\usepackage{xpatch}
\usepackage{fixmath}
\usepackage{soul}

\usepackage[htt]{hyphenat}

\usepackage[linesnumbered,ruled,vlined,algosection]{algorithm2e}
\DontPrintSemicolon
\SetKwInput{KwInput}{Input}                
\SetKwInput{KwOutput}{Output}              

\SetCommentSty{mycommfont}

\graphicspath{{figures/}}

\input{macros}

\numberwithin{equation}{section}

\setlist[itemize]{itemsep=0pt}
\setlist[enumerate]{itemsep=0pt}

\begin{document}

\setcounter{page}{1}
\renewcommand{\thepage}{\roman{page}}

\setcounter{section}{0}
\renewcommand{\thesection}{R\arabic{section}}

\renewcommand{\theequation}{R\arabic{section}.\arabic{equation}}


\setcounter{page}{1}
\renewcommand{\thepage}{\arabic{page}}

\setcounter{section}{0}
\renewcommand{\thesection}{\arabic{section}}
\renewcommand{\theequation}{\arabic{section}.\arabic{equation}}

\begin{frontmatter}

    \title{Declarative Probabilistic Logic Programming\\in Discrete-Continuous Domains}

    \author[orebro]{Pedro Zuidberg Dos Martires}
    \author[orebro,kuleuven,leuvenai]{Luc De Raedt}
    \author[kuleuven,leuvenai]{Angelika Kimmig}

    \address[orebro]{Centre for Applied Autonomous Sensor Systems, Örebro University, Sweden}
    \address[kuleuven]{Department of Computer Science, KU Leuven, Belgium}
    \address[leuvenai]{Leuven.AI, Belgium}


    \begin{abstract}
        Over the past three decades, the logic programming paradigm has been successfully expanded
        to support probabilistic modeling, inference and learning. The resulting paradigm
        of probabilistic logic programming (PLP) and its programming languages owes much of its success to a declarative semantics,
        the so-called distribution semantics. However, the distribution semantics is limited to discrete random variables only.
        While PLP has been extended in various ways for supporting hybrid, that is, mixed discrete and continuous
        random variables, we are still lacking a declarative semantics for hybrid PLP that not only generalizes
        the distribution semantics and the modeling language but also the standard inference algorithm
        that is based on knowledge compilation.
        We contribute the {\em measure semantics} together with the hybrid PLP language
        \dcproblogsty (where DC stands for distributional clauses) and its inference engine {\em infinitesimal algebraic likelihood weighting} (IALW).
        These have the original distribution semantics, standard PLP languages
        such as \problogsty, and standard inference engines for PLP based on knowledge compilation as special cases.
        Thus, we generalize the state of the art of PLP towards hybrid PLP in three different aspects: semantics, language and inference.
        Furthermore, IALW is the first inference algorithm for hybrid probabilistic programming based on knowledge compilation.
    \end{abstract}


\end{frontmatter}


\input{files_main/introduction}

\input{files_main/panorama}

\input{files_main/semantics}
\input{files_main/dcproblog}

\input{files_main/tasks}

\input{files_main/dc2smt}

\input{files_main/evaluating}

\input{files_main/related_work}

\input{files_main/conclusions}

\section*{Acknowledgement}

This research received funding from the Wallenberg AI, Autonomous Systems and Software Program (WASP) of the Knut and Alice Wallenberg Foundation, the Flemish Government (AI Research Program),
the KU Leuven Research Fund, the European Research Council (ERC) under the European Union’s Horizon 2020 research and innovation programme (grant agreement No [694980] SYNTH: Synthesising Inductive Data Models), and
the Research Foundation - Flanders.

\bibliography{references}

\newpage
\clearpage

\appendix

\setcounter{section}{0}
\renewcommand{\thesection}{\Alph{section}}
\renewcommand{\theequation}{\Alph{section}.\arabic{equation}}
\renewcommand{\thetheorem}{\Alph{section}.\arabic{theorem}}
\renewcommand{\thedefinition}{\Alph{section}.\arabic{definition}}

\input{files_appendix2/lp_new}

\input{files_appendix2/app_table}

\input{files_appendix2/semantics_proofs}

\input{files_appendix2/syntactic_sugar_semantics.tex}

\input{files_appendix2/detailed_comp_to_GutmanNittiDC}

\input{files_appendix2/transformation_proofs}

\

\end{document}

%% file: macros.tex
\newcommand{\pedroskip}[1]{{\color{black}(skipped text)}}

\newcommand{\lucskip}[1]{{\color{black}(skipped text)}}

\newcommand{\angelikaskip}[1]{{\color{black}(skipped text)}}

\newcommand{\new}[1]{{\color{black}{#1}}}

\newcommand{\fixed}[1]{{\color{black}{#1}}}

\newtheorem{definition}[algocf]{Definition}
\newtheorem{example}[algocf]{Example}
\newtheorem{lemma}[algocf]{Lemma}
\newtheorem{theorem}[algocf]{Theorem}
\newtheorem{proposition}[algocf]{Proposition}

\DeclareCoupledCountersGroup{allcounter}
\DeclareCoupledCounters[name=allcounter]{algocf,figure}

\counterwithin{figure}{section}

\setminted[problog.py:ProbLogLexer -x]{
        fontshape=n,
        xleftmargin=20pt,
        linenos=false,
        escapeinside=@@,
        mathescape=true,
        breaklines=true,
}

\newcounter{lstlabelcounter}

\newminted[problog]{problog.py:ProbLogLexer -x}{}
\newcommand{\probloginline}[1]{\mintinline{problog.py:ProbLogLexer -x}{#1}}
\newcommand{\mathprobloginline}[1]{\text{\mintinline{problog.py:ProbLogLexer -x}{#1}}}

\SetupFloatingEnvironment{listing}{name=Listing}

\setminted[text]{
        fontshape=n,
        xleftmargin=20pt,
        linenos=false,
        mathescape=true,
        escapeinside=@@,
        breaklines=true,
}
\newminted[dcplp]{text}{}
\newcommand{\dcplpinline}[1]{\mintinline[breaklines, breakafter=,]{text}{#1}}

\AtBeginEnvironment{minted}{}
\AtBeginEnvironment{problog}{}
\AtBeginEnvironment{problog*}{}

\newcommand{\functor}[2]{\probloginline{#1}$/{#2}$\xspace}
\newcommand{\mathfunctor}[2]{${#1}/{#2}$\xspace}
\newcommand{\predicate}[2]{\probloginline{#1}$/{#2}$\xspace}
\newcommand{\mathpredicate}[2]{${#1}/{#2}$\xspace}

\DeclareMathOperator*{\lpif}{\mathtt{{:}{\---}}}
\DeclareMathOperator*{\prob}{\mathtt{{:}{:}}}

\newcommand{\herbrandbase}{\ensuremath{\mathcal{A}}}
\newcommand{\evidenceset}{\ensuremath{\mathcal{E}}}
\newcommand{\queryset}{\ensuremath{\mathcal{Q}}}

\newcommand{\ive}[1]{\llbracket#1\rrbracket}

\newcommand{\varset}[1]{\ensuremath{\mathbf{#1}}}

\newcommand{\dcpprogram}{\ensuremath{\mathcal{P}}\xspace}

\newcommand{\boolval}{\ensuremath{b}\xspace}

\newcommand{\probabilitymeasure}{\ensuremath{P}\xspace}

\newcommand{\distributionfunctors}{\ensuremath{{\Delta}}\xspace}
\newcommand{\arithmeticfunctors}{\ensuremath{{\Phi}}\xspace}
\newcommand{\comparisonpredicates}{\ensuremath{{\Pi}}\xspace}

\newcommand{\samplespace}{\ensuremath{\Omega}\xspace}
\newcommand{\samplefunction}{\ensuremath{\omega}\xspace}

\newcommand{\distdb}{\ensuremath{\mathcal{D}}\xspace}
\newcommand{\comparisonfacts}{\ensuremath{\mathcal{F}}\xspace}
\newcommand{\measurecomparisonfacts}{\ensuremath{\probabilitymeasure_\comparisonfacts}\xspace}

\newcommand{\sigmaalgebra}{\ensuremath{ \Sigma}\xspace}

\newcommand{\probabilityspace}{\ensuremath{ \mathbb{P}}\xspace}

\newcommand{\program}{\ensuremath{{\mathcal{P}}}\xspace}
\newcommand{\adfreeprogram}{\ensuremath{{\program^{*}}}\xspace}
\newcommand{\dfprogram}{\ensuremath{{\program^{DF}}}\xspace}
\newcommand{\dfadfreeprogram}{\ensuremath{{\program^{DF,*}}}\xspace}
\newcommand{\logicprogram}{\ensuremath{{\mathcal{R}}}\xspace}

\newcommand{\dclauses}{\ensuremath{\mathcal{C}}\xspace}
\newcommand{\headsdc}{\ensuremath{\mathcal{H}}\xspace}

\newcommand{\randomtermset}{\ensuremath{\mathcal{T}}\xspace}
\newcommand{\randomvariableset}{\ensuremath{\mathcal{V}}\xspace}

\newcommand{\contextfunc}{\ensuremath{K}\xspace}

\newcommand{\ialw}{IALW\xspace}

\newcommand{\dcplpsty}{DC-PLP\xspace}
\newcommand{\problogsty}{ProbLog\xspace}
\newcommand{\dcsty}{Distributional Clauses\xspace}
\newcommand{\dfplpsty}{DF-PLP\xspace}
\newcommand{\dcproblogsty}{DC-ProbLog\xspace}
\newcommand{\blogsty}{BLOG\xspace}

\newcommand{\pyrosty}{Pyro\xspace}

\newcommand{\extendedprismsty}{Extended PRISM\xspace}

\newcommand{\problogsys}{ProbLog2\xspace}

\DeclareMathOperator*{\E}{\mathbb{E}}

\newcommand{\differential}{\mathop{}\!{d}}

\newcommand*\footcircled[1]{\tikz[baseline=(char.base)]{
                \node[shape=circle,draw,inner sep=1pt] (char) {\footnotesize {#1}};}}

\newcommand{\ie}{i.e.\xspace}
\newcommand{\eg}{e.g.\xspace}
\newcommand{\cf}{cf.\xspace}
\newcommand{\etc}{etc.\xspace}

%% file: files_main/introduction.tex
\section{Introduction}
\label{sec:introduction}

Probabilistic logic programming (PLP) is at the crossroads of two parallel developments in artificial intelligence and machine learning.
On the one hand, there are the probabilistic programming languages with built-in support for machine learning. These languages can be used to represent very expressive -- Turing equivalent
-- probabilistic models, and they provide primitives for inference and learning.
On the other hand, there is the longstanding open question for integrating the two main frameworks
for reasoning, that is logic and probability, within a common framework \citep{russell2015unifying,deraedt:mc16}.
Probabilistic logic programming~\citep{de2015probabilistic,riguzzi2018foundations} fits both paradigms and goes back to at least the early 90s
with seminal works by \citet{sato1995statistical} and \citet{poole1993probabilistic}. 
\citeauthor{poole1993probabilistic} introduced ICL, the Independent Choice Logic, an elegant
extension of the Prolog programming language, and \citeauthor{sato1995statistical} introduced
the {\em distribution semantics} for probabilistic logic programs in conjunction with a learning algorithm based on expectation maximization (EM).
The PRISM language \citep{sato1995statistical}, which utilizes the distribution semantics and the EM learning algorithm constitutes, to the best of the authors' knowledge, the very first probabilistic programming language with support for machine learning.

Today, there is a plethora of probabilistic logic programming languages, most of which are based
on extensions of the ideas by \citeauthor{sato1995statistical} and~\citeauthor{poole1993probabilistic} ~\citep{sato1997prism,kersting2000bayesian,vennekens2004logic,de2007problog}. However, the vast majority of them is restricted to discrete, and more precisely finite categorical, random variables. 
When merging logic with probability, the restriction to discrete random variables
is natural and allowed Sato to elegantly extend the logic program semantics
into the celebrated distribution semantics.

\begin{example}[Probabilistic Logic Program]
  \label{example:intro}

  \new{
  Consider the probabilistic logic program below (written in \problogsty syntax~\citep{fierens2015inference}), where we model the behavior of two machines. We first state that there are two machines (Line~\ref{line:exintro:1}). Subsequently, we say that the temperature has a probability of $0.8$ to be low (Line \ref{line:exintro:temp}) and that the cooling of the machines works with probability $0.99$ and $0.95$ respectively (Lines \ref{line:exintro:cool1} and \ref{line:exintro:cool2}) These labeled facts are called \textit{probabilistic facts}. We also model that the machines themselves work: either if the cooling is working (Line~\ref{line:exintro:work1}) or if the temperature is low (Line~\ref{line:exintro:work2}).
  }
  \begin{problog*}{linenos}
machine(1). machine(2). @\label{line:exintro:1}@
0.8::temperature(low). @\label{line:exintro:temp}@
0.99::cooling(1).@\label{line:exintro:cool1}@
0.95::cooling(2).@\label{line:exintro:cool2}@

works(N):- machine(N), cooling(N). @\label{line:exintro:work1}@
works(N):- machine(N), temperature(low). @\label{line:exintro:work2}@

    \end{problog*}
    \new{
We can now, for instance, ask for the conditional probability of the first machine working given that the second one works:
    $$
    P(\mathprobloginline{works(1)}=\top \mid \mathprobloginline{works(2)}=\bot).
    $$
The (exact) inference algorithm currently implemented in \problogsys~\citep{fierens2015inference,dries2015problog2} then returns as answer the probability $\approx 0.998$.
}
\end{example}

While \citeauthor{sato1995statistical}'s extension of logic programming to the probabilistic domain is elegant, it also imposes an important restriction to random variables with countable sample spaces. This raises the question of how to extend the distribution semantics towards hybrid, \ie discrete-continuous, random variables.

Defining the semantics of probabilistic programming language with support for random variables with infinite and possibly uncountable sample spaces is a much harder task. This can be observed when looking at the development of important imperative and functional probabilistic programming languages~\citep{goodman2008church,mansinghka2014venture} that support  continuous random variables. 
These works initially focused on inference, typically using  a particular Monte Carlo approach, yielding  an operational or procedural semantics. It is only follow-up work that started to address a declarative semantics for such hybrid probabilistic programming languages.
~\citep{staton2016semantics,wu2018discrete}.

The PLP landscape has experienced similar struggles. First approaches for  hybrid PLP languages were achieved by restricting the language~\citep{gutmann2010extending,gutmann2011magic,islam2012inference} or via recourse to procedural semantics~\citep{nitti2016probabilistic}.  The key contributions of this paper are:

\begin{enumerate}[label={\bf C\arabic*}]
  \setcounter{enumi}{0}
\item We introduce the {\em measure semantics} for mixed discrete-continuous probabilistic logic programming.
Our {\em measure semantics} (based on measure theory) extends \citeauthor{sato1995statistical}'s distribution semantics and supports:
\begin{itemize}
    \item  \label{item:k1} a countably infinite number of random variables,
     \item a uniform treatment of discrete and continuous random variables,
      \item a clear separation between probabilistic dependencies and logical dependencies by extending the ideas of \citet{poole2010probabilistic} to the hybrid domain.
    \end{itemize}
\item \label{item:k3} We introduce \dcproblogsty, an expressive PLP language in the discrete-continuous domain,
which incorporates the{\em measure semantics}. 
 \dcproblogsty has standard discrete PLP, \eg ProbLog~\citep{fierens2015inference}, as a special case (unlike other hybrid PLP languages~\citep{gutmann2011magic,nitti2016probabilistic}).
    \item \label{item:k2}  We introduce a novel inference algorithm, {\em infinitesimal algebraic likelihood weighting} (IALW), for hybrid PLPs,
which extends the standard knowledge compilation approach used in PLP towards mixed discrete continuous distributions, and  which 
provides an operational semantics for hybrid PLP.
\end{enumerate}

In essence, our contributions ~\ref{item:k1} and ~\ref{item:k3} generalize both  Sato's distribution semantics and discrete PLP such that in the absence of random variables with infinite sample spaces we recover the \problogsty language and declarative semantics. It is noteworthy that our approach of disentangling probabilistic dependencies and logical ones, allows us to express more general distributions than state-of-the-art approaches such as~\citep{gutmann2011magic,nitti2016probabilistic,azzolini2021semantics}. 
Contribution \ref{item:k2} takes this generalization to the inference level: in the exclusive presence of finite random variables our IALW algorithm reduces to \problogsty's current inference algorithm~\citep{fierens2015inference}.

%% file: files_main/panorama.tex
\section{A Panoramic Overview}
\label{sec:panorama}

Before diving into the technical details of the paper we first give a high-level overview of the \dcproblogsty language. This will also serve us as roadmap to the remainder of the paper.  
We will first introduce, by example, the \dcproblogsty language (Section~\ref{sec:panorama_semantics}). The formal syntax and semantics of which are discussed in Section~\ref{sec:semantics} and Section~\ref{sec:dcproblog}.
In Section~\ref{sec:panorama_inference} we demonstrate how to perform probabilistic inference in \dcproblogsty by translating a queried \dcproblogsty program to an algebraic circuit~\citep{zuidbergdosmartires2019transforming}. 
Before giving the details of this transformation in Section~\ref{sec:dc2smt} and Section~\ref{sec:alw}, we define conditional probability queries on \dcproblogsty programs (Section~\ref{sec:inference-tasks}).
The paper ends with a discussion on related work (Section~\ref{sec:related}) and concluding remarks in Section~\ref{sec:conclusions}.

Throughout the paper, we assume that the reader is familiar with basic concepts from logic programming and probability theory. We provide, however, a brief refresher of basic logic programming concepts in Appendix~\ref{app:lp_new}. In Appendix~\ref{app:table} we give a tabular overview of notations used, and in the remaining sections of the appendix we give proofs to propositions and theorems or discuss in more detail some of the more subtle technical issues.

\subsection{Panorama of the Syntax and Semantics}
\label{sec:panorama_semantics}

\begin{example}
    \label{example:sweets_dc}
    A shop owner creates random bags of  sweets with two independent random binary properties (\emph{large} and \emph{balanced}).  He first picks the number of red sweets from a Poisson distribution whose parameter is 20 if the bag is large and 10 otherwise, and then the number of yellow sweets from a Poisson whose parameter is the number of red sweets if the bag is balanced and twice that number otherwise. His favorite type of bag contains more than 15 red sweets and no less than 5 yellow ones. We model this in \dcproblogsty as follows:
\begin{problog*}{linenos}
0.5::large.
0.5::balanced.

red ~ poisson(20) :- large.
red ~ poisson(10) :- not large.

yellow ~ poisson(red) :- balanced.
yellow ~ poisson(2*red) :- not balanced.

favorite :- red > 15, not yellow < 5.
\end{problog*}

In the first two lines we encounter {\em probabilistic facts}, a well-known modelling construct in discrete PLP languages (\eg~\citep{de2007problog}). Probabilistic facts, written as logical facts labeled with a probability, express Boolean random variables that are true with the probability specified by the label. For instance, \probloginline{0.5::large} expresses that \probloginline{large} is true with probability \probloginline{0.5} and false with probability \probloginline{1-0.5}.

In Lines 4 to 8, we use {\em distributional clauses} (DCs); introduced by~\citet{gutmann2011magic} into the PLP literature. DCs are of the syntactical form \probloginline{v~d:-b} and define random variables \probloginline{v} that are distributed according to the distribution \probloginline{d}, given that \probloginline{b} is true.
For example, Line 4 specifies that when \probloginline{large} is true, \probloginline{red} is distributed according to a Poisson distribution. 
We call the left-hand argument of a \predicate{~}{2} predicate in infix notation a {\em random term}. The random terms in the program above are \probloginline{red} and \probloginline{yellow}.

Note how random terms reappear in three distinct places in the \dcproblogsty program. First, we can use them as parameters to other distributions, \eg \probloginline{yellow ~ poisson(red)}. 
Second, we can use them within arithmetic expression, such as \probloginline{2*red} in Line 8.
Third, we can use them in comparison atoms (\probloginline{red>15}) in Line 10. The comparison atoms appear in the bodies of logical rules that express logical consequences of probabilistic event, for example having more than 15 red sweets and less than 5 yellow ones.
\end{example}

Probabilistic facts and distributional clauses are the main modelling constructs to define random variables in probabilistic logic programs.
As they are considered to be fundamental building blocks of a PLP language, the semantics of a language are defined in function of these syntactical constructs (\cf.~\citep{fierens2015inference,gutmann2011magic}). 
We now make an important observation: probabilistic facts and distributional clauses can be deconstructed into a much more fundamental concept, which we call the {\em distributional fact}.  
Syntactically, a distributional fact is of the form \probloginline{v~d}. That is, a distributional clause with an empty body.
As a consequence, probabilistic facts and distributional clauses do not constitute fundamental concepts in PLP but are merely special cases, \ie while helpful for writing concise programs, they are only of secondary importance when it comes to semantics.

\begin{example}\label{ex:sweets_df}
We now rewrite the program in Example~\ref{example:sweets_dc} using distributional facts only. Note how probabilistic facts are actually syntactic \fixed{sugar for distributional facts}. The random variable is now distributed according to a Bernoulli distribution (\probloginline{flip}) and the atom of the probabilistic fact is the head of a rule with a probabilistic comparison in its body (\eg Lines 1 and 2 in the program below). Rewriting distributional facts is more involved. The main idea is to introduce a distinct random term for each distributional clause. Take for example, the random term \probloginline{red} in Example~\ref{example:sweets_dc}. This random term encodes, in fact, two distinct random variables, which we denote in the program below \probloginline{red_large} and \probloginline{red_small}. We now have to propagate this rewrite through the program and replace every occurrence of \probloginline{red} with \probloginline{red_large} and \probloginline{red_small}. This is why we get instead of two distributional clauses for \probloginline{yellow}, four distributional facts. It explains also why we get instead of one rule for \probloginline{favorite} in Example~\ref{example:sweets_dc} four rules now.

    \begin{problog*}{linenos}
rv_large ~ flip(0.5).
large :- rv_large=:=1.
rv_balanced ~ flip(0.5).
balanced :- rv_balanced=:=1.

red_large ~ poisson(20).
red_small ~ poisson(10).

yellow_large_balanced ~ poisson(red_large).
yellow_large_unbalanced ~ poisson(2*red_large).
yellow_small_balanced ~ poisson(red_small).
yellow_small_unbalanced ~ poisson(2*red_small).

favorite :- large, red_large > 15, 
              balanced, not yellow_large_balanced < 5.
favorite :- large, red_large > 15, 
              not balanced, not yellow_large_unbalanced < 5.
favorite :- not large, red_small > 15, 
              balanced, not yellow_small_balanced < 5.
favorite :- not large, red_small > 15, 
              not balanced, not yellow_small_unbalanced < 5.
\end{problog*}
\end{example}

The advantage of using probabilistic facts and distributional clauses is clear. They allow us to write much more compact and readable programs. However, as they do not really constitute fundamental building blocks of PLP, defining the semantics of a PLP language is much more intricate. For this reason we adapt a two-stage approach to define the semantics of \dcproblogsty. We first define the semantics of \dfplpsty (distributional fact PLP), a bare-bones language with no syntactic sugar only relying on distributional facts to define random variables. This happens in Section~\ref{sec:semantics}. During the second stage we define the program transformations to rewrite syntactic sugar (\eg distributional clauses) as distributional facts. The semantics of \dfplpsty and the program transformations then give us the \dcproblogsty language (\cf Section~\ref{sec:dcproblog}).

\subsection{Panorama of the Inference}
\label{sec:panorama_inference}

The part of the paper concerning inference consists of three sections. First, we start in Section~\ref{sec:inference-tasks} to define a query to a \dcproblogsty program.
For instance, we might be interested in the probability
$$
P(\mathprobloginline{favorite}=\top, \mathprobloginline{large}=\bot)
$$
In other words, the joint probability of \probloginline{favorite} being true and \probloginline{large} being false. While the example query above is simply a joint probability, we generalize this in Section~\ref{sec:inference-tasks} to conditional probabilities (possible with zero-probability events in the conditioning set).

Second, we map the queried ground program to a labeled Boolean formula. Contrary to the approach taken by~\citet{fierens2015inference} the labels are not probabilities (as usual in PLP) but indicator functions. This mapping to a labeled Boolean formula happens again in a series of program transformations, which we describe in Section~\ref{sec:dc2smt}. One of these steps is obtaining the relevant ground program to a query. For example, for the query above 
only the last two rules for  \probloginline{favorite} matter.
\begin{problog*}{}
favorite :- not large, rs > 15, balanced, not ysb < 5.
favorite :- not large, rs > 15, not balanced, not ysu < 5.
\end{problog*}
Here, we abbreviated \probloginline{red_small} as \probloginline{rs} and \probloginline{yellow_small_balanced} and \probloginline{yellow_small_unbalanced} as \probloginline{ysb} and \probloginline{ysu}, respectively. We can further rewrite these rules by replacing \probloginline{large} and \probloginline{balanced} with equivalent comparison atoms and pushing the negation into the comparisons:
\begin{problog*}{}
favorite :- rvl=:=0, rs > 15, rvb=:=1, ysb >= 5.
favorite :- rvl=:=0, rs > 15, rvb=:=0, ysu >= 5.
\end{problog*}
Again using abbreviations: \probloginline{rvl} for \probloginline{rv_large} and \probloginline{rvb} for \probloginline{rv_balanced}.

In Section~\ref{sec:alw} we then show how to compute the expected value of the labeled propositional Boolean formula corresponding to these rules by compiling it into an algebraic circuit, which is graphically depicted in Figure~\ref{fig:circuit:panorama}. In order to evaluate this circuit and obtain the queried probability (expected value), we introduce the IALW algorithm.

The idea of IALW is the following: sample the random variables dangling at the bottom of the circuit by sampling parents before children. For instance, we first sample from $\mathprobloginline{poisson}(10)$ (at the very bottom) before sampling from $\mathprobloginline{poisson}(rs)$ using the sampled value of the parent as the parameter of the child. Once we reach the comparison atoms (\eg $ysb\geq5$) we stick in the sampled values for the mentioned random variables. This evaluates the comparisons to either $1$ or $0$, for which we then perform the sums and products as prescribed by the circuit. We get a Monte Carlo estimate of the queried probability by averaging over multiple such evaluations of the circuit.

The method, as sketched here, is in essence the probabilistic inference algorithm Sampo presented in~\citep{zuidbergdosmartires2019exact}. The key contribution of IALW, which we discuss in Section~\ref{sec:alw}, is to extend Sampo such that conditional inference with zero probability events is performed correctly.  

\begin{figure}
	\resizebox{\linewidth}{!}{%

		\tikzstyle{distribution}=[rectangle, text centered, fill=white, draw, dashed,thick]
		\tikzstyle{leaf}=[rectangle, text centered, fill=gray!10, draw,thick]

		\tikzstyle{negate}=[
			rectangle split,
			rectangle split parts=3, 
			rectangle split horizontal,
			text centered,
			rectangle split part fill={gray!10,white,gray!10},
			draw,
			rectangle split draw splits=false,
			anchor=center,
			align=center,
			thick
		]
		\newcommand{\minus}{  ${\bm e^\otimes}$ \nodepart{second} ${{\bm \ominus}}$ \nodepart{third}  \phantom{${\bm e^\otimes}$}}

		\tikzstyle{sumproduct}=[
			rectangle split,
			rectangle split parts=3, 
			rectangle split horizontal,
			text centered,
			rectangle split part fill={gray!10,white,gray!10},
			draw,
			rectangle split draw splits=false,
			anchor=center,
			align=center,
			thick
		]
		\newcommand{\supr}[1]{  \phantom{${\bm e^\otimes}$} \nodepart{second} ${{\bm #1}}$ \nodepart{third} \phantom{${\bm e^\otimes}$}}
		
		\tikzstyle{circuitedge}=[ultra thick, thick,->]
		\tikzstyle{distributionedge}=[thick,->, dashed, in=-90]
		
		\tikzstyle{indexnode}=[draw,circle, inner sep=1pt]				
		
		\begin{tikzpicture}[remember picture]
			
			\node[sumproduct] (r) at (200.54bp,378.0bp) {\supr{\otimes}};
			\draw[ultra thick, thick,->] (r.90) to  ([shift={(0,1)}]r.90);
			
			\node[sumproduct] (l1) [below left = of r] {\supr{\oplus}};
			\node[sumproduct] (r1) [below right = of r,  yshift=-5.5cm]  {\supr{\otimes}};

			\node[leaf] (r21) [below right=of r1,  xshift=-1.3cm] {$\subnode{var_rvl}{rvl}=0$};
			\node[leaf] (r22) [below left=of r1, xshift=1.3cm] {$\subnode{var_rs}{rs}>15$};

			\node[sumproduct] (l21) [below left=of l1, xshift=0.5cm] {\supr{\otimes}};
			\node[sumproduct] (l22) [below right=of l1, xshift=-0.5cm] {\supr{\otimes}};

			\node[leaf] (l31) [below right=of l21,  xshift=-1.5cm] {$\subnode{var_rvb}{rvb}=1$};
			\node[leaf] (l32) [below left=of l21,  xshift=1.5cm] {$\subnode{var_ysb}{ysb}\geq5$};

			\node[distribution] (rvb)  [below=of l1,  yshift=-3.5cm,xshift=-0.5cm] {$\mathprobloginline{flip}(0.5)$};
			\node[distribution] (ysb)  [left=of rvb, xshift=0.4cm ]{$\mathprobloginline{poisson}(\subnode{poisson_rs_ysb}{rs})$};
			\node[distribution] (ysu)  [right=of rvb, xshift=-0.4cm] {$\mathprobloginline{poisson}(2\times \subnode{poisson_rs_ysu}{rs})$};

			\node[leaf] (l33) [below left=of l22,  xshift=1.5cm] {$\subnode{var_rvbbis}{rvb}=0$};
			\node[leaf] (l34) [below right=of l22,  xshift=-1.5cm] {$\subnode{var_ysu}{ysu}\geq5$};

			\node[distribution] (rvl)  [below=of r21] {$\mathprobloginline{flip}(0.5)$};
			\node[distribution] (rs)  [below=of rvb, yshift=-2cm] {$\mathprobloginline{poisson}(10)$};

			\draw[circuitedge] (l1) to  (r.mid);
			\draw[circuitedge] (r1) to  (r.three south |- r.mid);

			\draw[circuitedge] (l21) to  (l1.mid);
			\draw[circuitedge] (l22) to  (l1.three south |- l1.mid);

			\draw[circuitedge] (l21) to  (l1.mid);
			\draw[circuitedge] (l22) to  (l1.three south |- l1.mid);

			\draw[circuitedge] (l32) to  (l21.mid);
			\draw[circuitedge] (l31) to  (l21.three south |- l21.mid);

			\draw[circuitedge] (l33) to  (l22.mid);
			\draw[circuitedge] (l34) to  (l22.three south |- l22.mid);

			\draw[circuitedge] (r22) to  (r1.mid);
			\draw[circuitedge] (r21) to  (r1.three south |- r1.mid);

			\draw[distributionedge,out=90]  (ysb) to (var_ysb);
			\draw[distributionedge,out=90]  (ysu) to (var_ysu);
			\draw[distributionedge,out=90]  (rvb) to (var_rvb);
			\draw[distributionedge,out=90]  (rvb) to (var_rvbbis);

			\draw[distributionedge,out=90]  (rs) to (poisson_rs_ysb);
			\draw[distributionedge,out=90]  (rs) to (poisson_rs_ysu);
			\draw[distributionedge,out=90]  (rs) to (var_rs);

			\draw[distributionedge,out=90]  (rvl) to (var_rvl);

		\end{tikzpicture}
	}
    \captionof{figure}{Graphical representation of the computation graph (\ie algebraic circuit) used to compute the  probability $(\mathprobloginline{favorite}=\top, \mathprobloginline{large}=\bot)$ using the IALW algorithm introduced in Section~\ref{sec:alw}.}
    \label{fig:circuit:panorama}
\end{figure}
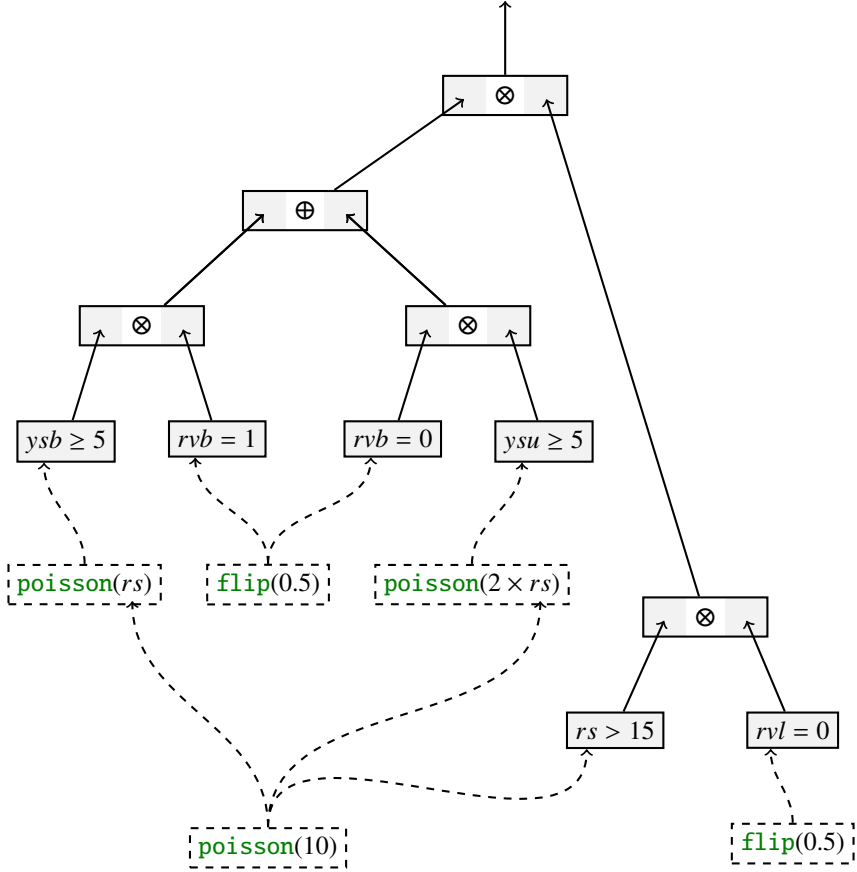

%% file: files_main/semantics.tex
\section{\dfplpsty}\label{sec:semantics}

Sato's distribution semantics~\citep{sato1995statistical} start from a probability measure over a countable set of facts $\comparisonfacts$, the so-called \emph{basic distribution}, which he then extends this to a probability measure over the Herbrand interpretations of the full program.
It is worth noting that the basic distribution is defined independently of the logical rules and that the random variables are mutually independent.

In our case, the set $\comparisonfacts$ consists of ground Boolean comparison atoms over the random variables, for which we drop the mutual independence assumption. These comparison atoms form an interface between the random variables (represented as terms) and the logical layer (clauses) that reasons about truth values of atoms.
\fixed{This is inspired by the work of \citet{gutmann2011magic} and \citet{nitti2016probabilistic} on the \dcsty language. We discuss this relationship in more detail in the related work Section~\ref{sec:related} and in Appendix~\ref{sec:dcproblog-dc}.
}

In this section we introduce the syntax and declarative semantics of \dfplpsty -- a probabilistic programming language with a minimal set of built-in predicates and functors. 
We do this in three steps.
Firstly, we discuss distributional facts and the probability measure over random variables they define (Section~\ref{sec:dist_facts_rand_var}).
Secondly, we introduce the Boolean comparison atoms that form the interface layer between random variables and a logic program (Section~\ref{sec:boolean_comparison_atoms}).
Thirdly, we add the logic program itself (Section~\ref{sec:log_cons_bool_comp}).
Fourht, we discuss practical considerations for constructing sets of distributional facts (Section~\ref{sec:fintiedistdb}).
An overview table of the notation related to semantics is provided in Appendix~\ref{app:table}.

\begin{definition}[Reserved Vocabulary]
    \label{def:reserved_vocabulary}
We use the following {\em reserved} vocabulary (built-ins), whose intended meaning is fixed across programs:
\begin{itemize}
    \item a finite set $\distributionfunctors$ of \emph{distribution functors}.
    \item a finite set $\arithmeticfunctors$ of \emph{arithmetic functors}.
    \item A finite set $\comparisonpredicates$ of binary \emph{comparison predicates}, 
    \item the binary predicate \predicate{~}{2} (in infix notation).
\end{itemize}
\end{definition}

Examples of distribution functors that we have already seen in Section~\ref{sec:panorama} are \functor{poisson}{1} and \functor{flip}{1} but may also include further functors such as \functor{normal}{2} to denote normal distributions. Possible arithmetic functors are \functor{*}{2} (\cf Example~\ref{example:sweets_dc}) but also \functor{max}{2}, \functor{+}{2}, \functor{abs}{1}, \etc.
Binary comparison predicates (in Prolog syntax and infix notation) are \predicate{<}{2}, \predicate{>}{2}, \predicate{=<}{2}, \predicate{>=}{2}, \predicate{=:=}{2}, \predicate{=\=}{2}.
The precise definitions of $\distributionfunctors$, $\arithmeticfunctors$ and $\comparisonpredicates$ are left to system designers implementing the language.

\begin{definition}[Regular Vocabulary] \label{def:regular_vocabulary}
    We call an atom $\mu(\rho_1, \dots, \rho_k)$ whose predicate \mathfunctor{\mu}{k} is not part of the reserved vocabulary a \emph{regular} atom. The set of all regular atoms constitutes the regular vocabulary. 
\end{definition}

As a  brief comment on notation: in the remainder of the paper we will usually denote logic program expressions in teletype font (\eg \probloginline{4>x}) when giving examples. When defining new concepts or stating theorems and propositions, we will use the Greek alphabet.

\subsection{Distributional Facts and Random Variables}\label{sec:dist_facts_rand_var}

\begin{definition}[Distributional Fact]\label{def:distfact}
A distributional fact is of the form $\nu \sim \delta$, with $\nu$ a regular ground term, and $\delta$ a ground term whose functor is in $\distributionfunctors$. 
The distributional fact states that the ground term $\nu$ is interpreted as a random variable distributed according to $\delta$.
\end{definition}

\begin{definition}[Sample Space]\label{def:samplespace}
    Let $\nu$ be be a random variable distributed according to $\delta$. The set of possible samples (or values) for $\nu$ is the sample space denoted by $\samplespace_\nu$ and which is determined by $\delta$. We denote a sample from $\samplespace_\nu$ by $\samplefunction(\nu)$, where $\samplefunction$ is the {\em sampling} or {\em value} function. 
\end{definition}

\begin{definition}[Distributional Database]\label{def:distDB}
A \emph{distributional database} is a \new{(not necessarily countable)} set $\distdb=\{\nu_1 \sim \delta_1,\nu_2 \sim \delta_2, \ldots \}$ of distributional facts, with distinct $\nu_i$. We let $\randomvariableset= \{\nu_1, \nu_2, \dots \}$ denote the set of random variables.
\end{definition}

\begin{example}\label{ex:dist_db} 
The following distributional database encodes a Bayesian network with  normally distributed random variables, two of which serve as parameters in the distribution of another one. We thus have $\randomvariableset= \{ \text{\probloginline{x}}, \text{\probloginline{y}},\text{\probloginline{z}} \}$.
\begin{problog*}{linenos}
x ~ normal(5,2).
y ~ normal(x,7).
z ~ normal(y,1).
\end{problog*}
\end{example}

\new{
\begin{definition}[Well-Defined Distributional Database]
    \label{def:well-distdb}
    We call a distributional database $\distdb$ well-defined if and only if the product of probability spaces induced by the random variables in $\distdb$ is unique. We denote this product probability space by
    $\probabilityspace_\distdb=(\samplespace_\distdb, \sigmaalgebra_\distdb, \probabilitymeasure_\distdb)$, where $\samplespace_\distdb$ is the sample space, $\sigmaalgebra_\distdb$ the sigma-algebra, and  $\probabilitymeasure_\distdb$ the probability measure. 
\end{definition}

\new{
Note that in Definition~\ref{def:distDB} we do not impose the restriction of having a countable set of random variables. This allows for modelling a rich class of probabilistic models, including random processes that are described via uncountable sets of random variables. We discuss the construction of well-defined databases in Section~\ref{sec:fintiedistdb}.
}

}




\subsection{Boolean Comparison Atoms over Random Variables} \label{sec:boolean_comparison_atoms}

Starting from a well-defined distributional database, we now introduce the \new{probability space $\probabilityspace_\comparisonfacts$ induced by Boolean comparison atoms}, which corresponds to the basic (discrete) distribution in Sato's distribution semantics.

\begin{definition}[Boolean Comparison Atoms]\label{def:comparison-atoms-set}
Let $\distdb$ be a well-defined distributional database. 
A binary \emph{comparison atom} $\gamma_1 {\bowtie} \gamma_2$ over $\distdb$ is a ground atom with predicate $\bowtie \in \comparisonpredicates$. The ground terms $\gamma_1$  and $\gamma_2$ are either random variables in $\randomvariableset$ or terms whose functor is in $\arithmeticfunctors$.
We denote by $\comparisonfacts$ a countable set of \new{$\probabilitymeasure_\distdb$-measurable} 
Boolean comparison atoms over $\distdb$.
\end{definition}

\begin{example}
    Examples of Boolean comparison atoms over the distributional database of Example~\ref{ex:dist_db} include   \probloginline{z>10},  \probloginline{x<y}, \probloginline{abs(x-y)=<1},  and \probloginline{7*x=:=y+5}. 
\end{example}

\new{
\begin{restatable}{proposition}{propomegaf}
    \label{prop:omegaf}
The Boolean comparison atoms $\comparisonfacts$ induce a product sample space $\samplespace_\comparisonfacts$.
\end{restatable}

\begin{proof}
    See Appendix~\ref{app:proof:omegaf}.
\end{proof}

}

\new{

\begin{restatable}{proposition}{proppfsigma}
    \label{prop:pfsigma}
    The Boolean comparison atoms $\comparisonfacts$ induce a sigma-algebra $\sigmaalgebra_\comparisonfacts{\subseteq}\sigmaalgebra_\distdb$.
\end{restatable}

\begin{proof}
    See Appendix~\ref{app:proof:pfsigma}.
\end{proof}

}

\new{

\begin{restatable}{proposition}{proppf}
    \label{prop:pf}
Let $\distdb$ be a well-defined distributional database, the function $\probabilitymeasure_\comparisonfacts$ defined via
$
\probabilitymeasure_\comparisonfacts(A)
=
\frac
{
    \probabilitymeasure_\distdb(A)
}
{
    \probabilitymeasure_\distdb(\samplespace_\comparisonfacts)
}
$ defines a unique probability measure over the sample space $\samplespace_\comparisonfacts$ and the sigma algebra $\sigmaalgebra_\comparisonfacts$.
\end{restatable}

\begin{proof}
    See Appendix~\ref{app:proof:pf}.
\end{proof}

}

\new{

\begin{proposition}
    The triplet $\probabilityspace_\comparisonfacts= (\samplespace_\comparisonfacts, \sigmaalgebra_\comparisonfacts, \probabilitymeasure_\comparisonfacts)$ forms a probability space.
\end{proposition}

\begin{proof}
    This follows immediately from Propositions \ref{prop:omegaf}, \ref{prop:pfsigma}, and \ref{prop:pf}.
\end{proof}
Note that, while \citeauthor{sato1995statistical} refers to $\probabilitymeasure_{\comparisonfacts}$ as the \textit{basic distribution}, we use the term  \textit{basic measure}. This is due to the fact that we do not necessarily have a distribution.
}


\subsection{Logical Consequences of Boolean Comparisons}\label{sec:log_cons_bool_comp}
We now define the semantics of a \dfplpsty program, \ie, extend the basic measure $\measurecomparisonfacts$ over the comparison atoms to a measure over the Herbrand interpretations of a logic program.

\begin{definition}[\dfplpsty Program]\label{def:core-prog}
A \dfplpsty program $\dfprogram= \distdb \cup \logicprogram $ consists of a well-defined distributional database $\distdb$ (Definition~\ref{def:well-distdb}), comparison atoms $\comparisonfacts$ (Definition~\ref{def:comparison-atoms-set}), and a normal logic program $\logicprogram$ where clause heads belong to the regular vocabulary (cf. Definition~\ref{def:regular_vocabulary}), and which can use comparison atoms from $\comparisonfacts$ in their bodies.
\end{definition}

\begin{example}\label{ex:random_choice_dependency}
We further extend the running example.
\begin{problog*}{linenos}
x ~ normal(5,2).
y ~ normal(x,7).
z ~ normal(y,1).
a :- abs(x-y) =< 1. 
b :- not a, z>10.
\end{problog*}
The logic program defines two logical consequences of Boolean comparisons, where \probloginline{a} is true if the absolute difference between random variables \probloginline{x} and \probloginline{y} is at most one, and \probloginline{b} is true if \probloginline{a} is false, and the random variable \probloginline{z} is greater than $10$. 
\end{example}

In order to extend the basic measure to logical consequences, \ie logical rules, we require the notion of a {\em consistent comparisons database} (CCD).
The key idea is that samples of the random variables in $\distdb$ jointly determine a truth value assignment to the comparison atoms in $\comparisonfacts$.

\begin{definition}
    A {\em value assignment}
    $\samplefunction(\randomvariableset)$ is a combined value assignment to all random variables $\randomvariableset=\{\nu_1,\nu_2,... \}$, \ie,  $\samplefunction(\randomvariableset)=(\samplefunction(\nu_1),\samplefunction(\nu_2),\ldots)$.
\end{definition}

\begin{definition}[Consistent Comparisons Database]\label{def:consistent-fact-db}
    Let $\distdb$ be a well-defined distributional database,  $\comparisonfacts=\{\kappa_1, \kappa_2,\ldots\}$ a set of measurable Boolean comparison atoms, and  $\samplefunction(\randomvariableset)$ a  value assignment to all random variables $\randomvariableset=\{ \nu_1,\nu_2,... \}$. 
    We define $I_{\samplefunction(\randomvariableset)}(\kappa_i)=\top$ if $\kappa_i$ is true after setting all random variables to their values under $\samplefunction(\randomvariableset)$, and $I_{\samplefunction(\randomvariableset)}(\kappa_i) = \bot$ otherwise. $I_{\samplefunction(\randomvariableset)}$ induces the consistent comparisons database $\comparisonfacts_{\samplefunction(\randomvariableset)}=\{\kappa_i \mid I_{\samplefunction(\randomvariableset)}(\kappa_i) = \top\}$.
\end{definition}

To define the semantics of a \dfplpsty program $\dfprogram$, we now require that, given a CCD $\comparisonfacts_{\samplefunction(\randomvariableset)}$, the logical consequences in $\dfprogram$ are uniquely defined.
As common in the PLP literature, we achieve this by requiring the program to have a two-valued well-founded model~\citep{van1991well} for each possible value assignment $\samplefunction(\randomvariableset)$.

\begin{definition}[Valid \dfplpsty Program]\label{def:core-valid}
A \dfplpsty program $\dfprogram=\distdb \cup\logicprogram $ is called \emph{valid} if and only if for each CCD $\comparisonfacts_{\samplefunction(\randomvariableset)}$, the logic program $\comparisonfacts_{\samplefunction(\randomvariableset)} \cup  \logicprogram$ has a two-valued well-founded model. 
\end{definition}

We follow the common practice of defining the semantics with respect to ground programs; the semantics of a program with non-ground $\logicprogram$ is defined as the semantics of its grounding with respect to the  Herbrand universe.

\begin{restatable}{proposition}{proppp}
\label{prop:pp}
A valid \dfplpsty program \dfprogram  induces a unique probability measure $\probabilitymeasure_\dfprogram$ over Herbrand interpretations.
\end{restatable}

\begin{proof}
    See Appendix~\ref{app:proof:pp}.
\end{proof}

\begin{definition}
    We define the declarative semantics of a \dfplpsty program \dfprogram to be the probability measure $\probabilitymeasure_\dfprogram$, and we call this the {\em measure semantics}.  
\end{definition}

In contrast to the imperative semantics of~\citet{nitti2016probabilistic}, in \dfplpsty the connection between comparison atoms and the logic program is purely declarative. That is, logic program negation on comparison atoms negates the (interpreted) comparison. For example, if we have a random variable \probloginline{n}, then \probloginline{n>=2} is equivalent to \probloginline{not n<2}. Such equivalences do not hold in the stratified programs introduced by~\citet{nitti2016probabilistic}.
This then allows the programmer to refactor the logic part as one would expect.

\subsection{Constructing Distributional Databases}
\label{sec:fintiedistdb}

\new{
In the definition of the measure semantics, we simply assumed that the distributional database $\distdb$ was valid, and we forewent an explicit construction of such databases.
When implementing \dcplpsty we would, however, ideally have a prescription for constructing $\distdb$.
A practical choice that is often made in probabilistic logic programming~\citep{kersting2000bayesian,fierens2015inference} restricts the distributional database to be a finite (and consequently countable) set.
For such distributional databases with a countable number of random variables to be meaningful, they have to encode a unique joint distribution over the variables~$\randomvariableset$.
}

We will now provide the conditions under which a finite distributional database is valid, but first we define the concepts of parents and ancestors that, which also hold for non-finite distributional databases.
\begin{definition}[Parents and Ancestors of Random Variables]
    \label{def:df_ancestor}
Let $\distdb$ be a distributional database. For facts $\nu_p\sim\delta_p$ and $\nu_c\sim\delta_c$  in $\distdb$.  The random variable $\nu_p$  is a \emph{parent} of the child random variable $\nu_c$ if and only if  $\nu_p$ appears in $\delta_c$. We define \emph{ancestor} to be the transitive closure of \emph{parent}. A node's ancestor set is the set of all its ancestors.
\end{definition}

\begin{example}[Ancestor Set]
    We graphically depict the ancestor set of the distributional database in Example~\ref{ex:dist_db} in Figure~\ref{fig:ex:dist_db_ancestor}.

    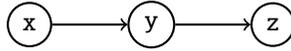
\begin{figure}[h]
        \centering
    \tikzstyle{node}=[circle, text centered, draw,thick]
        \begin{tikzpicture}[remember picture]

            \node[node] (x) {\probloginline{x}};
            \node[node] (y) [right=of x] {\probloginline{y}};
            \node[node] (z) [right=of y] {\probloginline{z}};

			\draw[->,thick] (x) to  (y);
			\draw[->,thick] (y) to  (z);

        \end{tikzpicture}
        \caption{Directed acyclic graph representing the ancestor relationship between the random variables in Example~\ref{ex:dist_db}. The ancestor set of \probloginline{x} is the empty set, the one of \probloginline{y} is $\{\mathprobloginline{x} \}$ and the one of \probloginline{z} is $\{\mathprobloginline{x}, \mathprobloginline{y} \}$.}
        \label{fig:ex:dist_db_ancestor}
    \end{figure}

\end{example}


\begin{definition}[Well-Defined Finite Distributional Database]\label{def:well-defd-facts}
A finite distributional database~$\distdb$  is called \emph{well-defined} if and only if it satisfies the following criteria:
\begin{description}
    \item[W!] Each $\nu \in \randomvariableset$ has a finite set of ancestors.
    \item[W1] The ancestor relation on the variables~$\randomvariableset$ is acyclic.
\item[W3] If $\nu \sim \delta \in \distdb$ and the parents of $\nu$ are $\{ \nu_1, \dots, \nu_m \}$, then replacing each occurrence of $\nu_i$ in $\delta$ by a sample $\samplefunction(\nu_i)$ always results in a well-defined distribution for~$\nu$. 
\end{description}  
\end{definition}

The distributional database in Example~\ref{ex:dist_db} is well-defined: the ancestor relation is acyclic and finite, and as normally distributed random variables are real-valued, using such a variable as the mean of another normal distribution is always well-defined. The database would no longer be well-defined after adding \probloginline{w ~ poisson(x)}, as not all real numbers can be used as a parameter of a Poisson distribution. 

\new{
Constructing valid finite distributional databases can be viewed as building Bayesian networks over a finite set of variables, where nodes represent (conditional) random variables and where each node's distribution is parameterized by the node's parents.
This approach was, for instance, taken by~\citep{kersting2000bayesian}.
More recently, \citet{wu2018discrete} extended these finite Bayesian network to allow also (under certain conditions) for infinite and uncountable numbers of nodes in these Bayesian networks. They dub this extension {\em measure theoretic Bayesian network} (MTBNs). 

While our measure semantics allow for MTBNs as the distributional database, we will restrict ourselves in the remainder of the paper to distributional databases that are finite. We leave it as future work on how to effectively use MTBNs in probabilistic logic programming. This would necessitate developing novel syntax and inference algorithms able to handle, for example, stochastic processes.
}

%% file: files_main/dcproblog.tex
\section{\dcproblogsty} \label{sec:dcproblog}

\fixed{
While the previous section has focused on the core elements of the \dcproblogsty language, we now introduce syntactic sugar to ease modeling.
We  consider three kinds of modeling patterns in \dfplpsty, and introduce a more compact notation for each of them.
Here we focus on examples and intuitions and relegate the more technical details to Appendix~\ref{sec:detaileddcproblog}.
Specifically, we discuss the precise semantics of the syntactic sugar introduced below and how to map it onto \dfplpsty language constructs (Appendix~\ref{sec:semantics_syntactic_sugar} and~\ref{sec:dcvalidity}).
In Appendix~\ref{sec:Additionalsyntacticsugar} we introduce additional syntactic sugar constructs for user-defined sample spaces and multivariate distributions, and in Appendix~\ref{sec:beyondmixtures} we study the intricacies introduced by negation when using syntactic sugar. 
}

\subsection{Boolean Random Variables}

The first modelling pattern concerns Boolean random variables, which we already encountered in Section~\ref{sec:panorama_semantics} as probabilistic facts (in \dcproblogsty) or as a combination of a Bernoulli random variable, a comparison atom, and a logic rule (in \dfplpsty). Below we give a more concise example.

\begin{example}
    We model, in \dfplpsty, an alarm that goes off for different reasons.
    \begin{problog*}{linenos}
issue1 ~ flip(0.1).
issue2 ~ flip(0.6).
issue3 ~ flip(0.3).

alarm :- issue1=:=1, not issue2=:=1.
alarm :- issue3=:=1, issue1=:=0.
alarm :- issue2=:=1.
    \end{problog*}
\end{example}
To make such programs more readable, we borrow the well-known notion of \emph{probabilistic fact} from discrete PLP, which directly introduces a logical atom as alias for the comparison of a random variable with the value \probloginline{1}, together with the probability of that value being taken.

\begin{definition}[Probabilistic Fact]
    A probabilistic fact is of the form $p\prob\mu$, where  $p$ is an arithmetic term that evaluates to a real number in the interval $[0,1]$ and $\mu$ is a regular ground atom.
\end{definition}
\begin{example}
    We use probabilistic facts to rewrite the previous example.
    \begin{problog*}{linenos}
0.1::problem1.
0.6::problem2.
0.3::problem3.

alarm :- problem1, not problem2.
alarm :- problem3, not problem1.
alarm :- problem2.
    \end{problog*}
\end{example}

\subsection{Probabilistically Selected Logical Consequences}
The second pattern concerns situations where a random variable with a finite domain is used to model a choice between several logical consequences:
\begin{example}
    We use a random variable to model a choice between whom to call upon hearing the alarm.
    \begin{problog*}{linenos}
call ~ finite([0.6:1,0.2:2,0.1:3]).
alarm.
call(mary) :- call=:=1, alarm.
call(john) :- call=:=2, alarm.
call(police) :- call=:=3, alarm.
    \end{problog*}
\end{example}
To more compactly specify such situations, we borrow the concept of an \emph{annotated disjunction} from discrete PLP~\citep{vennekens2004logic}.

\begin{definition}[Annotated Disjunction]
    An annotated disjunction (AD) is a rule of the form
    $$
        p_1\prob \mu_1; \dots; p_n \prob \mu_n \lpif \beta.
    $$
    where the  $p_i$'s are arithmetic terms each evaluating to a number in $[0,1]$ with a total sum of at most $1$. The $\mu_i$'s are regular ground atoms, and $\beta$ is a (possibly empty) conjunction of literals.
\end{definition}
The informal meaning of such an AD is "if $\beta$ is true, it probabilistically causes one of the $\mu_i$ (or none of them, if the probabilities sum to less than one) to be true as well".

\begin{example}
    We now use an AD for the previous example.
    \begin{problog}
alarm.
0.6::call(mary); 0.2::call(john); 0.1::call(police) :- alarm.
    \end{problog}
\end{example}

It is worth noting that the same head atom may appear in multiple ADs, whose bodies may be non-exclusive, \ie, be true at the same time. That is, while a single AD \emph{can} be used to model a multi-valued random variable, \emph{not all} ADs  encode such variables.
\begin{example}
    The following program models the effect of two kids throwing stones at a window.
    \begin{problog}
0.5::throws(suzy).
throws(billy).

0.8::effect(broken); 0.2::effect(none) :- throws(suzy).
0.6::effect(broken); 0.4::effect(none) :- throws(billy).
    \end{problog}
    Here, we have $P(\text{\probloginline{effect(broken)}})=0.76$ and $P(\text{\probloginline{effect(none)}})=0.46$, as there are worlds where both \probloginline{effect(broken)} and \probloginline{effect(none)} hold. The two ADs do hence not encode a categorical distribution.
    This is explicit in the \dfplpsty program, which contains two  random variables (\probloginline{x1} and \probloginline{x2}):
    \begin{problog}
x0 ~ flip(0.5).
throws(suzy) :- x0=:=1.
throws(billy).

x1 ~ finite([0.8:1,0.2:2]).
effect(broken) :- x1=:=1, throws(suzy).
effect(none) :- x1=:=2, throws(suzy).
x2 ~ finite([0.6:1,0.4:2]).
effect(broken) :- x2=:=1, throws(billy).
effect(none) :- x2=:=2, throws(billy).
    \end{problog}


\end{example}

\subsection{Context-Dependent Distributions}\label{sec:sugar-dc}
The third pattern is concerned with situations where the same conclusion is based on random variables with different distributions depending on specific properties of the situation, as illustrated by the following example.
\begin{example}
    We use two separate random variables to model that whether a machine works depends on the temperature being below or above a threshold. The temperature follows different distributions based on whether it is a hot day or not, but the threshold is independent of the type of day.
    \begin{problog*}{linenos}
0.2::hot.

temp_hot ~ normal(27,5).
temp_not_hot ~ normal(20,5).

works :- hot, temp_hot<25.0.
works :- not hot, temp_not_hot<25.0.
    \end{problog*}
\end{example}
To more compactly specify such situations, we borrow the syntax of \emph{distributional clauses} from the DC  language~\citep{gutmann2011magic}, which we already encountered in Section~\ref{sec:panorama_semantics}.

\begin{definition}[Distributional Clause]
    A distributional clause (DC) is a rule of the form
    $$
        \tau \sim \delta \lpif \beta.
    $$
    where $\tau$ is a regular ground expression, the functor of $\delta$ is in $\distributionfunctors$, and $\beta$ is a conjunction of literals.
\end{definition}

We call the left-hand side of the \mathpredicate{\sim}{2} prediate in a distributional clause a {\em random term} and the right-hand side a {\em distribution term}. Note that random terms cannot always be interpreted as random variables, which we discuss now.

The informal meaning of a distributional clause is "if $\beta$ is true, then the random term $\tau$ refers to a random variable that follows a distribution given by the distribution term $\delta$". Here, the distinction between \emph{refers to} a random variable and \emph{is} a random variable becomes crucial, as we will often have several distributional clauses for the same random term. This is also the case in the following example.
\begin{example}
    Using distributional clauses, we can rewrite the previous example with a single random term \probloginline{temp} as
    \begin{problog*}{linenos}
0.2::hot.

temp ~ normal(27,5) :- hot.
temp ~ normal(20,5) :- not hot.

works :- temp < 25.0.
    \end{problog*}
    The idea is that we still have two underlying random variables, one for each distribution, but the logic program uses the same term to refer to both of them depending on the logical context. The actual comparison facts are on the level of these implicit random variables, and \probloginline{temp<0.25} refers to one of them depending on context, just as in the original example.
\end{example}

%% file: files_main/tasks.tex
\section{Probabilistic Inference Tasks} \label{sec:inference-tasks}

In Section~\ref{sec:log_cons_bool_comp} we defined the probability distribution induced by a \dfplpsty program by extending the basic measure to logical consequences (expressed as logical rules). The joint distribution is then simply the joint distribution over all (ground) logical consequences. We obtain marginal probability distributions by marginalizing out specific logical consequences.
This means that marginal and joint probabilities of atoms in \dfplpsty programs are well-defined. 
Defining the semantics of probabilistic logic programs using an extension of Sato's distribution semantics gives us the semantics of probabilistic queries: the probability of an atom of interest is given by the probability induced by the joint probability of the program and marginalizing out all atoms one is not interested in.

The situation is more involved with regard to conditional probability queries. In contrast to unconditional queries, not all conditional queries are well-defined under the measure semantics (as well ass the distribution semantics). We will now give the formal definition of the PROB task, which lets us compute the (conditional) marginal probability of probabilistic events and which has so far not yet been defined in the PLP literature for hybrid domains under a declarative semantics (\eg \citep{azzolini2021semantics}).

After defining the task of computing conditional marginal probabilities, we will study how to compute these probabilities in the hybrid domain. Before defining the PROB task, we will first need to formally introduce the notion of a conditional probability with respect to a  \dcproblogsty program.

\begin{definition}[Conditional Probability]
\label{def:conditional_prob}
Let $\herbrandbase$ be 
the set of all ground atoms in a given \dcproblogsty program \dcpprogram.
Let $\evidenceset = \{\eta_1,\ldots,\eta_n\} \subset \herbrandbase$ be a set of observed atoms, and  $e=\langle e_1,\ldots, e_n\rangle$ a vector of corresponding  observed truth values, with $e_i\in \{\bot, \top \}$.
We refer to $(\eta_1=e_1) \wedge \ldots\wedge (\eta_n=e_n)$  as the evidence and write more compactly  $\evidenceset = e$.
Let  $\mu \in\herbrandbase$ be an atom of interest called the query.
If the probability of $\evidenceset = e$ is greater than zero, then the conditional probability of $\mu=\top$ given $\evidenceset=e$ is defined as:
\begin{align}
    P_\dcpprogram(\mu=\top \mid \evidenceset=e)= \frac{P_{\dcpprogram}(\mu=\top, \evidenceset=e)}{P_\dcpprogram(\evidenceset=e)} \label{eq:conditional_prob}
\end{align}
\end{definition}

\begin{definition}[PROB Task]
\label{def:prob_task}
Let $\herbrandbase$ be the set of all ground atoms of a given \dcproblogsty program \dcpprogram.
We are given the (potentially empty) evidence $\evidenceset=e$ (with $\evidenceset \subset \herbrandbase$)
and a set $\queryset \subset \herbrandbase$ of atoms of interest, called
query atoms.
The {\bf PROB task} consists of computing the conditional probability of the truth value of every atom in $\queryset$ given the evidence, \ie compute the conditional probability $P_{\dcpprogram} (\mu {=} \top \mid \evidenceset {=} e)$ for each $\mu \in \queryset$.
\end{definition}

\begin{example}[Valid Conditioning Set]
    Assume two random variables $\nu_1$ and $\nu_2$, where $\nu_1$ is distributed according to a normal distribution and $\nu_2$ is distributed according to a Poisson distribution. Furthermore, assume the following conditioning set $\evidenceset= \{\eta_1=\top, \eta_2=\top \}$, where $\eta_1 \leftrightarrow (\nu_1>0)$ and $\eta_2 \leftrightarrow (\nu_2=5)$. This is a valid conditioning set as none of the events has a zero probability of occurring, and we can safely perform the division in Equation~\ref{eq:conditional_prob}.
\end{example}

\subsection{Conditioning on  Zero-Probability Events}

A prominent class of conditional queries, which are not captured by Definition~\ref{def:conditional_prob}, are so-called  zero probability conditional queries. For such queries the probability of the observed event happening is actually zero but the event is still possible. Using Equation~\ref{eq:conditional_prob} does not work anymore as a division by zero would now occur.

\begin{example}[Zero-Probability Conditioning Set]
    Assume that we have a random variable $\nu$ distributed according to a normal distribution and that we have the conditioning set $\evidenceset= \{ \eta=\top \} $, with $\eta\leftrightarrow (\nu=20)$. In other words, we condition the query on the observation that the random variable $\nu$ takes the value $20$ -- for instance in a distance measuring experiment. This is problematic as the probability of any specific value for a random variable with uncountably many outcomes is in fact zero and applying Equation~\ref{eq:conditional_prob} leads to a division-by-zero. Consequently,  an ill-defined conditional probability arises.
\end{example}

In order to sidestep divisions by zero when conditioning on zero-probability (but possible) events, we modify Definition~\ref{def:conditional_prob}. Analogously to~\citet{nitti2016probabilistic}, we follow the approach taken in~\citep{kadane2011principles}.

\begin{definition}[Conditional Probability with Zero-Probability Events]
\label{def:conditional_prob_zero_event}
Let $\nu$ be a continuous random variable in the \dcproblogsty program \dcpprogram with ground atoms \herbrandbase. Furthermore, let us assume that the evidence consists of $\evidenceset = \{ \eta_0 = \top  \}$ with $\eta_0\leftrightarrow (\nu=w)$ and $w\in \samplespace_\nu$.
The conditional probability of an atom of interest $\mu\in \herbrandbase$ is now defined as:
\begin{align}
    P_\dcpprogram (\mu= \top \mid \eta_0= \top)
    &= \lim_{\Delta w \rightarrow 0}
    \frac{P_\dcpprogram(\mu=\top, \nu \in [w-\nicefrac{\Delta w}{2}, w+\nicefrac{\Delta w}{2} ]) }{P_\dcpprogram(\nu \in [w-\nicefrac{\Delta w}{2}, w+\nicefrac{\Delta w}{2} ] )} \label{eq:condition_prob_zero_ev}
\end{align}
\end{definition}

To write this limit more compactly, we introduce an 
infinitesimally small constant $\delta w$ and two new comparison atoms  $\eta_1 \leftrightarrow ( w-\nicefrac{\delta w}{2} \leq  \nu)$ and $\eta_2 \leftrightarrow (\nu \leq w+\nicefrac{\delta w}{2} )$ that together encode the limit interval.
Using these, we rewrite Equation~\ref{eq:condition_prob_zero_ev} as
\begin{align}
    P_\dcpprogram (\mu = \top \mid \eta_0 = \top) =  \frac{P_\dcpprogram(\mu=\top,  \eta_1=\top, \eta_2=\top )}{P_\dcpprogram( \eta_1=\top, \eta_2=\top )}
\end{align}

Applying the definition recursively, allows us to have multiple zero probability conditioning events. More specifically, let us assume an additional continuous random variable $\nu'$ that takes the value $w'$ for which we define: 
$\eta^{\prime}_1 \leftrightarrow ( w'-\nicefrac{\delta w'}{2} \leq  \nu')$ and $\eta^{\prime}_2 \leftrightarrow (\nu' \leq w'+\nicefrac{\delta w'}{2} )$.
This then leads to the following conditional probability:
\begin{align}
    P_\dcpprogram(\mu=\top \mid \nu=w, \nu'=w') 
    &= \frac{P_\dcpprogram(\mu=\top, \eta_1=\top, \eta_2=\top \mid \nu'=w' )}{P_\dcpprogram(\eta_1=\top, \eta_2=\top \mid \nu'=w')} \nonumber \\
    &= \frac{ \frac{P_\dcpprogram(\mu=\top, \eta_1=\top, \eta_2=\top, \eta'_1=\top, \eta'_2=\top)}{ \cancel{P_\dcpprogram(\eta'_1=\top, \eta'_2=\top)}} }{\frac{ P_\dcpprogram(\eta_1=\top, \eta_2=\top, \eta'_1=\top, \eta'_2=\top)}{\cancel{ P_\dcpprogram(\eta'_1=\top, \eta'_2=\top) }}} \nonumber \\
    &= \frac{P_\dcpprogram(\mu=\top,\eta_1=\top, \eta_2=\top, \eta'_1=\top, \eta'_2=\top)}{P_\dcpprogram(\eta_1=\top, \eta_2=\top, \eta'_1=\top, \eta'_2=\top)} 
\end{align}
Here we first applied the definition of the conditional probability for the observation of the random variable $\nu$ and then for the observation of the random variable $\nu'$. Finally, we simplified the expression.

\begin{proposition}
\label{prop:existence_cond_prob_zero_evidence}
The conditional probability as defined in Definition~\ref{def:conditional_prob_zero_event} exists.
\end{proposition}
\begin{proof}
See ~\citep[Equation 6]{nitti2016probabilistic}.
\end{proof}

In order to express zero-probability events in \dcproblogsty we add a new built-in comparison predicate to the finite set of comparison predicates $\comparisonpredicates =\{ \mathprobloginline{<}, \mathprobloginline{>},\allowbreak \mathprobloginline{=<}, \mathprobloginline{>=}, \mathprobloginline{=:=},\allowbreak \mathprobloginline{=\=} \}$ (cf. Definition~\ref{def:reserved_vocabulary}).

\begin{definition}[Delta Interval Comparison]\label{def:delta_interval}
    For a random variable \probloginline{v} and a rational number \probloginline{w}, we define \probloginline{delta_interval(v,w)} (with $\text{\predicate{delta_interval}{2}} \in \comparisonpredicates$) as follows. If \probloginline{v} has a countable sample space, then  \probloginline{delta_interval(v,w)} is equivalent to \probloginline{v=:=w}. Otherwise, \probloginline{delta_interval(v,w)} is equivalent to the conjunction of
    the two comparison atoms \probloginline{w-@$\delta$@w=<v} and \probloginline{v=<w+@$\delta$@w},
    where \probloginline{@$\delta$@w} is an infinitesimally small number.
\end{definition}

The delta interval predicate lets us express conditional probabilities with zero probability conditioning events as defined in Definition~\ref{def:conditional_prob_zero_event}.

Zero probability conditioning events are often abbreviated as $P_\dcpprogram (\mu=\top\mid \nu=w)$.
This can be confusing as it does not convey the intent of conditioning on an infinitesimally small interval. To this end, we introduce the symbol `$\doteq$' (equal sign with a dot on top). We use this symbol to explicitly point out an infinitesimally small conditioning set. For instance, we abbreviate the limit
\begin{align*}
    \lim_{\Delta w \rightarrow 0}
    \frac{P_\dcpprogram(\mu=\top, \nu\in [w-\nicefrac{\Delta w}{2}, w+\nicefrac{\Delta w}{2} ]) }{P_\dcpprogram(\nu \in [w-\nicefrac{\Delta w}{2}, w+\nicefrac{\Delta w}{2} ] )}
\end{align*}
in Definition~\ref{def:conditional_prob_zero_event} as:
\begin{align}
    P_\dcpprogram(\mu=\top \mid \nu\doteq w )
\end{align}

More concretely, if we measure the height $h$ of a person to be $180cm$ we denote this by $h\doteq 180$. This means that we measured the height  of the person to be in an infinitesimally small interval around $180cm$. Note that the $\doteq$ sign has slightly different semantics for random variables with a countable support. For discrete random variables the $\doteq$ is equivalent to the $equal$ sign.

\begin{example} \label{ex:conditional_prob_zero}
    Assume that we have a random variable $\nu$ distributed according to a normal distribution and that we have the evidence set $\evidenceset= \{ \eta=\top \} $, with $\eta\leftrightarrow (\nu\doteq 20)$. This is a valid conditional probability defined through Definition~\ref{def:conditional_prob_zero_event}.
\end{example}

\begin{example}
    Assume that we have a random variable $\nu$ distributed according to a normal distribution and that we have the conditioning set $\evidenceset= \{ \eta=\top, \eta' = \top \} $, with $\eta_1\leftrightarrow (\nu\doteq 20)$ and $\eta'\leftrightarrow (\nu\doteq 30)$. This does not encode a conditional probability as the conditioning event is not a possible event: one and the same random variable cannot be observed to have two different outcomes.
\end{example}

The notation used to condition on zero probability events (even when using `$\doteq$') hides away the limiting process that is used to define the conditional probability. This can lead to situations where seemingly equivalent conditional probabilities have diametrically opposed meanings.

\begin{example} \label{ex:conditional_prob}
    Let us consider the conditioning set $\evidenceset = \{ \eta=\top, \eta'=\top \}$, with $\eta \leftrightarrow (\nu\leq 20)$ and $\eta' \leftrightarrow (20\leq \nu)$, which we use again to condition a continuous random variable $\nu$. In contrast to Example~\ref{ex:conditional_prob_zero}, where we directly observed $\nu\doteq 20$, here, Definition~\ref{def:conditional_prob} applies, which 
    states that the conditional probability is undefined as $P(\nu\leq 20, 20\leq \nu )=0$.
\end{example}

\subsection{Discussion on the Well-Definedness of a Query}

The probability of an unconditional query to a valid \dcproblogsty program is always well-defined, as it is simply a marginal of the distribution represented by the program.  
This stands in stark contrast to conditional probabilities: an obvious issue are  divisions by zero occurring when the conditioning event does not belong to the set of possible outcomes of the conditioned random variable. Similarly to~\citet{wu2018discrete} we will assume for the remainder of the paper that conditioning events are always possible events, \ie events that have a non-zero probability but possibly an infinitesimally small probability  of occurring. This allows us to bypass potential issues caused by zero-divisions.\footnote{In general, deciding whether a conditioning event is possible or not is undecidable. This follows from the undecidability of general logic programs under the well-founded semantics~\citep{cherchago2007decidability}. A similar discussion is also presented in the thesis of Brian Milch~\citep[Proposition 4.8]{milch2006probabilistic} for the \blogsty language, which also discusses decidable language fragments~\citep[Section 4.5]{milch2006probabilistic}.}  

Even when discarding impossible conditioning events, conditioning a probabilistic event on a zero probability (but possible) event remains inherently ambiguous~\citep{jaynes2003probability} and might lead to the Borel-Kolmogorov paradox.
Problems arise when the limiting process used to define the conditional probability with zero probability events (cf. Definition~\ref{def:conditional_prob_zero_event}) does not produce a unique limit. 
For instance, a conditional probability  $P(\mu=\top \mid  2\nu \doteq \nu' )$, where $\nu$ and $\nu'$ are two random variables, depends on the parametrization used. We refer the reader to~\citep{shan2017exact} and \citep{jacobs2021paradoxes} for a more detailed discussion on ambiguities arising with zero probability conditioning events in the context of probabilistic programming.
We will sidestep such ambiguities completely by limiting observations of zero probability events to direct comparisons between random variables and numbers. This makes also sense from an epistemological perspective: we interpret a conditioning event as the outcome of an experiment, which produces a number, for instance the reading of a tape measure.

\subsection{Conditional Probabilities by Example}

\begin{example}\label{example:problog_machine}
The following ProbLog program models the conditions under which machines work. There are two machines (Line \ref{program:problog_machines_dec}), and three (binary) random terms, which we interpret as random variables as the bodies of the probabilistic facts are empty. The random variables are: the outside temperature (Line \ref{program:problog_machines_temp}) and  whether the cooling of each machine works (Lines \ref{program:problog_machines_cool1} and \ref{program:problog_machines_cool2}). Each machine works if its cooling works or if the temperature is low (Lines \ref{program:problog_machines_work_cool} and \ref{program:problog_machines_work_temp}).

	\begin{problog*}{linenos}
machine(1). machine(2). @\label{program:problog_machines_dec}@

0.8::temperature(low). @\label{program:problog_machines_temp}@
0.99::cooling(1). @\label{program:problog_machines_cool1}@
0.95::cooling(2). @\label{program:problog_machines_cool2}@

works(N):- machine(N), cooling(N). @\label{program:problog_machines_work_cool}@
works(N):- machine(N), temperature(low). @\label{program:problog_machines_work_temp}@
	\end{problog*}
We can query this program for the probability of \dcplpinline{works(1)} given that we have as evidence that \dcplpinline{works(2)} is true:
$$
P(\text{\dcplpinline{works(1)}}{=}1\mid\text{\dcplpinline{works(2)}}{=}1)\approx 0.998
$$
\end{example}

\begin{example}\label{example:dcproblog_machine}
In the previous example there are only Boolean random variables (encoded as probabilistic facts) and the \dcproblogsty program is equivalent to an identical \problogsty program. An advantage of \dcproblogsty is that we can now use an almost identical program to model the temperature as a continuous random variable.

	\begin{problog*}{linenos}
machine(1). machine(2).

temperature ~ normal(20,5). @\label{program:dcproblog_machines_temp}@
0.99::cooling(1).
0.95::cooling(2).

works(N):- machine(N), cooling(N).
works(N):- machine(N), temperature<25.0. @\label{program:dcproblog_machines_work_temp}@
	\end{problog*}
We can again ask for the probability of \probloginline{works(1)} given that we have as evidence that \probloginline{works(2)} is true but now the program also involves a continuous random variable:
$$
P(\text{\probloginline{works(1)}}{=}\top\mid\text{\probloginline{works(2)}}{=}\top)\approx 0.998
$$
\end{example}

In the two previous examples we were interested in a conditional probability where the conditioning event has a non-zero probability of occurring. However, \dcproblogsty programs can also encode conditional probabilities where the conditioning event has a zero probability of happening, while still being possible.
\begin{example}\label{example:dcproblog:observation}
	We model the size of a ball as a mixture of  different beta distributions, depending on whether the ball is made out of wood or metal (Line~\ref{program:dcproblog_machines_observation:ad}).
	We would now like to know the probability of the ball being made out of wood given that we have a measurement of the size of the ball.
	\begin{problog*}{linenos}
3/10::material(wood);7/10::material(metal).@\label{program:dcproblog_machines_observation:ad}@

size~beta(2,3):- material(metal)@\label{example:dcproblog:observation:beta23}@.
size~beta(4,2):- material(wood).
	\end{problog*}
Assume that we measure the size of the ball and we find that it is $0.4cm$, which means that we have a measurement (or observation) infinitesimally close to $0.4$. Using the `$\doteq$' notation, we write this conditional probability as:
	\begin{align}
		P\Bigl(\text{\probloginline{material(wood)}}{=}\top \mid \left(\text{\probloginline{size@$\doteq$@4/10}}\right) {=}\top \Bigr)
	\end{align}
\end{example}

The {\em Indian GPA problem} was initially proposed by Stuart Russell as an example problem to showcase the intricacies of mixed random variables.
Below we express the Indian GPA problem in \dcproblogsty.

\begin{example} \label{ex:indian_gpa}
The Indian GPA problem models US-American and Indian students and their GPAs. Both receive scores on the continuous domain, namely from 0 to 4 (American) and from 0 to 10 (Indian), cf. Line~\ref{ex:gpa:dens_am} and \ref{ex:gpa:dens_in}. With non-zero probabilities both student groups can also obtain marks at the extremes of the respective scales (Lines~\ref{ex:gpa:max_am}, \ref{ex:gpa:min_am}, \ref{ex:gpa:max_in}, \ref{ex:gpa:min_in}). 

\begin{problog*}{linenos}
1/4::american;3/4::indian.

19/20::isdensity(a).
99/100::isdensity(i).

17/20::perfect_gpa(a).
1/10::perfect_gpa(i).

gpa(a)~uniform(0,4):- isdensity(a). @\label{ex:gpa:dens_am}@
gpa(a)~delta(4.0):- not isdensity(a), perfect_gpa(a).  @\label{ex:gpa:max_am}@
gpa(a)~delta(0.0):- not isdensity(a), not perfect_gpa(a). @\label{ex:gpa:min_am}@
     
gpa(i)~uniform(0,10):- isdensity(i). @\label{ex:gpa:dens_in}@
gpa(i)~delta(10.0):- not isdensity(i), perfect_gpa(i).  @\label{ex:gpa:max_in}@
gpa(i)~delta(0.0):- not isdensity(i), not perfect_gpa(i).  @\label{ex:gpa:min_in}@
    
gpa(student)~delta(gpa(a)):- american.
gpa(student)~delta(gpa(i)):- indian.
\end{problog*}
Note that in order to write the probability distribution of \probloginline{gpa(a)} and \probloginline{gpa(i)} we used uniform and Dirac delta distributions. This allowed us to distribute the random variables \probloginline{gpa(a)} and \probloginline{gpa(i)} according to a discrete-continuous mixture distribution.
We then observe that a student has a GPA of $4$ and we would like to know the probability of this student being American or Indian. 
\begin{align}
P\Bigl(\text{\probloginline{american}}{=}\top \mid (\text{\probloginline{gpa(student)}}\doteq4)=\top \Bigr)&=1 
\nonumber \\
P \Bigl(\text{\probloginline{indian}}{=}\top \mid (\text{\probloginline{gpa(student)}}\doteq4)=\top \Bigr)&=0 \nonumber
\end{align}
\end{example}

%% file: files_main/dc2smt.tex
\section{Inference via Computing Expectations of Labeled Logic Formulas} \label{sec:dc2smt}

In the previous sections we have delineated the semantics of \dcproblogsty programs and described the PROB task that defines conditional probability queries on \dcproblogsty programs. The obvious next step is to actually perform the inference. We will follow an approach often found in implementations of PLP languages in the discrete domain: reducing inference in probabilistic programs to 
performing inference on labeled Boolean formulas that encode relevant parts of the logic program.
Contrary to languages in the discrete domain that follow this approach~\citep{fierens2015inference,riguzzi2011pita}, we will face the additional complication of handling random variables with infinite sample spaces. We refer the reader to~\citep[Section 5]{riguzzi2018foundations} for a broader overview of this approach.

Specifically, we are going to define a reduction from \dcproblogsty inference to the task of computing the expected label of a propositional formula. The formula is a propositional encoding of the  relevant part of the logic program (relevant with respect to a query), where atoms become propositional variables, and the labels of the basic facts in the distribution database are derived from the probabilistic part of the program.
At a high level, we extend \problogsty's inference algorithm such that Boolean comparison atoms over (potentially correlated) random variables are correctly being kept track of. The major complication, with regard to \problogsty and other systems such as PITA~\citep{riguzzi2011pita}, is the presence of context-dependent random variables, which are denoted by the same ground random term. For instance, the random term \probloginline{size} in the program in Example~\ref{example:dcproblog:observation} denotes two different random variables but is being referred to by one and the same term in the program.

Inference algorithms for PLP languages often consider only a fragment of the language for which the semantics have been defined. A common restriction for inference algorithms is to only consider range-restricted programs\footnote{We call a \dcproblogsty program range-restricted if it holds that for every statement all logic variables occurring in the head also occur as positive literals in the body. This guarantees that all terms will become ground during backward chaining. Note that range-restrictedness implies that all facts (including probabilistic and distributional ones) are ground.}. 
Furthermore, we consider, without loss of generality only AD-free programs, cf.~Definition~\ref{def:ad_free_program}, as annotated disjunctions or probabilistic facts can be eliminated up front by means of {\em local} transformations that  solely affect the annotated disjunctions (or probabilistic facts).\footnote{For non-ground ADs, we adapt Definition~\ref{def:elim-ad} to include all logical variables as arguments of the new random variable. As this introduces non-ground distributional facts, which are not range-restricted, we also move the comparison atom to the end of the rule bodies of the AD encoding to ensure those local random variables are ground when reached in backward chaining.}






The high level steps for converting a \dcproblogsty program to a labeled propositional formula closely follow the corresponding conversion for \problogsty programs provided by \citet[Section 5]{fierens2015inference}, \ie, given a \dcproblogsty program \dcpprogram,
		evidence $\evidenceset =e$ and a set of query atoms $\queryset$, the conversion algorithm performs the following steps:
\begin{enumerate}
    \item Determine the { relevant ground program} $\dcpprogram_g$ with respect to the atoms in  $\queryset\cup\evidenceset$ and obtain the corresponing \dfplpsty program. 
    \item Convert $\dcpprogram_g$  to an { equivalent propositional formula} $\phi_g$ and $\evidenceset =e$ to a propositional conjunction $\phi_e$.
    \item Define the { labeling function} for all atoms in $\phi_g$.
\end{enumerate}
Step 1 exploits the fact that ground clauses that have no influence on the truth values of  query or evidence atoms are irrelevant for inference and can thus be omitted from the ground program. Step 2 performs the conversion from logic program semantics to propositional logic, generating a formula that encodes \emph{all} models of the relevant ground program as well as a formula that serves to assert the evidence by conjoining both formulas. 
Step 3 completes the conversion by defining the labeling function. 
In the following, we discuss the three steps in more detail and prove correctness of our approach (cf. Theorem~\ref{thm:inference-by-expectation}).

\subsection{The Relevant Ground Program}

The first step in the conversion of a non-ground \dcproblogsty program to a labeled  Boolean formula consists of grounding the program with respect to a query set $\queryset$ and the evidence $\evidenceset=e$.
For each ground atom in $\queryset$ and $\evidenceset$ we construct its dependency set. That is, we collect the set of ground atoms and ground rules that occur in any of the proofs of an atom in $\queryset\cup \evidenceset$. The union of all dependency sets for all the ground atoms in $\queryset\cup\evidenceset$ is the dependency set of the \dcproblogsty with respect to the sets $\queryset$ and $\evidenceset$. This dependency set, consisting of ground rules and ground atoms, is called the relevant ground program (with respect to a set of queries and evidence).

\begin{example}
	\label{ex:grounded_nicely}
	Consider the non-ground (AD-free) \dcproblogsty program below.     
	\begin{problog*}{linenos}
rv_hot ~ flip(0.2).
hot:- rv_hot=:=1.
rv_cool(1) ~ flip(0.99).
cool(1):- rv_cool(1)=:=1.

temp(1) ~ normal(27,5):- hot.
temp(1) ~ normal(20,5):- not hot.

works(N):- cool(N).
works(N):- temp(N)<25.0.
	\end{problog*}
If we ground it with respect to the query \probloginline{works(1)} and subsequently apply the rewrite rules from Section~\ref{sec:eliminate_dc} we obtain:
	\begin{problog*}{linenos}
rv_hot ~ flip(0.2). 
hot:- rv_hot=:=1.
rv_cool(1) ~ flip(0.99).
cool(1):- rv_cool(1)=:=1.

temp(hot) ~ normal(27,5).
temp(not_hot) ~ normal(20,5).

works(1):- cool(1).
works(1):- hot, temp(hot)<25.0,
works(1):- not hot, temp(not_hot)<25.0.
	\end{problog*}
\end{example}

A possible way, as hinted at in Example~\ref{ex:grounded_nicely}, of obtaining a ground \dfplpsty program from a non-ground \dcproblogsty program is to first ground out all the logical variables. Subsequently, one can apply Definition~\ref{def:elim-ad} to eliminate annotated disjunctions and probabilistic facts, and Definition~\ref{def:adfree-to-core} in order to obtain a \dfplpsty program with no distributional clauses.
A possible drawback of such a two-step approach (grounding logical variables followed by obtaining a \dcproblogsty program) is that it might introduce spurious atoms to the relevant ground program. A more elegant but also more challenging approach is to interleave the grounding of logical variables and distributional clause elimination. We leave this for future research.

\begin{restatable}[Label Equivalence]{theorem}{theorgp}
	\label{theo:rgp}
	Let $\dcpprogram$ be a \dcproblogsty program and let $\dcpprogram_g$ be the relevant ground program for $\dcpprogram$ with respect to a query $\mu$ and the evidence $\evidenceset=e$ obtained by first grounding out logical variables and subsequently applying transformation rules from Section~\ref{sec:dcproblog}. The programs $\dcpprogram$ and $\dcpprogram_g$ specify the same probability:
	\begin{align}
		P_{\dcpprogram}(\mu=\top \mid \evidenceset=e)
		=
		P_{\dcpprogram_g}(\mu=\top \mid \evidenceset=e)
	\end{align}
    \end{restatable}

    \begin{proof}
        See Appendix~\ref{app:proof:rgp}.
        \end{proof}

\subsection{The Boolean Formula for the Relevant Ground Program}
\label{sec:relprog2boolfrom}

Converting a ground logic program, \ie a set of ground rules, into an equivalent Boolean formula is a purely logical problem and well-studied in the non-probabilistic logic programming literature. We refer the reader to \citet{janhunen2004representing} for an account of the transformation to Boolean formula in the non-probabilistic setting and to \citet{mantadelis2010dedicated} and \citet{fierens2015inference} in the probabilistic setting,
including correctness proofs. We will only illustrate the most basic case with an example here.

\begin{example}[Mapping \dcproblogsty to Boolean Formula] \label{example:dc2bool}
Consider the ground program in Example~\ref{ex:grounded_nicely}. 
To highlight the move from logic programming to propositional logic,	
	we introduce for every atom \probloginline{a} in the program a 
	corresponding propositional variable  $\phi_\text{\probloginline{a}}$.
	As the program does not contain cycles, we can use Clark's completion for the transformation, \ie, a derived atom is true if and only if the disjunction of the bodies of its defining rules is true. The propositional formula $\phi_g$ corresponding to the program is then the conjunction of the following three formulas:
	\begin{align*}
	    \phi_\mathprobloginline{works(1)} &\leftrightarrow
		\Big( \phi_{\text{\probloginline{cool(1)}}} 
		\lor \phi_\text{\probloginline{hot}} \land  \phi_\text{\probloginline{temp(hot)<25.0}}  
		\lor\neg \phi_\text{\probloginline{hot}}  \land \phi_\text{\probloginline{temp(not_hot)<25.0}}   \Big) \\
		  \phi_{\text{\probloginline{cool(1)}}} &\leftrightarrow  \phi_{\text{\probloginline{rv_cool(1)=:=1}}}\\
		\phi_\text{\probloginline{hot}} &\leftrightarrow \phi_{\text{\probloginline{rv_hot=:=1}}}
	\end{align*}

\end{example}

Note that the formula obtained by converting the relevant ground program still admits \emph{any} model of that program, including ones that are inconsistent with the evidence. In order to use that formula to compute conditional probabilities, we still need to assert  the evidence into the formula by conjoining the corresponding propositional literals. 
The following theorem then directly applies to our case as well.
\begin{theorem}[Model Equivalence~\citep{fierens2015inference} (Theorem 2, part 1)]
\label{theo:model_equivalence}
Let $\dcpprogram_g$ be the relevant ground program for a \dcproblogsty program $\dcpprogram$ with respect to  query set \queryset\ and  evidence $\evidenceset=e$. Let $MOD_{\evidenceset=e}(\dcpprogram_g)$ be those models in $MOD(\dcpprogram_g)$ that are consistent with the evidence.
Let $\phi_g$ denote the  propositional formula derived from $\dcpprogram_g$, and set $\phi \leftrightarrow \phi_g \land \phi_e$, where  $\phi_e$ is the conjunction of literals that corresponds to the observed truth values of the atoms in $\evidenceset$. We then have {\bf model equivalence}, \ie, 
\begin{align}
    MOD_{\evidenceset=e}(\dcpprogram_g)
    =
    ENUM(\phi)
\end{align}
where $ENUM(\phi)$ denotes the set of models of $\phi$.
\end{theorem}

\subsection{Obtaining a Labeled Boolean Formula}

In contrast to a \problogsty program, a \dcproblogsty program does not explicitly provide independent probability labels for the basic facts in the distribution semantics, and we thus need to suitably adapt the last step of the conversion. 
We will first define the labeling function on propositional atoms and will then show that the probability of the label of a propositional formula is the same as the probability of the relevant ground program under the measure semantics from Section~\ref{sec:semantics}. We call this {\em label equivalence} and prove it in Theorem~\ref{theo:label_equivalence}.

\begin{definition}[Label of Literal] \label{def:labeling_function}
The label $\alpha(\phi_{\rho})$ of a propositional atom $\phi_{\rho}$ (or its negation) is given by:
\begin{align}
    \alpha( \phi_{\rho })=
    \begin{cases} 
    \ive{c(vars(\rho)},  &  \text{if $\rho$ is a  comparison atom} \\
    1, & \text{otherwise}
    \end{cases}
\end{align}
and for the negated atom:
\begin{align}
    \alpha(\neg \phi_{\rho})=
    \begin{cases} 
    \ive{\neg c(vars(\rho))},  &  \text{if $\rho$ is a  comparison atom} \\
    1, & \text{otherwise}
    \end{cases}
\end{align}
We use Iverson brackets $\ive{\cdot}$~\citep{iverson1962programming} to denote an indicator function. Furthermore, $vars(\rho)$ denotes the random variables that are present in the arguments of the atom $\rho$ and $c(\cdot)$ encodes the constraint given by $\rho$.
\end{definition}

\begin{example}[Labeling function]
Continuing Example~\ref{example:dc2bool}, we obtain, inter alia, the following labels:
\begin{align*}
    & \alpha(\phi_\text{\probloginline{rv_hot=:=1}} ) = \ive{rv\_hot=1}  \\
    & \alpha(\neg \phi_\text{\probloginline{rv_hot=:=1}} ) = \ive{\neg(rv\_hot=1)} = \ive{rv\_hot=0} \\
    & \alpha(\phi_\text{\probloginline{hot}} ) = 1\\  
    & \alpha(\neg \phi_\text{\probloginline{hot}} ) = 1 
\end{align*}
\end{example}

\begin{definition}[Label of Boolean Formula]
\label{def:label_bool_formula}
Let $\phi$ be a Boolean formula and $\alpha(\cdot)$ the labeling function for the variables in $\phi$ as given by Definition~\ref{def:labeling_function}. We define the label of $\phi$ as
\begin{align*}
    \alpha(\phi) &= \sum_{\varphi\in ENUM(\phi)}\prod_{\ell\in\varphi}\alpha(\ell) 
\end{align*}
i.e. as the sum of the labels of all its models, which are in turn defined as the product of the labels of their literals.
\end{definition}

\begin{example}[Labeled Boolean Formula]
The label of the conjunction 
$$\neg \phi_\text{\probloginline{hot}} \land \neg \phi_{\text{\probloginline{rv_hot=:=1}}} \land  \phi_\text{\probloginline{temp(not_hot)<25.0}} \land \neg \phi_{\text{\probloginline{cool(1)}}}  \land
		\neg \phi_{\text{\probloginline{rv_cool(1)=:=1}}} \land \phi_{\text{\probloginline{works(1)}}}$$
which describes one model of the example formula, is computed as follows:
\begin{align*}
    &\phantom{{}={}}
    \alpha( \neg \phi_\text{\probloginline{hot}}
    \land
    \neg \phi_{\text{\probloginline{rv_hot=:=1}}}
    \land
    \phi_\text{\probloginline{temp(not_hot)<25.0}}\\
    &\phantom{{}={}}\land
    \neg \phi_{\text{\probloginline{cool(1)}}}
    \land
	\neg \phi_{\text{\probloginline{rv_cool(1)=:=1}}}
	\land \phi_{\text{\probloginline{works(1)}}} )  \nonumber \\
	&=
    \alpha( \neg \phi_\text{\probloginline{hot}})
    \times \alpha( \neg \phi_{\text{\probloginline{rv_hot=:=1}}})
    \times  \alpha(\phi_\text{\probloginline{temp(not_hot)<25.0}}) \nonumber \\
    &\phantom{{}={}}
    \times \alpha(\neg \phi_{\text{\probloginline{cool(1)}}})
    \times \alpha(\neg \phi_{\text{\probloginline{rv_cool(1)=:=1}}} \times \alpha(\phi_{\text{\probloginline{works(1)}}}) )  \nonumber \\
    &= 1 \times \ive{rv\_hot=0} \times \ive{temp(not\_hot)<25} \times 1 \times \ive{rv\_cool(1)=0}\times 1\nonumber \\
    &= \ive{rv\_hot=0} \times \ive{temp(not\_hot)<25} \times \ive{rv\_cool(1)=0} 
\end{align*}

\end{example}

\begin{restatable}[Label Equivalence]{theorem}{Labelequivalence}
\label{theo:label_equivalence}

Let $\dcpprogram_g$ be the relevant ground program for a \dcproblogsty program $\dcpprogram$ with respect to a query $\mu$ and the evidence $\evidenceset=e$. Let $\phi_g$ denote the  propositional formula derived from $\dcpprogram_g$ and let $\alpha$ be the labeling function as defined in Definition~\ref{def:labeling_function}. We then have {\bf label equivalence}, \ie
\begin{align}
    \forall \varphi \in ENUM(\phi_g):  \E_{\randomvariableset \sim  \dcpprogram_g} [\alpha( \varphi )] = P_{\dcpprogram_g}(\varphi)
\end{align}
In other words, for all models $\varphi$ of $\phi_g$, the expected value ($\E_\cdot [\cdot]$) of the label of $\varphi$ is equal to the probability of $\varphi$ according to the probability measure of relevant ground program $\dcpprogram_g$.
\end{restatable}

\begin{proof}
See Appendix~\ref{app:proof:label_equivalence}.
\end{proof}

Theorem~\ref{theo:label_equivalence} states that we can reduce inference in hybrid probabilistic logic programs to computing the expected value of  labeled Boolean formulas, as summarized in the following theorem.
\begin{theorem}
\label{thm:inference-by-expectation}
Given a \dcproblogsty program \dcpprogram, a set \queryset\ of queries, and evidence $\evidenceset = e$, for every $\mu\in\queryset$, we obtain the conditional probability of $\mu = q$ ($q\in \{\bot,\top \}$) given $\evidenceset = e$ as
\begin{align*}
    P(\mu=q\mid\evidenceset = e) = \frac{\E_{vars(\phi) \sim  \dcpprogram_g} [\alpha( \phi \wedge \phi_q)] }{\E_{vars(\phi) \sim  \dcpprogram_g} [\alpha( \phi )] }
\end{align*}
where $\phi$ is the formula encoding the relevant ground program $\dcpprogram_g$ with  the evidence asserted  (cf.~Theorem~\ref{theo:model_equivalence}), and $\phi_q$ the propositional atom for $\mu$.
\end{theorem}
\begin{proof}
This directly follows from model and label equivalence together with the definition of conditional probabilities. 
\end{proof}

We have shown that the probability of a query to a \dcproblogsty program can be expressed as the expected label of a propositional logic formula.

%% file: files_main/evaluating.tex
\section{Computing Expected Labels via Algebraic Model Counting}
\label{sec:alw}

In this section we will adapt the approach taken by~\citet{zuidbergdosmartires2019exact}, dubbed {\em Sampo} to compute the expected value of labeled propositional Boolean formulas.
The method approximates intractable integrals that appear when computing expected labels using Monte Carlo estimation. The main difference between Sampo and our approach, which we dub {\em infinitesimal algebraic likelihood weighting} (IALW) is that IALW can also handle infinitesimally small intervals, which arise when conditioning on zero probability events.

\subsection{Monte Carlo Estimate of Conditional Query}

In Definition~\ref{def:conditional_prob} we defined the conditional probability as:
\begin{align}
    P_\dcpprogram(\mu=\top\mid\evidenceset=e)= \frac{P_{\dcpprogram}(\mu=\top, \evidenceset=e)}{P_\dcpprogram(\evidenceset=e)} 
\end{align}
and we also saw in Definition~\ref{def:conditional_prob_zero_event} that using infinitesimal intervals allows us to consider zero probability events, as well. Computing the probabilities in the numerator and denominator in the equation above is, in general, computationally hard. We resolve this using a Monte Carlo approximation. 

\begin{restatable}[Monte Carlo Approximation of a Conditional Query]{proposition}{mcapproxconditional}
\label{prop:mcapproxconditional}
Let the set 
\begin{align}
    \mathcal{S} = \left\{ \left(s_1^{(1)}, \dots, s_M^{(1)} \right), \dots , \left(s_1^{(\lvert \mathcal{S} \rvert)}, \dots, s_M^{(\lvert \mathcal{S} \rvert)} \right)  \right\} \label{eq:rejection_samples}
\end{align}
denote $\lvert \mathcal{S} \rvert$ i.i.d. samples for each random variable in $ \dcpprogram_g$.
A conditional probability query to a  \dcproblogsty program \dcpprogram can be approximated as: 
\begin{align}
P_\dcpprogram(\mu = q \mid \evidenceset = e) 
&\approx  \frac{ \sum_{i=1}^{\lvert \mathcal{S} \rvert}  \sum_{\varphi \in ENUM(\phi \land \phi_q) } \alpha^{(i)}(\varphi) } { \sum_{i=1}^{\lvert \mathcal{S} \rvert} \sum_{\varphi \in ENUM(\phi) } \alpha^{(i)}(\varphi) }, & \quad |\mathcal{S}|<\infty
\end{align}
The index $(i)$ on $\alpha^{(i)}(\varphi)$ indicates that the label of $\varphi$ is evaluated at the $i$-th ordered set of samples $ \left(s_1^{(i)}, \dots, s_M^{(i)} \right)$. 
\end{restatable}

\begin{proof}
    See Appendix~\ref{app:proof:mcapproxconditional}.
\end{proof}

In the limit $\lvert \mathcal{S} \rvert\rightarrow \infty$ this sampling approximation scheme is perfectly valid. However, in practice, with only limited resources available, such a rejection sampling strategy will perform poorly (in the best case) or even give completely erroneous results. After all, the probability of sampling a value from the prior distribution that falls exactly into an infinitesimally small interval given in the evidence tends to zero.
To make the computation of a conditional probability, using Monte Carlo estimates, feasible, we are going to introduce {\em infinitesimal algebraic likelihood weighting}. But first, we will need to introduce the concept of infinitesimal numbers.

\subsection{Infinitesimal Numbers}
Remember that infinitesimal intervals arise in zero probability conditioning events and describe an infinitesimally small interval around a specific observed value, \eg $\nu \in [w-\nicefrac{\Delta w}{2}, w+\nicefrac{\Delta w}{2} ]$ for a continuous random variable $\nu$ that was observed to take the value $w$ (cf. Definition~\ref{def:conditional_prob_zero_event}).
We will describe these infinitesimally small intervals using so-called {\em infinitesimal numbers}, which were first introduced by~\citet{nitti2016probabilistic} and further formalized in~\citet{wu2018discrete}, \citep{zuidberg2020atoms} and~\citep{jacobs2021paradoxes}. The latter work also coined the term {\em `infinitesimal number'} \fixed{and we refer the reader specifically to \citet[Section 5.2]{jacobs2021paradoxes} for an intuitive exposition of infinitesimal numbers.}

\begin{definition}[Infinitesimal Numbers]
    \label{def:inf_number}
An infinitesimal number is a pair \fixed{$(r, n) \in \mathbb{R} \times (\mathbb{N} \cup +\infty )$}, also written as $r\epsilon^n$, and which corresponds to a real number when $n=0$. We denote the set of all infinitesimal numbers by $\mathbb{I}$.
\end{definition}

\begin{definition}[Operations in $\mathbb{I}$]
\label{def:inf_number_ops}
Let $(r, n)$ and $(t,m)$ be two numbers in $\mathbb{I}$. We define the addition and multiplication as binary operators:
\begin{align}
	(r,n) \oplus
	(t,m)
	&\coloneqq
	\begin{cases}
		(r+t,n) &\quad  \text{if $n=m$} \\
		(r,n) &\quad  \text{if $n<m$} \\
		(t,m) &\quad  \text{if $n>m$}
	\end{cases} 
    \label{eq:infininumber_plus}
    \\
	(r,n) \otimes
	(t,m)
	&\coloneqq (r \times t , n+m)  &
    \label{eq:infininumber_times}
\end{align}
The operations $+$ and $\times$ on the right hand side denote the usual addition and multiplication operations for real and integer numbers.
\end{definition}

\begin{definition}[Neutral Elements]
\label{def:inf_number_neutral_elem}
The neutral elements of the addition and multiplications in $\mathbb{I}$ are, respectively, defined as:
\begin{align}
	e^\oplus  \coloneqq (0,+\infty)  &&
	e^\otimes  \coloneqq (1,0)
\end{align}
\end{definition}

Probabilistic inference and generalization thereof can often be cast as performing computations using commutative semirings~\citep{kimmig2017algebraic}. We will follow a similar strategy.

\begin{definition}\label{def:comm_semiring} 
	A {\bf  commutative semiring} is an algebraic structure $(\mathcal{A},\oplus,\otimes,\allowbreak e^{\oplus},e^\otimes)$ equipping a set of elements $\mathcal{A}$ with addition and multiplication such that
	\begin{enumerate}
		\item addition $\oplus$ and multiplication $\otimes$ are binary operations $\mathcal{A}\times \mathcal{A}\rightarrow \mathcal{A}$
		\item addition $\oplus$ and multiplication $\otimes$ are  associative and commutative binary operations over the set $\mathcal{A}$
		\item $\otimes$ distributes over $\oplus$
		\item  $e^\oplus \in \mathcal{A}$ is the neutral element of $\oplus$
		\item  $e^\otimes \in \mathcal{A}$ is the neutral element of $\otimes$
		\item $e^\oplus \in \mathcal{A}$ is an annihilator for $\otimes$
	\end{enumerate}
\end{definition}

\begin{lemma}
The structure $(\mathbb{I}, \oplus, \otimes, e^\oplus , e^\otimes )$ is a commutative semiring.
\end{lemma}
\begin{proof}
This follows trivially from the operations defined in Definition~\ref{def:inf_number_ops} and the neutral elements in Definition~\ref{def:inf_number_neutral_elem}.
\end{proof}
We will also need to perform subtractions and divisions in $\mathbb{I}$,


\begin{definition}[Subtraction and Division in $\mathbb{I}$]
    \label{def:subdiv}
    Let $(r, n)$ and $(s,m)$ be two numbers in $\mathbb{I}$. We define the subtraction and division as:
    \begin{align}
            (r,n) \ominus (t,m) &\coloneqq  (r,n) \oplus (-t,m) \\
            (r,n) \oslash (t,m) &\coloneqq 
            \begin{cases}
                \text{undefined} &\quad \text{if $|n|=|m|=\infty$ and $sign(n) \neq sign(m)$} \\
                (\nicefrac{r}{t}, n-m)  &\quad \text{if $t\neq 0$} \\
                \text{undefined} &\quad \text{if $t=0$}
            \end{cases}
    \end{align}

\end{definition}

We would like to note that similar algebraic structures have been used for counting optimal variable assignments in graphical models~\citep{marinescu2019counting} and probabilistic inference in generating circuits~\citep{harviainen2023inference}.

\subsection{Infinitesimal Algebraic Likelihood Weighting}
The idea behind IALW is that we do not sample random variables that fall within an infinitesimal small interval, encoded as a delta interval (cf. Definition~\ref{def:delta_interval}), but that we force, without sampling, the random variable to lie inside the infinitesimal interval. 
To this end, assume again that we have $\lvert \mathcal{S} \rvert$ i.i.d. samples for each random variable. That means that we have again a set of ordered sets of samples:
\begin{align}
    \label{eq:ancestral_samples}
    \mathcal{S} = \left\{ \left(s_1^{(1)}, \dots, s_M^{(1)} \right), \dots , \left(s_1^{(\lvert \mathcal{S} \rvert)}, \dots, s_M^{(\lvert \mathcal{S} \rvert)} \right)  \right\}
\end{align}

This time the samples are drawn with the infinitesimal delta intervals taken into account. For example, assume we have a random variable $\nu_1$ distributed according to a normal distribution $\mathcal{N}(5,2)$ and we have an atom \probloginline{delta_interval(@$\nu_1$@,4)} in the propositional formula $\phi$. Each sampled value of $s_1^{(i)}$ will then equal $4$ ( $1\leq i\leq \lvert \mathcal{S} \rvert$). Furthermore, when sampling, we sample the parents of a random variable prior to sampling the random variable itself. For instance, take the random variable $\nu_2\sim \mathcal{N}(\nu_3=w,2)$, where $\nu_3$ is itself a random variable. We first sample $\nu_3$ and once we have a value for $\nu_3$ we plug that into the distribution for $\nu_2$, which we sample subsequently. In other words, we sample according to the ancestor relationship between the random variables.
We call the ordered set of samples $\varset{s}^{(i)} \in \mathcal{S}$ an {\em ancestral sample}.

\begin{definition}[IALW Label] \label{def:sample_labeling_function}
    \fixed{
Let $\delta_k$ denote the probability distribution of a random variable $\nu_k$.
Given an ancestral sample $\varset{s}^{(i)}= (s_1^{(i)}, \dots,  s_M^{(i)} ) $ for the random variables $\randomvariableset = (\nu_1,\dots, \nu_M)$, we denote by $\delta_k( \varset{s}^{(i)} )$ the evaluation of the density $\delta_k$ at $\varset{s}^{(i)}$, where $i$ specifies the $i$-th sample.
}
The IALW label of a positive literal $\ell$ is an infinitesimal number given by:
\begin{align}
    &\alpha_{IALW}^{(i)}( \ell) \nonumber \\
    &=\begin{cases}
    (\delta_k(\varset{s}^{(i)}) , 1),  &  \text{if $\ell$ is a \probloginline{delta_interval} whose first argument } \\
    & \text{is a continuous random variable} \\
    ( \ive{ c_\ell(\varset{s}^{(i)}) }, 0), & \text{if $\ell$ is any comparison atom}
    \\
    (1, 0), & \text{otherwise} 
    \end{cases}
    \nonumber
\end{align}
The expression $\ive{ c_\ell(\varset{s}^{(i)}) }$ denotes the indicator function on the constraint that corresponds to the literal $\ell$ and which is evaluated using the samples $\varset{s}^{(i)}$. 

For the negated literals we have the following labeling function:
\begin{align}
    &\alpha_{IALW}^{(i)}( \neg \ell) \nonumber \\
    &=\begin{cases}
    (1 , 0),  &  \text{if $\ell$ is a \probloginline{delta_interval} whose first argument } \\
                        & \text{is a continuous random variable} \\
    (1{-}\ive{ c_\ell(\varset{s}^{(i)}) }, 0), & \text{if $\ell$ is any other comparison atom}
    \\
    (1, 0), & \text{otherwise} 
    \end{cases}
    \nonumber
\end{align}
\end{definition}

Intuitively speaking and in the context of probabilistic inference,
the first part of an infinitesimal number accumulates (unnormalized) likelihood weights, while the second part counts the number of times we encounter a \probloginline{delta_interval} atom. This counting happens with $\oplus$ operation of the infinitesimal numbers (Equation~\ref{eq:infininumber_plus}). The $\oplus$ operation tells us that for two infinitesimal numbers $(r,n)$ and $(t,m)$ with $n<m$, the event corresponding to the first of the two infinitesimal numbers is infinitely more probable to happen and that we drop the likelihood weight of the second infinitesimal number (Equation~\ref{eq:infininumber_plus}). 
In other words, an event with fewer \probloginline{delta_interval}-atoms is infinitely more probable than an event with more such intervals.

\begin{example}[IALW Label of \probloginline{delta_interval} with Continuous Random Variable]
Let us consider a random variable \probloginline{x}, which is normally distributed: $p(\mathprobloginline{x}|\mu, \sigma)=\nicefrac{1}{(\sigma \sqrt{2 \pi})} \exp \left( - \nicefrac{(\mathprobloginline{x}-\mu)^2}{2 \sigma^2 } \right) $), 
where $\mu$ and $\sigma>0$ are real valued parameters that we can choose freely.
The atom \probloginline{delta_interval(x,3)} gets the label
$$
\left( \frac{1}{(\sigma \sqrt{2 \pi})} \exp \left( - \nicefrac{(\mathprobloginline{3}-\mu)^2}{2 \sigma^2 } \right)  , 1\right)
$$
The first element of the infinitesimal number is the probability distribution evaluated at the observation, in this case \probloginline{3}. As this is a zero probability event, the label also picks up a non-zero second element.

The label of $\neg \mathprobloginline{delta_interval(x,3)}$ is $(1,0)$. The intuition here being that the complement of an event with zero probability of happening will happen with probability $1$. As the complement event is not a zero probability event the second element of the label is $0$ instead of $1$. 
\end{example}

\begin{example}[IALW Label of \probloginline{delta_interval} with Discrete Random Variable]
    Let us consider a discrete random variable \probloginline{k}, which is Poisson distributed: 
    $$
    p(\mathprobloginline{k}|\lambda)=\nicefrac{\lambda^\mathprobloginline{k} e^{-\lambda}}{\mathprobloginline{k}!}
    $$
    where $\lambda>0$ is a real-valued parameter that we can freely choose.

    As a \probloginline{delta_interval} with a discrete random variable is equivalent to a \probloginline{=:=} comparison (\cf Definition~\ref{def:delta_interval}), we get for the label of the atom \probloginline{delta_interval(k,3)}:
    $(\ive{s_x^{(i)}=3}, 0)$, where $s_\mathprobloginline{k}^{(i)}$ is the $i$-th sample for \probloginline{k}.  
\end{example}


\begin{definition}[Infinitesimal Algebraic Likelihood Weighting]
\label{def:alw}
Let $\mathcal{S}$ be a set of ancestral samples and let $ DI(\varphi)$ denote the subset of literals in $\varphi$ that are delta intervals. We then define IALW as expressing the expected value of the label of a propositional formula (given a set of ancestral samples) in terms of a fraction of two infinitesimal numbers:
\begin{align}
    \left( \E \left[ \sum_{\varphi\in ENUM(\phi)} \prod_{\ell \in \varphi}  \alpha \left(\ell \right) \bigg| \mathcal{S} \right] ,0 \right)
    \approx
    \frac
    {\displaystyle \bigoplus_{i=1}^{\lvert \mathcal{S} \rvert}  \bigoplus_{\varphi\in ENUM(\phi)} \bigotimes_{\ell \in \varphi}  \alpha_{IALW}^{(i)} \left(\ell \right)}
    {\displaystyle \bigoplus_{i=1}^{\lvert \mathcal{S} \rvert} \bigoplus_{\varphi\in ENUM(\phi)} \bigotimes_{\ell \in  DI(\varphi)}  \alpha_{IALW}^{(i)} \left(\ell \right)  } \label{eq:ALW}
\end{align}
The left hand side expresses the expected value as an infinitesimal number.
\end{definition}

\begin{restatable}[Consistency of IALW]{proposition}{alwconsistency}
\label{prop:alw_consistency}
Infinitesimal algebraic likelihood weighting is consistent, that is, the approximate equality in Equation~\ref{eq:ALW} is almost surely an equality for $\lvert \mathcal{S} \rvert \rightarrow \infty$.
\end{restatable}
\begin{proof}
    See Appendix~\ref{app:proof:alw_consistency}.
\end{proof}

Likelihood weighting, the core idea behind IALW, is a well known technique for inference in Bayesian networks~\citep{fung1990weighing} and probabilistic programming~\citep{milch2005approximate,nitti2016probabilistic}, and falls within the broader class of self-normalized importance sampling~\citep{kahn1950random,kloek1978bayesian,casella1998post}.
Just like IALW, the inference approaches proposed by \citet{nitti2016probabilistic}, \citet{wu2018discrete}, and \citet{jacobs2021paradoxes}  generalize the idea of likelihood weighting to the setting with infinitesimally small intervals. What sets IALW apart from these methods is its semiring formulation. The semiring formulation will allow us to seamlessly combine IALW with knowledge compilation~\citep{darwiche2002knowledge}, a technique underlying state-of-the art probabilistic inference algorithms in the discrete setting. We examine this next.

Having proven the consistency of IALW, we can now express the probability of a conditional query to a \dcproblogsty program in terms of semiring operations for infinitesimal numbers $\mathbb{I}$.

\begin{restatable}{proposition}{alwapproximation}
\label{prop:alwapproximation}
A conditional probability query to a  \dcproblogsty program \dcpprogram can be approximated as: 
    \begin{align}
    P_\dcpprogram(\mu=q|\evidenceset=e) 
    \approx    
    \frac
    { \bigoplus_{i=1}^{\lvert \mathcal{S} \rvert}  \bigoplus_{\varphi\in ENUM(\phi \land \phi_{q})} \bigotimes_{\ell \in \varphi}  \alpha_{IALW}^{(i)} \left(\ell \right)}
    { \bigoplus_{i=1}^{\lvert \mathcal{S} \rvert}  \bigoplus_{\varphi\in ENUM(\phi)} \bigotimes_{\ell \in \varphi}  \alpha_{IALW}^{(i)} \left(\ell \right)}  
    \label{eq:prop:alwapproximation}
    \end{align}
    
\end{restatable}

\begin{proof}
    See Appendix~\ref{app:proof:alwapproximation}.
\end{proof}







\subsection{Infinitesimal Algebraic Likelihood Weighting via Knowledge Compilation}
\label{sec:ALWviaKC}

Inspecting Equation~\ref{eq:prop:alwapproximation} we see that we have to evaluate expressions of the following form in order to compute the probability of a conditional query to a \dcproblogsty program.
\begin{align}
    \bigoplus_{i=1}^{\lvert \mathcal{S} \rvert} \underbrace{\bigoplus_{\omega\in ENUM(\varphi)} \bigotimes_{\ell \in \varphi}  \alpha_{IALW}^{(i)} \left(\ell \right)}_{= \text{algebraic model count}} \label{eq:alw_show}
\end{align}
In other words, we need to compute $\lvert \mathcal{S} \rvert$ times a sum over products -- each time with a different ancestral sample. Such a sum over products is also called the algebraic model count of  a formula $\phi$~\citep{kimmig2017algebraic}. 
Subsequently, we then add up the $\lvert \mathcal{S} \rvert$ results from the different algebraic model counts giving us the final answer.

Unfortunately, computing the algebraic model count is in general a computationally hard problem~\citep{kimmig2017algebraic} -- \#P-hard to be precise~\citep{valiant1979complexity}.
A popular technique to mitigate this hardness is to use a technique called knowledge compilation~\citep{darwiche2002knowledge}, which splits up the computation into a hard step and a subsequent easy step. The idea is to take the propositional Boolean formula underlying 
an algebraic model counting problem (cf. $\varphi$ in Equation~\ref{eq:alw_show}) and compile it into a logically equivalent formula that allows for the tractable computation of algebraic model counts. The compilation constitutes the computationally hard part (\#P-hard). Afterwards, the algebraic model count is performed on the compiled structure, also called {\em algebraic circuit}~\citep{zuidbergdosmartires2019transforming}. Intuitively speaking, knowledge compilation takes the sum of products and maps it to recursively nested sums and products. Effectively, finding a dynamic programming scheme~\citep{bellman1957dynamic} to compute the initial sum of products.

Different circuit classes have been identified as valid knowledge compilation targets~\citep{darwiche2002knowledge} -- all satisfying different properties.
Computing the algebraic model count on an algebraic circuit belonging to a specific target class is only correct if the properties of the circuit class match the properties of the deployed semiring.
The following three lemmas will help us determining which class of circuits we need to knowledge-compile our propositional formula $\phi$ into.

\begin{lemma}
\label{lem:non_idem}
The operator $\oplus$ (c. Definition~\ref{def:inf_number_ops}) is not idempotent. That is, it does not hold for every $a \in \mathbb{I}$ that $a\oplus a =a$. 
\end{lemma}
\begin{lemma} The pair  $(\oplus, \alpha_{\ialw})$ is not neutral. That is, it does not hold that $\alpha_{\ialw}(\ell)\oplus \alpha_{\ialw}(\neg \ell) = e^{\otimes}$ for arbitrary $\ell$.
\end{lemma}
\begin{lemma} 
\label{lem:non_cons}
The pair  $(\otimes, \alpha_{\ialw})$ is not consistency-preserving. That is, it does not hold that $\alpha_{\ialw}(\ell)\otimes \alpha_{\ialw}(\neg \ell) = e^{\oplus}$ for arbitrary $\ell$.
\end{lemma}

From~\citep[Theorem 2 and Theorem 7]{kimmig2017algebraic} and the three lemmas above, we can conclude that we need to compile our propositional logic formulas into so-called smooth, deterministic and decomposable negation normal form (sd-DNNF) formulas~\citep{darwiche2001tractable}.\footnote{Note that we only require smoothness over derived atoms (otherwise case in Definition~\ref{def:sample_labeling_function}), as for the other cases the neutral sum property holds. Certain encodings of logic programs eliminate derived atoms. For such encodings the smoothness property can be dropped~\citep{vlasselaer2014compiling}. A more detailed discussion on the smoothness requirement of circuits in a PLP context can be found in \citep[Appendix C]{fierens2015inference}.}

\begin{restatable}[ALW on d-DNNF]{proposition}{alwonddnnf}
\label{prop:alwonddnnf}
    We are given the propositional formulas $\phi$ and $\phi_q$ and a set $\mathcal{S}$ of ancestral samples, we can use Algorithm~\ref{alg:prob_via_alw_kc} to compute the conditional probability $P_\dcpprogram(\mu=q|\evidenceset=e)$.
\end{restatable}

\begin{proof}
    See Appendix~\ref{app:proof:alwonddnnf}.
\end{proof}

Algorithm~\ref{alg:prob_via_alw_kc} takes as input a two propositional logic formulas $\phi$ and $\phi_q$, and a set of ancestral samples. It then knowledge-compiles the formulas $\phi \land \phi_q$ and $\phi$ into circuits $\Gamma_q$ and $\Gamma$. These circuits are then evaluated using Algorithm~\ref{alg:unormalize_alw}. The variables $ialw_q$ and $ialw$ hold infinitesimal numbers.
\fixed{
The ratio of these two numbers, which corresponds to the ratio in Equation~\ref{eq:prop:alwapproximation}, is an infinitesimal number having as second argument $0$ and as first argument the conditional probability.}
\begin{algorithm}[h]
    \SetKwFunction{ProbALW}{ProbALW}
    \SetKwFunction{KC}{KC}
    \SetKwFunction{IALW}{IALW}

    \SetKwProg{Fn}{function}{}{}
    \SetKwProg{ElseComment}{function}{}{}

    \caption{Conditional Probability via IALW and KC}
	\label{alg:prob_via_alw_kc}
\Fn{\ProbALW{$\phi$, $\phi_q$, $\mathcal{S}$ }}{
    $\Gamma_{q}$ $\leftarrow$ \KC{$\phi \land \phi_q$}\;
     \label{alg:prob_via_alw_kc:kc_qe} 
    $\Gamma$ $\leftarrow$ \KC{$\phi$}\;
    \label{alg:prob_via_alw_kc:kc_e}
    $ialw_{q}$ $\leftarrow$ \IALW{$\Gamma_{q}$,$\mathcal{S}$} \label{alg:prob_via_alw_kc:alw_qe}
    \tcp*[r]{cf. Algorithm~\ref{alg:unormalize_alw}} 
    $ialw$ $\leftarrow$ \IALW{$\Gamma$,$\mathcal{S}$}   \label{alg:prob_via_alw_kc:alw_e} 
    \tcp*[r]{cf. Algorithm~\ref{alg:unormalize_alw}}
    $(p,0)$ $\leftarrow$ $ialw_{q} \oslash ialw$\;
    \Return $p$ \label{alg:prob_via_alw_kc:conditional}
	}
\end{algorithm}

\begin{algorithm}[h]
    \SetKwFunction{IALW}{IALW}
    \SetKwFunction{Eval}{Eval}
    \SetKwProg{Fn}{function}{}{}
    \SetKwProg{ElseComment}{function}{}{}
    \caption{Computing the IALW}
	\label{alg:unormalize_alw}
\Fn{\IALW{$\Gamma$, $\mathcal{S}$}}{
    $ialw$ $\leftarrow$ (0,0) \;
    \For{$\varset{s}^{(i)} \in \mathcal{S}$  }{
        $ialw$ $\leftarrow$ $ialw$ $\oplus$ \Eval{$\Gamma$, $\varset{s}^{(i)}$}
        \tcp*[r]{cf. Algorithm~\ref{alg:eval}}  
    }
    \Return $ialw$
	}
\end{algorithm}

Algorithm~\ref{alg:unormalize_alw} computes the IALW given as input a circuit $\Gamma$ and a set of ancestral samples.
The loop evaluates the circuit (using Algorithm~\ref{alg:eval}) for each ancestral sample $\varset{s}^{(i)}$ and accumulates the result, which is then returned once the loop terminates. The accumulation inside the loop corresponds to the  $\bigoplus_{i=1}^{\lvert \mathcal{S} \rvert}$ summation in Equation~\ref{eq:alw_show}.
Algorithm~\ref{alg:eval} evaluates a circuit $\Gamma$ for a single ancestral sample $\varset{s}^{(i)}$ and is a variation of the circuit evaluation algorithm presented by~\citet{kimmig2017algebraic}.

\begin{algorithm}[h]
    \SetKwFunction{EvalFn}{Eval}
    \SetKwProg{Fn}{function}{}{}
    \SetKwProg{ElseComment}{function}{}{}

	\caption{Evaluating an sd-DNNF circuit $\Gamma$ for labeling function $\alpha^{(i)}$ (Definition~\ref{def:sample_labeling_function}) and semiring operations $\oplus$  and $\otimes$ (Definition~\ref{def:inf_number_ops})}
	\label{alg:eval}
\Fn{\EvalFn{$\Gamma$,$\varset{s}^{(i)}$}}{
		\If{  $\Gamma$ is a literal node $l$}{ \Return $\alpha^{(i)}(l)$}
		\ElseIf{$\Gamma$ is a disjunction $\bigvee_{j=1}^m \Gamma_j$}
		{\Return $\bigoplus_{j=1}^m$ \EvalFn{$\Gamma_j$, $\varset{s}^{(i)}$}}
		\Else( \tcp*[f]{$\Gamma$ is a conjunction $\bigwedge_{j=1}^m \Gamma_j$}){
		 \Return $\bigotimes_{i=j}^m$ \EvalFn{$\Gamma_j$, $\varset{s}^{(i)}$}}
	}
\end{algorithm}

\begin{example}[IALW on Algebraic Circuit]
\label{example:eval_observation}

Consider a version of the program in Example~\ref{example:dcproblog:observation} where the annotated disjunction has been eliminated and been replaced with a binary random variable $m$ and a \probloginline{flip} distribution.
	\begin{problog*}{linenos}
m~flip(0.3).

size~beta(2,3):- m=:=0.
size~beta(4,2):- m=:=1.
	\end{problog*}
We query the program for the conditional probability $P((\mathprobloginline{m=:=1}) =\top \mid \mathprobloginline{size}\doteq 4/10 )$.
Following the program transformations introduced in Section~\ref{sec:dc2smt} and then compiling the labeled propositional formula, we obtain a circuit representation of the queried program. Evaluating this circuit yields the probability of the query. To be precise, we actually obtain two circuits, one representing the probability of relevant program with the evidence enforced and with additionally having the value of the query atom set. 
In Figure~\ref{fig:circuit:ialw} we show the circuit where only the evidence is enforced.

	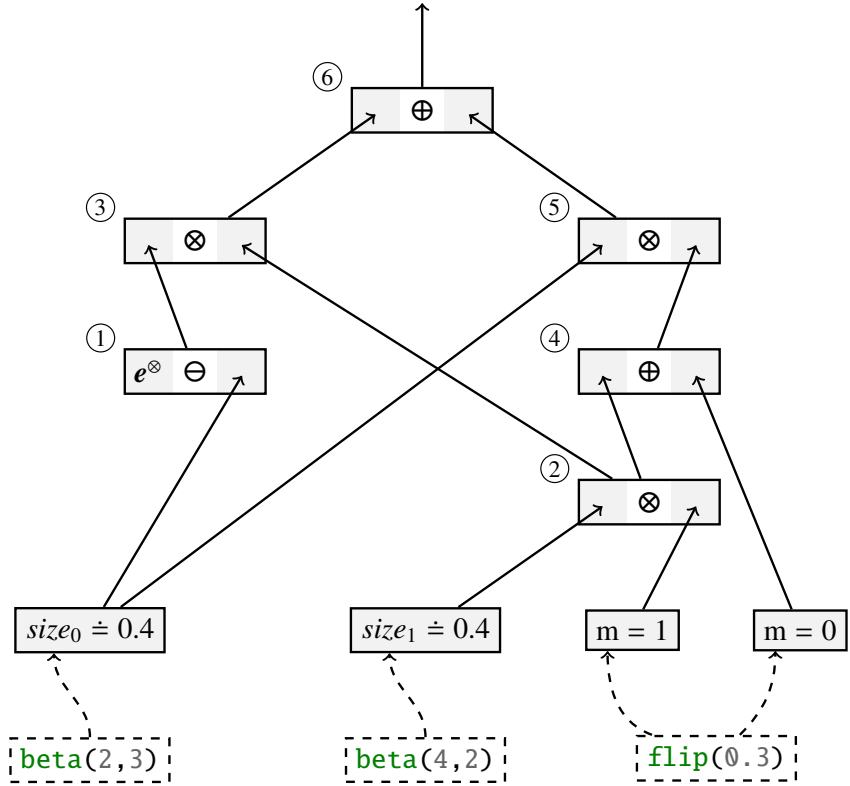
\begin{figure}[h]
	\resizebox{\linewidth}{!}{%

		\tikzstyle{distribution}=[rectangle, text centered, fill=white, draw, dashed,thick]
		\tikzstyle{leaf}=[rectangle, text centered, fill=gray!10, draw,thick]

		\tikzstyle{negate}=[
			rectangle split,
			rectangle split parts=3, 
			rectangle split horizontal,
			text centered,
			rectangle split part fill={gray!10,white,gray!10},
			draw,
			rectangle split draw splits=false,
			anchor=center,
			align=center,
			thick
		]
		\newcommand{\minus}{  ${\bm e^\otimes}$ \nodepart{second} ${{\bm \ominus}}$ \nodepart{third}  \phantom{${\bm e^\otimes}$}}

		\tikzstyle{sumproduct}=[
			rectangle split,
			rectangle split parts=3, 
			rectangle split horizontal,
			text centered,
			rectangle split part fill={gray!10,white,gray!10},
			draw,
			rectangle split draw splits=false,
			anchor=center,
			align=center,
			thick
		]
		\newcommand{\supr}[1]{  \phantom{${\bm e^\otimes}$} \nodepart{second} ${{\bm #1}}$ \nodepart{third} \phantom{${\bm e^\otimes}$}}
		
		\tikzstyle{circuitedge}=[ultra thick, thick,->]
		\tikzstyle{distributionedge}=[thick,->, dashed, in=-90]
		
		\tikzstyle{indexnode}=[draw,circle, inner sep=1pt]				
		
		\begin{tikzpicture}[remember picture]
			
			\node[sumproduct] (14) at (200.54bp,378.0bp) {\supr{\oplus}};
			\draw[ultra thick, thick,->] (14.90) to  ([shift={(0,1)}]14.90);
			
			\node[sumproduct] (9) [below left = of 14] {\supr{\otimes}};
			\node[sumproduct] (13) [below right = of 14]  {\supr{\otimes}};
			
			\node[negate] (m1) [below=of 9] {\minus};
			\node[sumproduct] (12)  [below  = of 13] {\supr{\oplus}};
			
			\node[sumproduct] (8)  [below=of 12]  {\supr{\otimes}};
			
			\node[indexnode, left=of 14, xshift=0.9cm, yshift=0.4cm] {\footnotesize {$6$}};			
			\node[indexnode, left=of 13, xshift=0.9cm, yshift=0.4cm] {\footnotesize {$5$}};			
			\node[indexnode, left=of 8, xshift=0.9cm, yshift=0.4cm] {\footnotesize {$2$}};			
			\node[indexnode, left=of 12, xshift=0.9cm, yshift=0.4cm] {\footnotesize {$4$}};			
			\node[indexnode, left=of 9, xshift=0.9cm, yshift=0.4cm] {\footnotesize {$3$}};			
			\node[indexnode, left=of m1, xshift=0.9cm, yshift=0.4cm] {\footnotesize {$1$}};

			\node[leaf, below= of 8, xshift=-0.2cm] (4)   {$\subnode{var_m0}{$m$}=1$};
			\node[leaf, below= of 8, xshift=1.8cm] (2) {$\subnode{var_m1}{$m$}=0$};				
			\node[leaf] (size1obs) [left=of 4, xshift=0cm]  {$\subnode{var_s_11_un}{size_{1}} \doteq 0.4$};
			\node[leaf] (size0obs) [left=of 4, xshift=-4cm] {$\subnode{var_s_01_un}{size_{0}} \doteq 0.4$};				

			\node[distribution] (size0)  [below=of size0obs] {\probloginline{beta(2,3)}};
			\node[distribution] (size1)  [below=of size1obs] {\probloginline{beta(4,2)}};
			\node[distribution] (flip)  [below= of 4, xshift=1cm] {\probloginline{flip(0.3)}};

			\draw[distributionedge,out=90]  (size1) to (var_s_11_un);
			\draw[distributionedge,out=90] (size0) to  (var_s_01_un);
			\draw[distributionedge,out=160] (flip) to  (var_m0);
			\draw[distributionedge] (flip) to  (var_m1);

			\draw[circuitedge] (9) to  (14.mid);
			\draw[circuitedge] (13) to  (14.three south |- 14.mid);
			\draw[circuitedge] (m1) to  (9.mid);
			\draw[circuitedge] (12) to  (13.three south |- 13.mid);
			\draw[circuitedge] (size0obs) to   (m1.three south |- m1.mid);	
			\draw[circuitedge] (2) to  (12.three south |- 12.mid);				
			\draw[circuitedge] (size1obs) to   (8.mid);
			\draw[circuitedge] (4) to  (8.three south |- 8.mid);	
			\draw[circuitedge] (8) to  (9.three south |- 9.mid);	
			\draw[circuitedge] (8) to   (12.mid);
			\draw[circuitedge] (size0obs) to  (13.mid);

		\end{tikzpicture}
	}
    \captionof{figure}{At the bottom of the circuit we see the distributions feeding in. The \probloginline{flip} distribution feeds into its two possible (non-zero probability) outcomes. The two \probloginline{beta} distributions feed into an observation statement each. We use the `$\doteq$' symbol to denote such an observation. Note how we identify each of the two random variables for the size by a unique identifier in their respective subscripts. The circled numbers next to the internal nodes, \ie the sum and product nodes, will allow us to reference the nodes later on and do not form a part of the algebraic circuit.}
    \label{fig:circuit:ialw}
\end{figure}

The probability of the query (given the evidence) can now be obtained by evaluating recursively the internal nodes in the algebraic circuit using Algorithm~\ref{alg:eval}.
We perform the evaluation  of the circuit in Figure~\ref{fig:circuit:ialw} for a single iteration of the loop in Algorithm~\ref{alg:unormalize_alw}, and we assume that we have sampled the value $m=0$ from the \probloginline{flip(0.3)} distribution.

\begin{minipage}{0.49\linewidth}
    \begin{align*}
        &\mathtt{Eval}(\footcircled{$1$})\\
        &=
        e^\otimes \ominus \alpha_{IALW}\big(size_0\doteq 0.4\big) \\
        &=
        (1,0) \ominus (1.728,1) \\
        &= (1,0)
        \\
        \hfill
        \\
        &\mathtt{Eval}(\footcircled{$2$})\\
        &=
        \alpha_{IALW}\big( size_1\doteq 0.4) \otimes \alpha_{IALW}\big(  m=1 \big) \\
        &=
        (0.768,1) \otimes (0,0) \\
        &=
        (0,1)\\
        \hfill
        \\
        &\mathtt{Eval}(\footcircled{$3$})\\
        &=
        \mathtt{Eval}(\footcircled{$1$}) \otimes \mathtt{Eval}(\footcircled{$2$}) \\
        &=
        (1,0) \otimes (0,1) \\
        &=
        (0,1)\\
    \end{align*}
\end{minipage}
\vline
\begin{minipage}{0.49\linewidth}
    \begin{align*}
        &\mathtt{Eval}(\footcircled{$4$})\\
        &=
        \mathtt{Eval}(\footcircled{$2$}) \oplus \alpha_{IALW}\big(  m=0 \big) \\
        &=
        (0,1) \oplus (1,0) \\
        &=
        (1,0)\\
        \hfill
        \\
        &\mathtt{Eval}(\footcircled{$5$})\\
        &=
        \alpha_{IALW}\big(  size_0\doteq 0.4 \big) \otimes \mathtt{Eval}(\footcircled{$2$})  \\
        &=
        (1.728,1) \otimes (1,0) \\
        &=
        (1.728,1)\\
        \hfill
        \\
        &\mathtt{Eval}(\footcircled{$6$})\\
        &=
        \mathtt{Eval}(\footcircled{$3$}) \oplus \mathtt{Eval}(\footcircled{$5$})  \\
        &=
        (0,1) \oplus (1.728,1) \\
        &=
        (1.728,1)\\
    \end{align*}
\end{minipage}

If we evalute the circuit for a sample $m=1$ we obtain in a similar fashion the result $\mathtt{Eval}(\footcircled{$6$})= (0.768,1)$. Moreover, if we evaluate the circuit multiple times we obtain (in the limit) 70\% of the time the outcome $(1.728,1)$ and 30\% of the time the value $(0.768,1)$. This yields an average of $(0.7\times 1.728, 1)\oplus (0.3\times 0.768, 1)= (1.440,1)$ and represents the unnormalized infinitesimal algebraic likelihood weight of the evidence.
The unnormalized infinitesimal algebraic likelihood weight of the query conjoined with the evidence is obtain again in a similar fashion but with the samples for $m=0$ being discarded. This then yields the result $(0.3\times 1.728, 1)$.
Dividing these two (unnormalized) infinitesimal algebraic likelihood weights by each other gives the probability of the query.
\begin{align*}
    &P((\mathprobloginline{m=:=1}) =\top \mid \mathprobloginline{size}\doteq 4/10 )\\
    &=(0.3\times 1.728, 1) \oslash \Big( (0.7\times 0.768,1 ) \oplus (0.3\times 1.728,1) \Big) \\
    &= (0.2304/1.440 , 1{-}1) \\
    &= (0.16,0)
\end{align*}
\end{example}

\subsection{Partial Symbolic Inference}

Evaluating circuits using binary random variables is quite wasteful: on average half of the samples are unused for one of the two possible outcomes ($0$ or $1$). We can remedy this by performing (exact) symbolic inference on binary random variables and replace the comparisons where they appear with their expectation. For instance, we replace \probloginline{m=:=1} by the infinitesimal number $(0.3,0)$ instead of sampling a value for \probloginline{m} and testing whether the sample satisfies the constraint. This technique is also used by other probabilistic programming languages such as \problogsty~\citep{fierens2015inference} and Dice~\citep{holtzen2020dice}. The main difference to \dcproblogsty is that those languages only support binary random variables (and by extension discrete random variables with finite support), while \dcproblogsty interleaves discrete and continuous random variables.

In a sense, the expectation gets pushed from the root of the algebraic circuit representing a probability to its leaves. This is, however, only possible if the circuit respects specific properties. Namely, the ones respected by \mbox{d-DNNF} formulas (cf. Section~\ref{sec:ALWviaKC}), which we use as our representation language for the probability.

\begin{definition}[Symbolic IALW Label of a Literal] \label{def:sample_probability_labeling_function}
Given an ancestral sample $\varset{s}^{(i)}= (s_1^{(i)}, \dots,  s_M^{(i)} ) $ for the random variables $\randomvariableset = (\nu_1,\dots, \nu_M)$.
The  Symbolic IALW (SIALW) label of a positive literal $\ell$ is an infinitesimal number given by:
\begin{align}
    \alpha_{SIALW}^{(i)}( \ell)
    =\begin{cases}
    (p_{\ell}, 0), & \text{if $\ell$ encodes a probabilistic fact} \\
    \alpha^{(i)}_{IALW}(\ell), & \text{otherwise} 
    \end{cases}
    \nonumber
\end{align}
For the negated literals we have the following labeling function:
\begin{align}
    \alpha_{SIALW}^{(i)}( \neg \ell)
    =\begin{cases}
    (1{-}p_{\ell}, 0), & \text{if $\ell$ encodes a probabilistic fact} \\
    \alpha^{(i)}_{IALW}(\neg \ell), & \text{otherwise}
    \end{cases}
    \nonumber
\end{align}
The number $p_\ell$ is the label of the probabilistic fact in a \dcproblogsty program.
\end{definition}

In the definition above we replace the label of a comparison that corresponds to a probabilistic fact with the probability of that fact being satisfied. This has already been shown to be beneficial when performing inference, both in terms of inference time and accuracy of Monte Carlo estimates~\citep{zuidbergdosmartires2019exact}. Following the work of~\citep{kolb2019exploit} one could also develop more sophisticated methods to detect which comparison in the leaves can be replaced with their expectation. We leave this for future work.

\begin{example}[Symbolic IALW on Algebraic Circuit]
\label{example:eval_observation_marginalized}

Symbolic inference for the random variable $m$ from the circuit in Example~\ref{example:eval_observation} results in annotating the leaf nodes for the different outcomes of the random variable $m$ with the probabilities of the respective outcomes. This can be seen in the red dashed box in the bottom right of Figure~\ref{fig:circuit:sialw}.

	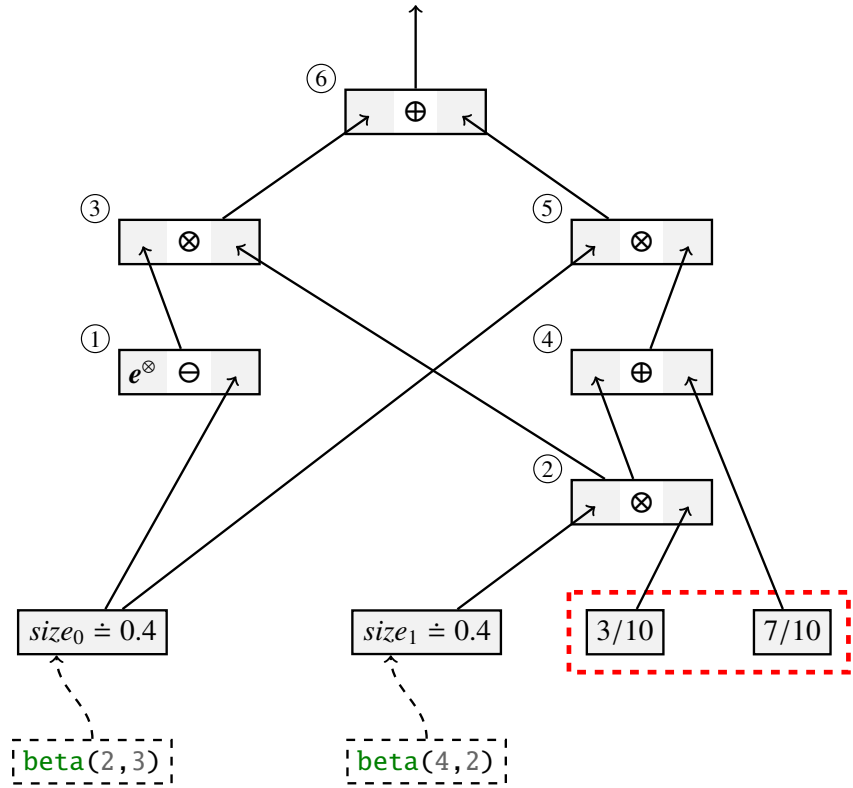
\begin{figure}[h]
		\resizebox{\linewidth}{!}{%
			
			\tikzstyle{distribution}=[rectangle, text centered, fill=white, draw, dashed,thick]
			\tikzstyle{leaf}=[rectangle, text centered, fill=gray!10, draw,thick]

			\tikzstyle{negate}=[
			rectangle split,
			rectangle split parts=3, 
			rectangle split horizontal,
			text centered,
			rectangle split part fill={gray!10,white,gray!10},
			draw,
			rectangle split draw splits=false,
			anchor=center,
			align=center,
			thick
			]
			\newcommand{\minus}{  ${\bm e^\otimes}$ \nodepart{second} ${{\bm \ominus}}$ \nodepart{third}  \phantom{${\bm e^\otimes}$}}
			
			\tikzstyle{sumproduct}=[
			rectangle split,
			rectangle split parts=3, 
			rectangle split horizontal,
			text centered,
			rectangle split part fill={gray!10,white,gray!10},
			draw,
			rectangle split draw splits=false,
			anchor=center,
			align=center,
			thick
			]
			\newcommand{\supr}[1]{  \phantom{${\bm e^\otimes}$} \nodepart{second} ${{\bm #1}}$ \nodepart{third} \phantom{${\bm e^\otimes}$}}
			
			\tikzstyle{circuitedge}=[ultra thick, thick,->]
			\tikzstyle{distributionedge}=[thick,->, dashed, in=-90]
			
	    	\tikzstyle{indexnode}=[draw,circle, inner sep=1pt]				
			
			\begin{tikzpicture}[remember picture]
				
				\node[sumproduct] (14) at (200.54bp,378.0bp) {\supr{\oplus}};
				\draw[ultra thick, thick,->] (14.90) to  ([shift={(0,1)}]14.90);
				
				\node[sumproduct] (9) [below left = of 14] {\supr{\otimes}};
				\node[sumproduct] (13) [below right = of 14]  {\supr{\otimes}};
				
				\node[negate] (m1) [below=of 9] {\minus};
				\node[sumproduct] (12)  [below  = of 13] {\supr{\oplus}};
				
				\node[sumproduct] (8)  [below=of 12]  {\supr{\otimes}};
				
    			\node[indexnode, left=of 14, xshift=0.9cm, yshift=0.4cm] {\footnotesize {$6$}};			
    			\node[indexnode, left=of 13, xshift=0.9cm, yshift=0.4cm] {\footnotesize {$5$}};			
    			\node[indexnode, left=of 8, xshift=0.9cm, yshift=0.4cm] {\footnotesize {$2$}};			
    			\node[indexnode, left=of 12, xshift=0.9cm, yshift=0.4cm] {\footnotesize {$4$}};			
    			\node[indexnode, left=of 9, xshift=0.9cm, yshift=0.4cm] {\footnotesize {$3$}};			
    			\node[indexnode, left=of m1, xshift=0.9cm, yshift=0.4cm] {\footnotesize {$1$}};

				\node[leaf, below= of 8, xshift=-0.2cm] (4)   {$3/10$};
				\node[leaf, below= of 8, xshift=1.8cm] (2) {$7/10$};
				\draw[red,ultra thick,dashed] ($(4.north west)+(-0.2,0.2)$)  rectangle ($(2.south east)+(0.2,-0.2)$);

				\node[leaf] (size1obs) [left=of 4, draw,  xshift=0cm]  {$\subnode{var_s_11}{size_{1}} \doteq 0.4$};
				\node[leaf]  (size0obs) [left =of 4, xshift=-4cm] {$\subnode{var_s_01}{size_{0}} \doteq 0.4$};				
				
		    	\node[distribution] (size0)  [below=of size0obs] {\probloginline{beta(2,3)}};
		    	\node[distribution] (size1)  [below=of size1obs] {\probloginline{beta(4,2)}};

				\draw[distributionedge,out=90]  (size1) to (var_s_11);
				\draw[distributionedge,out=90] (size0) to  (var_s_01);

				\draw[circuitedge] (9) to  (14.mid);
				\draw[circuitedge] (13) to  (14.three south |- 14.mid);
				\draw[circuitedge] (m1) to  (9.mid);
				\draw[circuitedge] (12) to  (13.three south |- 13.mid);
				\draw[circuitedge] (size0obs) to   (m1.three south |- m1.mid);	
				\draw[circuitedge] (2) to  (12.three south |- 12.mid);				
				\draw[circuitedge] (size1obs) to   (8.mid);
				\draw[circuitedge] (4) to  (8.three south |- 8.mid);	
				\draw[circuitedge] (8) to  (9.three south |- 9.mid);	
				\draw[circuitedge] (8) to   (12.mid);
				\draw[circuitedge] (size0obs) to  (13.mid);
				
			\end{tikzpicture}
		}
        \captionof{figure}{Circuit representation of the SIALW algorithm for the probability $P(\mathprobloginline{size}\doteq 4/10 )$.}
        \label{fig:circuit:sialw}

	\end{figure}
	
Evaluating the marginalized circuit now returns immediately the unnormalized algebraic model count for the evidence without the need to draw samples and consequently without the need to sum over the samples. 

\begin{minipage}{0.49\linewidth}
    \begin{align*}
        &\mathtt{Eval}(\footcircled{$1$})\\
        &=
        e^\otimes \ominus \alpha_{SIALW}\big(size_0\doteq 0.4\big) \\
        &=
        (1,0) \ominus (1.728,1) \\
        &= (1,0)
        \\
        \hfill
        \\
        &\mathtt{Eval}(\footcircled{$2$})\\
        &=
        \alpha_{SIALW}\big( size_1\doteq 0.4) \otimes \alpha_{SIALW}\big(  m=1 \big) \\
        &=
        (0.768,1) \otimes (0.3,0) \\
        &=
        (0.2304,1)\\
        \hfill
        \\
        &\mathtt{Eval}(\footcircled{$3$})\\
        &=
        \mathtt{Eval}(\footcircled{$1$}) \otimes \mathtt{Eval}(\footcircled{$2$}) \\
        &=
        (1,0) \otimes (0.2304,1) \\
        &=
        (0.2304,1)\\
    \end{align*}
\end{minipage}
\vline
\begin{minipage}{0.49\linewidth}
    \begin{align*}
        &\mathtt{Eval}(\footcircled{$4$})\\
        &=
        \mathtt{Eval}(\footcircled{$2$}) \oplus \alpha_{SIALW}\big(  m=0 \big) \\
        &=
        (0.2304,1) \oplus (0.7,0) \\
        &=
        (0.7,0)\\
        \hfill
        \\
        &\mathtt{Eval}(\footcircled{$5$})\\
        &=
        \alpha_{SIALW}\big(  size_0\doteq 0.4 \big) \otimes \mathtt{Eval}(\footcircled{$2$})  \\
        &=
        (1.728,1) \otimes (0.7,0) \\
        &=
        (1.2096,1)\\
        \hfill
        \\
        &\mathtt{Eval}(\footcircled{$6$})\\
        &=
        \mathtt{Eval}(\footcircled{$3$}) \oplus \mathtt{Eval}(\footcircled{$5$})  \\
        &=
        (0.2304,1 \oplus (1.2096,1) \\
        &=
        (1.440,1)\\
    \end{align*}
\end{minipage}

\end{example}

\subsection{Experimental Evaluation}
\label{sec:experimental}

\new{

In order to demonstrate the benefits of adapting the technique of knowledge compilation to the discrete-continuous domain with zero-probability conditioning events, we model a machine that runs either in operating \probloginline{mode1} or \probloginline{mode2}; (with probability $0.2$ and $0.8$ respectively). Furthermore, the machine can be faulty with a small probability of $10^{-5}$. In this case we would like to switch the machine of and repair it.

If the machine is not faulty the temperature measurements we perform on the machine are distributed according to two Gaussian (Lines~\ref{line:ex:faulty:temp1} and \ref{line:ex:faulty:temp2}).
If the machine is faulty, however, we get a deterministic temperature reading of $2.0$ (Line \ref{line:ex:faulty:temp3}).

}

\begin{problog*}{linenos}
0.00001::faulty.
0.2::mode1;0.7::mode2.

temperature ~ normal(0.5,1.0) :- \+faulty, mode1. @\label{line:ex:faulty:temp1}@
temperature ~ normal(2.0,2.0) :- \+faulty, mode2. @\label{line:ex:faulty:temp2}@
temperature ~ delta(2.0) :- faulty. @\label{line:ex:faulty:temp3}@
\end{problog*}

\new{
We are now interested in computing $p( \mathprobloginline{faulty}{=}\top {\mid} \mathprobloginline{temperature}{\doteq} 2.0 )$. That is, what is the probability that the machine is faulty given that the temperature measurement is $2.0$.

In our experiment we compared the performance of SIALW (\cf Definition~\ref{def:sample_probability_labeling_function}) to the inference algorithm of \dcsty~\citep{nitti2016probabilistic}. The latter is equivalent to the algorithms presented by \Citet{wu2018discrete} and \Citet{jacobs2021paradoxes} as all three perform, in essence, likelihood weighting with infinitesimal numbers.   
Specifically, we study the sensitivity of the algorithms with regard to the probability of the machine being faulty.

In Figure~\ref{fig:experiment:faultVScorrectness} we see that using SIALW computes the correct probability regardless of the fault probability $\{10^{-5}, 10^{-4}, 10^{-3}, 10^{-2}, 10^{-1} \}$. We also see that the naivc likelihood weighting algorithm (without the exact symbolic inference of SIALW) needs a substantial amount of samples to infer the correct probability. Most notably, for a fault probability of $10^{-5}$ not even a sample size of $10^5$ is sufficient.

The large discrepancy between SIALW and the competing approach by \Citet{nitti2016probabilistic} is explained as follows: in order to correctly infer the queried probability one of the samples drawn from the Bernoulli distribution \probloginline{faulty ~ flip(0.00001)} needs to be true. This would then trigger the rule for \probloginline{temperature ~ delta(2.0) :- faulty}. As this is, however, extremely unlikely the crucial rule needed to perform correct likelihood weighting with infinitesimal numbers is never triggered and the returned probability is incorrect.
SIALW, in constrast, does not sample \probloginline{faulty ~ flip(0.00001)} but performs exact inference using knowledge compilation. As a result SIALW always computes the correct posterior probability.

\begin{figure}
    \begin{center}
        \includegraphics[width=\linewidth]{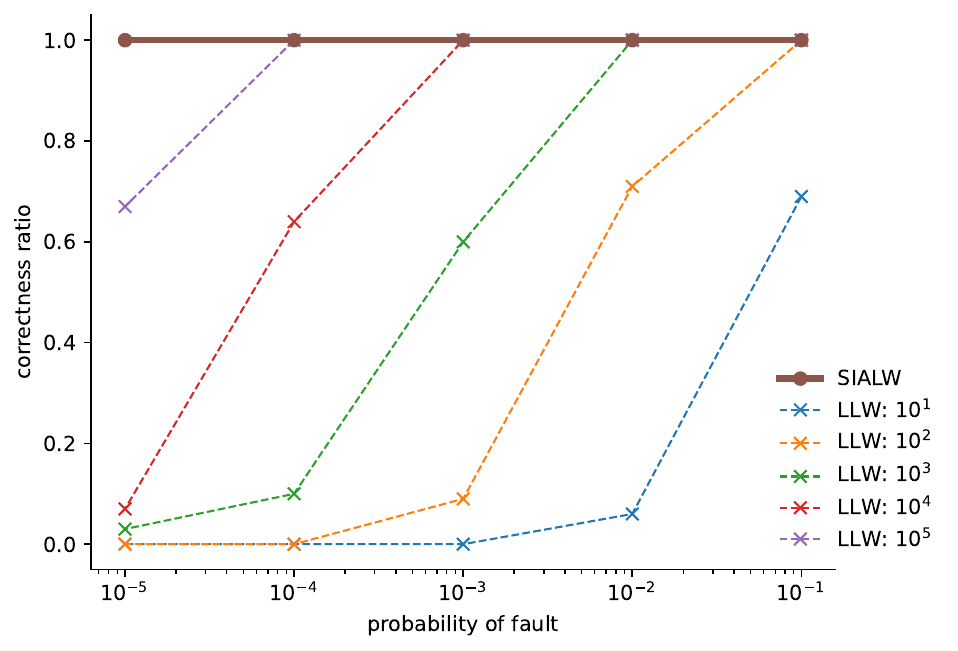}
    \end{center}
    \caption{
        \new{
        We queried SIALW and (non-symbolci) likelihood weighting each $100$ times for the probability $p( \mathprobloginline{faulty}{=}\top {\mid} \mathprobloginline{temperature}{\doteq} 2.0 )$. On the $y$-axis we give the ratio $\nicefrac{\text{\#correct runs}}{\text{\#runs}}$.
        Die to the use of knowledge compilation, SIALW is insensitive to the probability of fault (on the x-axis).
        This is in contrast to the log-likelihood weighting (LLW) algorithm presented by~\Citet{nitti2016probabilistic}, which necessitates a considerable number of samples to compute the queried probability reliably. The different dotted lines indicate settings with varying sample sizes ($\{10^1, 10^2, 10^3, 10^4, 10^5 \}$).  
        }
    }
    \label{fig:experiment:faultVScorrectness}
\end{figure}
}


%% file: files_main/related_work.tex
\section{\dcproblogsty and the Probabilistic Programming Landscape}\label{sec:related}

In recent years a plethora of different probabilistic programming languages have been developed. 
We discuss these by pointing out key features present in \dcproblogsty (listed below), which are missing in specific related works.
We organize these features along the three key contributions stated in Section \ref{sec:introduction}. Our first key contribution is the introduction of the measure semantics with the following features:

\begin{enumerate}[label=C1.\arabic*, leftmargin=2\parindent]
    \item \new{possibly infinite number (even uncountable) number of random variables} 
    \label{f:infinite_number_of rv}
    \item random variables with (possibly) infinite sample spaces
    \label{f:infinite_samplespaces}
    \item functional dependencies between random variables
    \label{f:rv_dependencies}
    \item uniform treatment of discrete and continuous random variables
    \label{f:dc_rv_uniform}
    \item negation
    \label{f:declarative_negation}
\end{enumerate}

\noindent Our second contribution is the  introduction of  the \dcproblogsty language, which  
\begin{enumerate}[label=C2.\arabic*, leftmargin=2\parindent]
    \item has purely discrete PLPs and their semantics as a special case,
    \label{f:discrete_special}
    \item supports a rich set of comparison predicates, and
    \label{f:comparison_predicates}
    \item is a Turing complete language (\dcplpsty)
    \label{f:turing_complete}
\end{enumerate}

\noindent Our last contributions concern inference, which includes 
\begin{enumerate}[label=C3.\arabic*, leftmargin=2\parindent]
    \item a formal definition of the hybrid probabilistic inference task,
    \label{f:definition_inference}
    \item an inference algorithm called IALW, 
    \label{f:inference_algorithm}
    \item that uses standard knowledge compilation in the hybrid domain.
    \label{f:inference_kc}
\end{enumerate}

\subsection{\problogsty and \dcsty}

The \dcproblogsty language is a generalization of \problogsty, both in terms of syntax and semantics. A \dcproblogsty program that does not use distributional clauses (or distributional facts) is also a \problogsty program, and both define the same distribution over the logical vocabulary of the program. \dcproblogsty properly generalizes \problogsty to include random variables with infinite sample spaces (\ref{f:infinite_samplespaces}). 

On a syntactical level, \dcproblogsty is closely related to the \dcsty (DC) language, with which it shares the  \predicate{~}{2} predicate used in infix notation. 
In Appendix~\ref{sec:dcproblog-dc} we discuss in more detail the relationship between \dcproblogsty and the \dcsty language. Concretely, we point out that 
\dcproblogsty generalizes the original and negation-free version of DC~\citep{gutmann2011magic} (\ref{f:declarative_negation}).
However, \dcproblogsty differs in its declarative interpretation of negation from the procedural interpretation as introduced to DC by~\citet{nitti2016probabilistic}.
As a consequence, the semantics of DC and \problogsty differ in the absence of continuous random variables, while \dcproblogsty is a strict generalization of \problogsty (\ref{f:discrete_special}).

\subsection{Bayesian Logic Programs}
\label{sec:blp}

\new{
Bayesian logic programs (BLPs)~\citep{kersting2000bayesian} can be seen as a special case of \dcplpsty: while the semantics of \dcplpsty allow for a (possibly uncountable) infinite number of random variables~\ref{f:infinite_number_of rv}, BLPs are limited to finite distributional databases (expressed as Bayesian networks).

Moreover, using the construct of distributional clauses we introduce syntax to interleave logic programming statements and the declaration of the distributional database. This is not supported in the BLP language.
Lastly, IAWL equips \dcproblogsty with a sound inference algorithm for the discrete-continuous space. This is again in contrast to BLP's inference algorithm that only handles (discrete mixtures of) continuous random variables. 
}

\subsection{\extendedprismsty}

An early attempt of equipping a probabilistic logic programming language with continuous random variables can be found in~\citep{islam2012inference}, which was dubbed {\em \extendedprismsty}. 
Similar to \dcproblogsty, \extendedprismsty's semantics are based again on \citeauthor{sato1995statistical}'s distribution semantics.
However, \extendedprismsty assumes, just like \dcsty , pairwise mutually exclusive proofs (we refer again to Appendix~\ref{sec:dcproblog-dc} for details on this).
On the expressivity side, \extendedprismsty only supports linear equalities -- in contrast to \dcproblogsty, where also inequalities are included in the semantics of the language (\ref{f:comparison_predicates}). 
An advantage of restricting possible constraints to equalities is the possibility of performing exact symbolic inference. In this regard, \extendedprismsty, together with its symbolic inference algorithm, can be viewed as a logic programming language that has access to a computer algebra system. Swapping out the approximate Sampo-inspired inference algorithm in \dcproblogsty by an exact inference algorithm using symbolic expression manipulations would result in an inference approach closely related to that of \extendedprismsty. One possibility would be to use the Symbo algorithm presented in~\citep{zuidbergdosmartires2019exact}, which uses the PSI-language~\citep{gehr2016psi} as its (probabilistic)  computer algebra system.

\subsection{Probabilistic Constraint Logic Programming}

Impressive work on extending probabilistic logic programs with continuous random variables was presented by~\citet{michels2015new} with the introduction of Probabilistic Constraint Logic Programming (PCLP). The semantics of PCLP are again based on \citeauthor{sato1995statistical}'s distribution semantics and the authors also presented an approximate inference algorithm for hybrid probabilistic logic programs.
Interestingly, the algorithm presented in~\citep{michels2015new} to perform (conditional) probabilistic inference extends weighted model counting to continuous random variables using imprecise probabilities, and more specifically credal sets.

A shortcoming of PCLP's semantics is the lack of direct support for generative definitions of random variables, \ie, random variables can only be interpreted within constraints, but not within distributions of other random variables as is possible in \dcproblogsty (\ref{f:rv_dependencies}).
\citet{azzolini2021semantics} define a non-credal version of this semantics using a product measure over a space that explicitly separates discrete and continuous random variables, assuming that a measure over the latter is given as part of the input without further discussion of how this part of the measure is specified in a program. Furthermore, they do not define any inference tasks (\ref{f:definition_inference}), \eg computing conditional probabilities (\cf Section~\ref{sec:inference-tasks}), nor
do they provide an inference algorithm (\ref{f:inference_algorithm}).

A later proposal for the syntax of such programs~\citep{azzolini:iclp21} combines two classes of terms (logical and continuous ones) with typed predicates and functors, and defines mixture variables as well as arithmetic expressions over random variables through logical clauses. In other words, user-defined predicates define families of random variables through the use of typed arguments of the predicate identifying a specific random variable, arguments providing parameters for the distribution, and one argument representing the random variable itself.
In contrast, the syntax of \dcproblogsty clearly identifies all  random variables through explicit terms introduced through distributional facts or distributional clauses, explicitly exposes the probabilistic dependency structure by using random variable terms inside distribution terms, and avoids typing through argument positions.
Moreover, \dcproblogsty takes a uniform view on all random variables in terms of semantics, thereby avoiding treating discrete and continuous random variables separately (\ref{f:dc_rv_uniform}).

\subsection{\blogsty}

Notable in the domain of probabilistic logic programming is also the BLOG language~\citep{milch2005blog,wu2018discrete}. Contrary to the aforementioned probabilistic logic programming languages, BLOG's semantics are not specified using \citeauthor{sato1995statistical}'s distribution semantics but via so-called {\em measure-theoretic Bayesian networks} (MTBN), which were introduced in~\citep{wu2018discrete}. MTBNs can be regarded as the assembly language for BLOG: every BLOG program is translated or compiled to an MTBN.
With \dcproblogsty we follow a similar pattern: every \dcproblogsty program with syntactic sugar (\eg annotated disjunctions)  is transformed into \dfplpsty program. The semantics are defined on the bare-bones program. Note that the assembly language for \dcproblogsty (\dfplpsty) is Turing complete. This is not the case for MTBNs (\ref{f:turing_complete}).


\subsection{Non-logical Probabilistic Programming}

As first pointed out by~\citet{russell2015unifying} and later on elaborated upon by~\citet{kimmig2017probabilistic}, probabilistic programs fall either into the {\em possible worlds semantics} category or the {\em probabilistic execution traces semantics} category. The former is usually found in logic based languages, while the latter is the prevailing view in imperative and functional probabilistic languages.

While, the probabilistic programming languages discussed so far follow the possible worlds paradigm,
many languages follow the execution traces paradigm, either as a probabilistic functional language~\citep{goodman2008church,wood2014approach} or as a imperative probabilistic language~\citep{gehr2016psi,salvatier2016probabilistic,carpenter2017stan,bingham2019pyro,ge2018turing}. Generally speaking, functional and imperative probabilistic programming languages target first and foremost continuous random variables, and discrete random variables are only added as an afterthought. A notable exception is the functional probabilistic programming language Dice~\citep{holtzen2020dice}, which targets discrete random variables exclusively.

Concerning inference in probabilistic programming, we can observe general trends in logical and non-logical probabilistic languages. While the latter are interested in adapting and speeding up approximate inference algorithms, such as Markov chain Monte Carlo sampling schemes or variational inference, the former type of languages are more invested in exploiting independences in the probabilistic programs, mainly by means of knowledge compilation. Clearly, these trends are not strict. For instance, \citet{obermeyer2019functional} proposed so-called {\em funsors} to express and exploit independences in \pyrosty~\citep{bingham2019pyro}, an imperative probabilistic programming language, and \citet{gehr2016psi} developed a computer algebra system to perform exact symbolic probabilistic inference.

\subsection{Representation of Probabilistic Programs at Inference Time}

Lastly, we would like to point out a key feature of the IALW inference algorithm that sets it apart from any other inference scheme for probabilistic programming in the hybrid domain. But first, let us briefly talk about computing probabilities in probabilistic programming. Roughly speaking, probabilities are computed summing and multiplying weights. These can for example be represented as floating point numbers or symbolic expressions. The collection of all operations that were performed to obtain the probability of a query to a program is called the computation graph. Now, the big difference between IALW and other inference algorithms lies in the structure of the computation graph. IALW represents the computation graph as a directed acyclic graph (DAG), while all other languages, except some purely discrete languages~\citep{fierens2015inference,holtzen2020dice}, use a tree representation. IALW is the first inference algorithm in the discrete-continuous domain that uses DAGs (\ref{f:inference_kc}). In cases where the computation graph can be represented as a DAG the size of the representation might be exponentially smaller compared to tree representations, which leads to faster inference times.

Note that \citet{gutmann2010extending} and more recently~\citet{saad2021sppl} presented implementations of hybrid languages where the inference algorithm leverages directed acyclic graphs, as well. However, the constraints that may be imposed on random variables are limited to univariate equalities and inequalities. In the weighted model integration literature it was shown that such probability computations can be mapped to probability computations of discrete random variables only~\citep{zeng2019efficient}.

\subsection{Probabilistic Neurosymbolic AI}

\new{
As noted by \Citet{desmet2023neural} a shortcoming of many neurosymbolic AI systems~\citep{garcez2019neural,marra2024statistical}, \ie systems that combine the function approximation power of neural networks with logic reasoning, is their restriction to only allowing discrete random variables.
Based on the semantics of \dcplpsty, \citet{desmet2023neural} extended distributional facts to so-called neural-distributional facts. 
That is, they allowed for neural networks to estimate the parameters of the distribution in the distributional fact.
Importantly, they showed that endowing a neurosymbolic system in the discrete-continuous domain with proper probabilistic semantics is advantageous when comparing to systems that exhibit, for instance, a fuzzy logic semantics~\citep{badreddine2022logic}.
}

%% file: files_main/conclusions.tex
\section{Conclusions}\label{sec:conclusions}

We introduced \dcproblogsty, a hybrid PLP language for the discrete-continuous domain and its accompanying measure semantics.
\dcproblogsty strictly extends the discrete \problogsty language \citep{de2007problog,fierens2015inference} and the negation-free \dcsty~\citep{gutmann2011magic} language.
In designing the language and its semantics we adapted ~\citet{poole2010probabilistic}'s design principle of percolating probabilistic logic programs into two separate layers: the random variables and the logic program.
 Boolean comparison atoms then form the  link between the two layers.
It is this clear separation between the random variables and the logic program that has allowed us to use simpler language constructs and to write programs using a  more concise and intuitive syntax than alternative hybrid PLP approaches \citep{gutmann2010extending,nitti2016probabilistic,speichert2019learning,azzolini2021semantics}.

Separating random variables from the logic program also allowed us to develop the IALW algorithm to perform inference in the hybrid domain. 
IALW is the first algorithm based on knowledge compilation and algebraic model counting for hybrid probabilistic programming languages
and as such it generalizes the standard knowledge compilation based approach for PLP.
It is noteworthy that IALW correctly computes conditional probabilities in the discrete-continuous domain using the newly introduced infinitesimal numbers semiring.

Interesting future research directions include adapting ideas from functional probabilistic programming (the other declarative programming style besides logic programming) in the context of probabilistic logic programming. For instance, extending \dcproblogsty with a type system~\citep{Schrijvers2008TowardsTP} or investigating more recent advances, such as {\em quasi-Borel spaces}~\citep{heunen2017convenient} in the context of the measure semantics.

%% file: files_appendix2/lp_new.tex
\section{Logic Programming}
\label{app:lp_new}

We briefly summarize key concepts of the syntax and semantics of logic programming; for a full introduction, we refer to \citep{lloyd2012foundations}.

\subsection{Building Blocks}
The basic building blocks of logic programs are \emph{variables} (denoted by strings starting with upper case letters), \emph{constants}, \emph{functors} and
\emph{predicates} (all denoted by strings starting with lower case letters).  A \emph{term} is a variable, a constant, or a functor
$f$ of \emph{arity} $n$ followed by $n$ terms $t_i$, \ie,
$f(t_1,\dots,t_n)$. 
An \emph{atom} is a predicate $p$ of arity $n$ followed by $n$ terms $t_i$, \ie,
$p(t_1,\dots,t_n)$. A predicate $p$ of arity $n$ is also written as \mathpredicate{p}{n}. A \emph{literal} is an
atom or a negated atom $ not(p(t_1,\dots,t_n))$.

\subsection{Logic Programs}

A \emph{definite clause} is a universally quantified expression of the form $h
\lpif b_1, \dots, b_n$ where $h$ and the $b_i$ are atoms.
$h$ is called the \emph{head} of the clause, and $b_1, \dots, b_n$ its
\emph{body}. Informally, the meaning of such a clause is that if all
the $b_i$ are true, $h$ has to be true as well. 
 A \emph{normal clause}  is a universally quantified expression of the form $h
\lpif l_1, \dots, l_n$ where $h$ is an atom and the $l_i$ are
literals.
If $n=0$, a clause is called \emph{fact} and simply written
as $h$.
A \emph{definite clause program} or \emph{logic program} for
short is a finite set of definite clauses. A \emph{normal logic
  program} is a finite set of normal clauses. 

\subsection{Substitutions}

A \emph{substitution} $\theta$ is an expression of the form
$\{V_1/t_1,\dots,V_m/t_m\}$ where the $V_i$ are different variables and
the $t_i$ are terms. Applying a substitution $\theta$ to an expression
$e$ (term or clause) yields the \emph{instantiated} expression $e\theta$
where all variables $V_i$ in $e$ have been simultaneously replaced by
their corresponding terms $t_i$ in $\theta$. If an expression does not
contain variables it is \emph{ground}. Two expressions $e_1$ and $e_2$ can be \emph{unified} if and only if there are substitutions $\theta_1$ and $\theta_2$ such that $e_1\theta_1 = e_2\theta_2$.

\subsection{Herbrand Universe}

The \emph{Herbrand universe} of a logic program is the set of ground terms that can be constructed using the functors and constants occurring in the program.
The \emph{Herbrand base} of a logic program is the set of ground atoms that can be constructed from the predicates in the program and the terms in its Herbrand universe. 
A truth value assignment to all atoms in the 
Herbrand base is called \emph{Herbrand interpretation}, and is also represented as the set of atoms that are true according to the assignment. 
A Herbrand interpretation is a \emph{model} of a clause $h \lpif b_1,\ldots ,b_n\ldotp$ if for every substitution $\theta$ such that all $b_i\theta$ are in the interpretation, $h\theta$ is in the interpretation as well. It is a model of a logic program if it is a model of all clauses in the program. The model-theoretic semantics of a definite clause program is given by its smallest Herbrand model with respect to set inclusion, the so-called \emph{least Herbrand model} (which is unique). We say that a logic program $\dcpprogram$ \emph{entails} an atom $a$, denoted $\dcpprogram\models a$, if and only if $a$ is true in the least Herbrand model of $\dcpprogram$.



%% file: files_appendix2/app_table.tex
\section{Table of Notations}\label{app:table}
\begin{center}
    \rowcolors{2}{gray!15}{white}

\begin{tabular}{c|p{0.5\textwidth}|c}
symbol & meaning & for details, see \\ \hline
 
  $\distributionfunctors$  & set of distribution functors  & Definition~\ref{def:reserved_vocabulary}\\
  $\arithmeticfunctors$   &  set of arithmetic functors & Definition~\ref{def:reserved_vocabulary}\\
  $\comparisonpredicates$   &  set of comparison  predicates & Definition~\ref{def:reserved_vocabulary}\\
  \hline
  $\samplespace_\nu$ & sample space of random variable $\nu$ & Definition~\ref{def:samplespace}\\
  $\samplefunction(\cdot)$ & value assignment function & Definition~\ref{def:samplespace}\\
  $\distdb$ & distributional database & Definition~\ref{def:distDB}\\
  $\randomvariableset$ & set of random variables & Definition~\ref{def:distDB}\\
  $(\samplespace_\distdb, \sigmaalgebra_\distdb, \probabilitymeasure_\distdb)$ & probability space induced by $\distdb$ & Definition~\ref{def:well-distdb}\\
  \hline
  $\comparisonfacts$ & set of Boolean comparison atoms & Definition~\ref{def:comparison-atoms-set}\\
  $\samplespace_\comparisonfacts$ & sample space induced by $\comparisonfacts$ & Proposition~\ref{prop:omegaf}\\
  $\sigmaalgebra_\comparisonfacts$ & sigma algebra induced by $\comparisonfacts$ & Proposition~\ref{prop:pfsigma}\\
  $\probabilitymeasure_\comparisonfacts$ & probability measure induced by $\comparisonfacts$ & Proposition~\ref{prop:pf}\\
  \hline
  $\dfprogram=\distdb  \cup \logicprogram$ & \dfplpsty program & Definition~\ref{def:core-prog}\\
  $\comparisonfacts_{\samplefunction(\randomvariableset)}$ & consistent comparison database induced by $\samplefunction$ on the random variables in $\randomvariableset$ & Definition~\ref{def:consistent-fact-db}\\
  $\probabilitymeasure_\dfprogram$ & probability measure over Herbrand interpretations defined by $\dfprogram$ & Proposition~\ref{prop:pp}\\
  \hline
  $\program$  & \dcproblogsty program & Definition~\ref{def:fullprog}\\
  \adfreeprogram & AD-free \dcproblogsty program & Definition~\ref{def:ad_free_program}\\
  $\headsdc_\adfreeprogram$ & set of heads of distributional clauses in \adfreeprogram & Definition~\ref{def:ad_free_program}\\
  $\randomtermset_\adfreeprogram$ & random terms in $\headsdc _\adfreeprogram$ & Definition~\ref{def:ad_free_program}\\
  $\dclauses_{\adfreeprogram}$ & set of distributional clauses in $\adfreeprogram$ & Definition~\ref{def:dc-df-well-def}\\
  $\contextfunc(\cdot)$ &  contextualization function & Definition~\ref{def:context_function}\\
  $\dfadfreeprogram$ &  \dfplpsty program providing the semantics of $\adfreeprogram$ & Definition~\ref{def:adfree-to-core}\\
  \hline
  $MOD(\dcpprogram)$ & models of a program $\dcpprogram$ & Theorem~\ref{theo:model_equivalence} \\
  $ENUM(\phi)$ & models of a propositional formula $\phi$ & Theorem~\ref{theo:model_equivalence} \\
  $\alpha(\cdot)$ & labeling function of a propositional literal & Definition~\ref{def:labeling_function} \\
  $\ive{\cdot}$ & Iverson bracket denoting an indicator function & Definition~\ref{def:labeling_function} \\
  $\E[\cdot]$ & expected value & Theorem~\ref{theo:label_equivalence} \\
  \hline
  $\mathbb{I}$ & set of infinitesimal numbers & Equation~\ref{def:inf_number} \\
  $\mathcal{S}$ & set of ancestral samples & Equation~\ref{eq:ancestral_samples}
\end{tabular}
\end{center}

%% file: files_appendix2/semantics_proofs.tex
\section{Proofs of Propositions in Section~\ref{sec:semantics}}
\label{app:semantics_proofs}





\subsection{Proof of Proposition~\ref{prop:omegaf}}
\label{app:proof:omegaf}

\new{

\propomegaf*

\begin{proof}
    Consider the set of comparison atoms $\comparisonfacts$=   $\{ \kappa_1,  \kappa_2, \dots\}$. 
    Each $\kappa_i$ depends on a finite subset $\randomvariableset_i$ of random variables, namely  those mentioned in $\kappa_i$.
    We write $\randomvariableset_{\leq n}=\bigcup_{1\leq j\leq n} \randomvariableset_j $ for the union of random variables that the first $n$ atoms in the enumeration depend on.
    We obtain the set of all random variables from the following limit:
    \begin{align}
        \mathcal{V}_\comparisonfacts = \lim_{n\rightarrow \infty} \mathcal{V}_{\leq n}.
    \end{align}
    We construct the sample space of $\randomvariableset_{\comparisonfacts}$  with a (countable) Cartesian product
    \begin{align}
        \samplespace_{\comparisonfacts}=\prod_{\nu \in \mathcal{V}_{\comparisonfacts}} \samplespace_\nu.
    \end{align}
\end{proof}
}

\subsection{Proof of Proposition~\ref{prop:pfsigma}}
\label{app:proof:pfsigma}

\new{

\proppfsigma*

\begin{proof}
    We construct the following cylinder set for each comparison atom in the set $\comparisonfacts=\{ \kappa_1,  \kappa_2, \dots \}$:
    \begin{align}
        K_j &= \{  \omega \in \samplespace_{\comparisonfacts} \mid \kappa_j(\omega)=\top \}.
    \end{align}
    Here we use $\kappa_j(\omega)$ to explicitly denote the evaluation of the comparison atom at $\omega$. We denote the set of all such cylinder sets by $\mathcal{K}_{\comparisonfacts}= \lim_{n\rightarrow \infty} \bigcup_{j=1}^n K_j$.
    
    Finally, we form the sigma-algebra $\Sigma_{\comparisonfacts}$ as the sigma-algebra generated by the collection of cylinder sets $\mathcal{K}_{\comparisonfacts}$:
    \begin{align}
        \Sigma_{\comparisonfacts} = \sigma(\mathcal{K}_{\comparisonfacts}).
    \end{align}
    where $\sigma(\mathcal{K}_{\comparisonfacts})$ denotes the intersection of all sigma-algebras containing $\mathcal{K}_{\comparisonfacts}$.
    Given that $\mathcal{K}_{\comparisonfacts} \subseteq \sigmaalgebra_\distdb$ we also have that $\sigmaalgebra_\comparisonfacts \subseteq \sigmaalgebra_\distdb$.
\end{proof}

}

\subsection{Proof of Proposition~\ref{prop:pf}}
\label{app:proof:pf}

\new{

\proppf*

\begin{proof}
    To show existence of the  measure $\measurecomparisonfacts$, we need to show that 
    \begin{enumerate}
        \item non-negativity: $\probabilitymeasure_\comparisonfacts(A)\geq0 , \quad \forall A \in \sigmaalgebra_\comparisonfacts$
        \item normality: $\probabilitymeasure_\comparisonfacts(\samplespace_\comparisonfacts)=1$
        \item countably additivity: for any collection $\{A_i\}_{i=1}^\infty$ of disjoint sets in $\sigmaalgebra_\comparisonfacts$ we have
        \begin{align}
            \probabilitymeasure_\comparisonfacts
            \left(
                \bigcup_{i=1}^\infty A_i
            \right)
            =
            \sum_{i=1}^\infty \probabilitymeasure_\comparisonfacts(A_i)
        \end{align} 
    \end{enumerate}
    Using the fact that $\sigmaalgebra_\comparisonfacts\subseteq \sigmaalgebra_\distdb$ it is straightforward to show these three properties hold. Uniqueness of $\probabilitymeasure_\comparisonfacts$ is also inherited from the uniqueness of $\probabilitymeasure_\distdb$.
\end{proof}

}

\subsection{Proof of Proposition~\ref{prop:pp}}
\label{app:proof:pp}

\proppp*

\new{
\begin{proof}

To show this, we follow Sato's construction to obtain the probability measure $\probabilitymeasure_\dfprogram$ over Herbrand interpretations from $\measurecomparisonfacts$. 
To this end we denote the set of atoms in the Herbrand base by $\mu_1, \mu_2, \dots$, which also includes those in $\comparisonfacts$.
As \dfprogram is valid, for every consistent comparison database $\comparisonfacts_{\samplefunction(\randomvariableset)}$ (\cf Definition~\ref{def:consistent-fact-db}), the logic program
$\comparisonfacts_{\samplefunction(\randomvariableset)}\cup \logicprogram$
has a total well-founded model $M_{\samplefunction(\randomvariableset)}$, and we can define   
\begin{align}
\probabilitymeasure_\dfprogram(\mu_1=\boolval_1,\mu_2=\boolval_2, \ldots)
:=
\measurecomparisonfacts
    \left(
    \left\{
    \samplefunction(\randomvariableset)
        ~|~
        M_{\samplefunction(\randomvariableset)}
    \right\}
    \right)
\end{align}
What remains, is to show that the set
$
    \left\{
    \samplefunction(\randomvariableset)
        ~|~
        M_{\samplefunction(\randomvariableset)}
    \right\}
$
is an element of the sigma-algebra $\sigmaalgebra_\comparisonfacts$. To this end, we rewrite the set as:

\begin{align}
    \Bigl\{
    \samplefunction(\randomvariableset)
        ~|~
        \mu_1(\samplefunction(\randomvariableset))=\boolval_1\land \mu_2(\samplefunction(\randomvariableset))=\boolval_2\land \ldots
    \Bigr\}
    \label{eq:proof:prop:semantics}
\end{align}
where we have that, 
\begin{align}
    M_{\samplefunction(\randomvariableset)} \models
    \mu_1(\samplefunction(\randomvariableset))=\boolval_1\land \mu_2(\samplefunction(\randomvariableset))=\boolval_2\land \ldots
\end{align}
We rewrite the set in Equation~\ref{eq:proof:prop:semantics} as

\begin{align}
    \bigcap_{j=1}^n
    \left\{
    \samplefunction(\randomvariableset)
        ~|~
        \mu_j(\samplefunction(\randomvariableset))={\boolval_j}       
    \right\}.
\end{align}
We now retain only those $ \mu_j$'s that depend on (a subset of) $\randomvariableset$, which we denote by $ \kappa_j$:
\begin{align}
    \bigcap_{j=1}
    \left\{
        \samplefunction(\randomvariableset)
            ~|~
            \kappa_j(\samplefunction(\randomvariableset))={\boolval_j}      
        \right\}.
\end{align}
The last line is an intersection of elements from $\comparisonfacts$ (or their complements) and thereby trivially part of $\sigmaalgebra_\comparisonfacts$, which concludes the proof.
\end{proof}

}

%% file: files_appendix2/syntactic_sugar_semantics.tex
\section{Detailed Discussion on \dcproblogsty}
\label{sec:detaileddcproblog}

\subsection{Syntactic Sugar Semantics}
\label{sec:semantics_syntactic_sugar}

We now formalize the declarative semantics of \dcproblogsty, \ie \dfplpsty extended with probabilistic facts, annotated disjunctions and distributional clauses,
The idea is to define program transformations that eliminate these three modelling constructs from a \dcproblogsty program, resulting in a \dfplpsty program for which we have defined the semantics in Section~\ref{sec:semantics}.

Throughout this section, we will treat distributional facts as distributional clauses with empty bodies, and we will only consider ground programs for ease of notation. As usual, a non-ground program is shorthand for its Herbrand grounding.

\begin{definition}[Statement]
    A \emph{\dcproblogsty statement} is either a probabilistic fact, an annotated disjunction, a distributional clause, or a normal clause.
\end{definition}

\begin{definition}[\dcproblogsty program] \label{def:fullprog}
    A \dcproblogsty program $\program$ is a countable set of ground \dcproblogsty statements.
\end{definition}

\subsubsection{Eliminating Probabilistic Facts and Annotated Disjunctions}

\begin{example}\label{ex:running-sugar-full}
	We use the following \dcproblogsty program as running example.
	\begin{problog*}{linenos}
p ~ beta(1,1).
p::a.
b ~ normal(3,1) :- a.
b ~ normal(10,1) :- not a.
c ~ normal(b,5).
0.2::d; 0.5::e; 0.3::f :- not b<5, b < 10.
g :- a, not f, b+c<15.
	\end{problog*}
\end{example}

\begin{definition}[Eliminaton Rules for Probabilistic Facts and ADs]\label{def:elim-ad}
	Let $\program$ be  a \dcproblogsty program. We define the following elimination rules (ER) to eliminate probabilistic facts and annotated disjunctions.
	\begin{itemize}
		\item[ER1:] replace each probabilistic fact $p\prob \mu$ in $\program$ by
		      \begin{align*}
			       & \nu \sim flip(p).  \\
			       & \mu \lpif \nu=:=1.
		      \end{align*}
		      with a fresh random variable $\nu$ for each probabilstic fact.
		\item[ER2:] replace each AD $p_1\prob\mu_1;\dots;p_n \prob \mu_n \lpif \beta$ in $\program$ by
		      \begin{align*}
			       & \nu \sim finite([p_1 : 1, \dots, p_n : n]) \\
			       & \mu_1 \lpif \nu=:=1, \beta.                \\
			       & \dots                                      \\
			       & \mu_n \lpif \nu=:=n, \beta.
		      \end{align*}
		      with a fresh random variable $\nu$ for each AD.
	\end{itemize}
\end{definition}
Note that if the probability label(s) of a fact or AD include random terms, as in the case of \probloginline{p::a} in the Example~\ref{ex:running-sugar-full},  then these are  parents of the newly introduced random variable. However, the new random variable will not be a parent of other random variables, as they are only used locally within the new fragments. They thus introduce neither cycles nor infinite ancestor sets into the program.

\begin{definition}[AD-Free Program] \label{def:ad_free_program}
	An \emph{AD-free} \dcproblogsty program \adfreeprogram is a \dcproblogsty program that contains neither probabilistic facts nor annotated disjunctions. We denote by $\headsdc_\adfreeprogram$ the set of atoms $\tau \sim \delta$ that appear as head of a distributional clause in $\adfreeprogram$, and by $\randomtermset_\adfreeprogram$ the set of random terms in $\headsdc_\adfreeprogram$.
\end{definition}

\begin{example}\label{ex:running-sugar-adfree}
	Applying Definition~\ref{def:elim-ad} to Example~\ref{ex:running-sugar-full} results in
	\begin{problog*}{linenos}
p ~ beta(1,1).
x ~ flip(p).
a :- x =:= 1.
b ~ normal(3,1) :- a.
b ~ normal(10,1) :- not a.
c ~ normal(b,5).
y ~ finite([0.2:1,0.5:2,0.3:3]).
d :- y =:= 1, not b<5, b < 10.
e :- y =:= 2, not b<5, b < 10.
f :- y =:= 3, not b<5, b < 10.
g :- a, not f, b+c<15.
	\end{problog*}
	We have $\headsdc_\adfreeprogram=\{ \, \text{\probloginline{p~beta(1,1)}},\,\allowbreak \text{\probloginline{x~flip(p)}},\,\allowbreak  \text{\probloginline{b~normal(3,1)}},\,\allowbreak \text{\probloginline{b~normal(10,1)}},\allowbreak \text{\probloginline{c~normal(b,5)}},\,\allowbreak \text{\probloginline{y~finite[0.2:1,0.5:2.0.3:3])}} \, \}$. Furthermore, we also have $\randomtermset_\adfreeprogram = \{ \mathprobloginline{p}, \mathprobloginline{x}, \mathprobloginline{b}, \mathprobloginline{c}, \mathprobloginline{y} \}$.
\end{example}

\subsubsection{Eliminating Distributional Clauses}
\label{sec:eliminate_dc}

While eliminating probabilistic facts and annotated disjunctions is a rather straightforward local transformation, eliminating distributional clauses is more involved. The reason is that a distributional clause has a global effect in the program, as it defines a condition under which a random term has to be \emph{interpreted} as a specific random variable when mentioned in a distributional clause or comparison atom. Therefore, eliminating a distributional clause involves both introducing the relevant random variable explicitly to the program and pushing the condition from the body of the distributional clause to all the places in the logic program that interpret the original random term.

Before delving into the mapping from an AD-free \dcproblogsty to a \dfplpsty program, we introduce some relevant terminology.

\begin{definition}[Parents and Ancestors for Random Terms]
	\label{def:parentancestor2}
	Given an AD-free program \adfreeprogram with $\tau_p$ and $\tau_c$ in $\randomtermset_\adfreeprogram$. We call $\tau_p$   a \emph{parent} of  $\tau_c$ if and only if  $\tau_p$ appears in the distribution term $\delta_c$ associated with $\tau_c$  in $\headsdc_\adfreeprogram$ ($\tau_c\sim \delta_c \in \headsdc_\adfreeprogram$).
	We define \emph{ancestor} to be the transitive closure of \emph{parent}.
\end{definition}

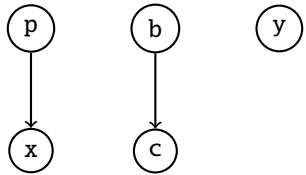
\begin{figure}[h]
	\centering
	\tikzstyle{node}=[circle, text centered, draw,thick]
	\begin{tikzpicture}[remember picture]

		\node[node] (p) {\probloginline{p}};
		\node[node] (x) [below=of p] {\probloginline{x}};

		\node[node] (b) [right=of p] {\probloginline{b}};
		\node[node] (c) [below=of b] {\probloginline{c}};

		\node[node] (y) [right=of b] {\probloginline{y}};

		\draw[->,thick] (p) to  (x);
		\draw[->,thick] (b) to  (c);

	\end{tikzpicture}
	\caption{Directed acyclic graph representing the ancestor relationship between the random variables in Example~\ref{def:ad_free_program}.
		The random terms \probloginline{p}, \probloginline{b} and \probloginline{y} have the empty set as their ancestor set. The ancestor set of \probloginline{x} is $\{ \mathprobloginline{p} \}$ and \probloginline{c} is $\{ \mathprobloginline{b} \}$.}
\end{figure}

For random terms, we distinguish  \emph{interpreted occurrences} of the term that need to be resolved to the correct random variable from other occurrences where the  random term is treated as any other term in a logic program, \eg, as an argument of a logical atom.
\begin{definition}[Interpreted Occurrence]
	\label{def:interpreted_occ}
	An \emph{interpreted occurrence} of a random term $\tau$ in an AD-free program \adfreeprogram is one of the following:
	\begin{itemize}
		\item the use of $\tau$ as parameter of a distribution term in the head of a distributional clause in \adfreeprogram
		\item the use of $\tau$  in a comparison literal in the body of a (distributional or normal) clause in \adfreeprogram
	\end{itemize}
	We say that a clause \emph{interprets} $\tau$ if there is at least one interpreted occurrence of $\tau$ in the clause.
\end{definition}

\begin{definition}[Well-Defined AD-free Program]\label{def:dc-df-well-def}
	Given an AD-free program \adfreeprogram with $\dclauses_\adfreeprogram$ the set of distributional clauses in \adfreeprogram, we call $\dclauses_\adfreeprogram$ \emph{well-defined} if the following conditions hold:
	\begin{description}
		\item[DC1] For each random term $\tau \in \randomtermset_\adfreeprogram$,  the number of  distributional clauses $\tau \sim \delta \lpif \beta$ in $\adfreeprogram$ is finite, and these clauses all have mutually exclusive bodies. This means that only a single rule can be true at once.
		\item[DC2] All distribution terms in $\dclauses_\adfreeprogram$ are well-defined for all possible values of the random terms they interpret.
		\item[DC3] Each random term has a finite set of ancestors.
		\item[DC4] The ancestor relation is acyclic.
	\end{description}
\end{definition}


We now discuss how to reduce a (valid) \dcproblogsty program to a \dfplpsty program. This happens in two steps. First, we eliminate distributional clauses and introduce appropriate distributional facts instead (see Definition~\ref{def:elim-dc}). Second, we {\em contextualize} interpreted occurrences of random terms in clause bodies (see Definition~\ref{def:core_encoding}).

The first step introduces a new built-in predicate \predicate{rv}{2} that associates random terms in a well-defined AD-free program with explicit random variables in the \dfplpsty program it is transformed into. This predicate is used in the bodies of clauses that interpret random terms (\cf Definition~\ref{def:interpreted_occ}) to appropriately contextualize those.

The idea behind the built-in \predicate{rv}{2} predicate is to restrict the applicability of a clause to contexts where all the random terms can be interpreted, \ie to contexts where the random terms are random variables. This implies that in contexts where such a random term cannot be interpreted, the \emph{entire} body evaluates to false.

\begin{definition}[Eliminating Distributional Clauses]\label{def:elim-dc}
	Let $\dclauses_{\adfreeprogram}$ be a well-defined set of distributional clauses.
	We denote by $\delta_{\rho_1, \dots, \rho_k}$ a distribution term that involves exactly $k$ different random terms $\rho_1,\ldots,\rho_k$.
	For each ground random term $\tau\in \randomtermset_{\adfreeprogram}$ we simultaneously define the following sets:
	\begin{itemize}
		\item the set of distributional facts for $\tau$
		      \begin{align*}
			      \distdb(\tau) =
			      \{
			      \tau^\beta_{\nu_1,\ldots,\nu_k}
			       & \sim \delta^\beta_{\nu_1,\ldots,\nu_k}
			      \\
			       & \mid
			      (\tau \sim \delta_{\rho_1,\ldots,\rho_k} \lpif \beta \in \dclauses_\adfreeprogram
			      ,
			      v_1\in \randomvariableset(\rho_1), \ldots, v_k\in \randomvariableset(\rho_k)
			      \}
		      \end{align*}
		\item    the set of (fresh) random variables for $\tau$
		      \begin{equation*}
			      \randomvariableset(\tau) = \{\nu \mid \nu\sim \delta \in \distdb(\tau)\}
		      \end{equation*}
		\item the set of context clauses for $\tau$
		      \begin{align*}
			       & \logicprogram^{c}(\tau)
			      =
			      \\
			       &
			      \quad \biggl\{\text{\probloginline{rv}}(\tau,\tau^\beta_{\nu_1,\ldots,\nu_k})
			      \lpif
			      \text{\probloginline{rv}}(\rho_1,\nu_1),\ldots,\text{\probloginline{rv}}(\rho_k,\nu_k), \beta
			      \\
			       &
			      \quad \bigm|
			      \tau \sim \delta_{\rho_1,\ldots,\rho_k} \lpif \beta \in \dclauses_{\adfreeprogram}, \nu_1\in \randomvariableset(\rho_1), \ldots, \nu_k\in \randomvariableset(\rho_k)
			      \biggr\}
		      \end{align*}
	\end{itemize}
\end{definition}

At first glance, Definition~\ref{def:elim-dc} seems to contain a mutual recursion involving $\distdb(\cdot)$ and $\randomvariableset(\cdot)$.
However, if we recall that for a well-defined set of distributional clauses $\dclauses_\adfreeprogram$ the ancestor relationship between random terms constitutes an acyclic directed graph, the apparent mutual recursion evaporates. We can now define the distributional facts encoding of the distributional clauses, which will give rise to a \dfplpsty program instead of \dcproblogsty program.

\begin{definition}[Distributional Facts Encoding]
	\label{def:core_encoding}
	Let $\adfreeprogram$ be an AD-free \dcproblogsty program and $\dclauses_{\adfreeprogram}$ its set of distributional clauses. We define the distributional facts encoding of $\dclauses_{\adfreeprogram}$  as
	$\dclauses_{\adfreeprogram}^{DF}\coloneqq \distdb\cup \logicprogram^{c}$, with
	\begin{align*}
		\distdb = \bigcup_{\tau\in \randomtermset_{\adfreeprogram}}\distdb(\tau)
		 &  &
		\logicprogram^{c} = \bigcup_{\tau\in \randomtermset_{\adfreeprogram}} \logicprogram^{c}(\tau)
	\end{align*}
	using $\distdb(\cdot)$ and $\logicprogram^{c}(\cdot)$ from Definition~\ref{def:elim-dc}.
\end{definition}

\begin{example}[Eliminating Distributional Clauses]
	\label{ex:running-sugar-core-encoding}
	We demonstrate the elimination of distributional clauses using the DCs in Example~\ref{ex:running-sugar-adfree}, \ie
	\begin{problog*}{linenos}
p ~ beta(1,1).
x ~ flip(p).
b ~ normal(3,1) :- a.
b ~ normal(10,1) :- not a.
c ~ normal(b,5).
y ~ finite([0.2:1,0.5:2,0.3:3]).
	\end{problog*}
	Here,  the distribution terms in Line 2 and Line 5 (\probloginline{flip(p)} and \probloginline{normal(b,5)}) contain one parent random term each (\probloginline{p} and \probloginline{b}, respectively), whereas all others have no parents. As \probloginline{b} is defined by two clauses, we get fresh random variables for each of them, which in turn  introduces different fresh random variables for the child \probloginline{c}.
	This gives us:
	\begin{problog*}{linenos}
v1 ~ beta(1,1).
rv(p,v1).
v2 ~ flip(v1).
rv(x,v2) :- rv(p,v1).
v3 ~ normal(3,1).
rv(b,v3) :- a.
v4 ~ normal(10,1).
rv(b,v4) :- not a.
v5 ~ normal(v3,5).
rv(c,v5) :- rv(b,v3).
v6 ~ normal(v4,5).
rv(c,v6) :- rv(b,v4).
v7 ~ finite([0.2:1,0.5:2,0.3:3]).
rv(y,v7).
	\end{problog*}
\end{example}

Eliminating distributional clauses (following Definition~\ref{def:elim-dc}) introduces the distributional facts and context rules necessary to encode the original distributional clauses. To complete the transformation to a \dfplpsty program, we further transform the logical rules. Prior to that, however, we need to define the {\em contextualization function}.

\begin{definition}[Contextualization Function] \label{def:context_function}
	Let $\beta$ be a conjunction of atoms and let its comparison literals interpret the random terms $\tau_1,\ldots,\tau_n$.
	Furthermore, let $\Lambda_i$ be a special logical variable associated to a random term $\tau_i \in \randomtermset_{\adfreeprogram}$ for each $\tau_i$.
	We define $\contextfunc(\beta)$ to be the  conjunction of literals obtained by replacing the interpreted occurrences of the $\tau_i$ in $\beta$ by their corresponding $\Lambda_i$
	and conjoining to this modified conjunction $\text{\probloginline{rv}}(\tau_i,\Lambda_i)$ for each $\tau_i$.
	We call $\contextfunc(\cdot)$ the contextualization function.
\end{definition}

\begin{definition}[Contextualized Rules]\label{def:adfree-to-core}
	Let \adfreeprogram be an AD-free program with logical rules $\logicprogram^\adfreeprogram$ and distributional clauses $\dclauses_\adfreeprogram$,
	and let $\dclauses_\adfreeprogram^{DF}=\distdb\cup \logicprogram^{c}$ be the distributional facts encoding of $\dclauses_\adfreeprogram$. We define the contextualization of the bodies of the rules $\logicprogram^\adfreeprogram \cup \logicprogram^{DF}$ as the sequential application of the following contextualization rules:
	\begin{enumerate}[label=\alph*.]
		\item[CR1:] apply the contextualization function $\contextfunc$ to all bodies in $\logicprogram^\adfreeprogram\cup \logicprogram^{c}$ and obtain:
		      \begin{align*}
			      \logicprogram^\Lambda = \{\eta \lpif \contextfunc(\beta) \mid \eta \lpif \beta \in \logicprogram^\adfreeprogram \cup \logicprogram^{c} \}
		      \end{align*}
		\item[CR2:] obtain the set of ground logical rules $\logicprogram$ by grounding each logical variable $\Lambda_i$ in $\logicprogram^\Lambda$  with random variables $\nu_i\in \randomvariableset(\tau_i)$ in all possible ways.
	\end{enumerate}
	We call $\logicprogram$ the contextualized logic program of $\adfreeprogram$, \fixed{Furthermore, we define $\dfadfreeprogram\coloneqq\distdb \cup \logicprogram $ to be the corresponding \dcplpsty program.}
\end{definition}

The contextualization function $\contextfunc(\cdot)$ creates non-ground comparison atoms, \eg $\Lambda_i>5$. Contrary to (ground) random terms, non-ground logical variables in such a comparison atom are not interpreted occurrences (\cf~Definition~\ref{def:interpreted_occ}) and the comparison itself only has a logical meaning. By grounding out the freshly introduced logical variables we obtain a purely logical program where the comparison atoms contain either arithmetic expressions or random variables (instead of random terms).

\begin{example}[Contextualizing Random Terms]
	\label{example:contextualization}
	Let us now study the effect of the second transformation step.
	Consider again the AD-free program in Example~\ref{ex:running-sugar-adfree} and the set of rules and distributional clauses obtained in Example~\ref{ex:running-sugar-core-encoding}.
	The contextualization rule CR1 (\cf Definition~\ref{def:adfree-to-core})  rewrites the logical rules in the AD-free input program to
	\begin{problog*}{linenos}
a :- rv(x,Lx), Lx =:= 1.
d :- rv(y,Ly), rv(b,Lb), Ly =:= 1, not Lb<5, Lb < 10.
e :- rv(y,Ly), rv(b,Lb), Ly =:= 2, not Lb<5, Lb < 10.
f :- rv(y,Ly), rv(b,Lb), Ly =:= 3, not Lb<5, Lb < 10.
g :- rv(b,Lb), rv(c,Lc), a, not f, Lb+Lc < 15.
	\end{problog*}
	These rules then get instantiated (rule CR2) to
	\begin{problog*}{linenos}
a :- rv(x, v2), v2 =:= 1.
d :- rv(y,v7), rv(b,v3), v7 =:= 1, not v3<5, v3 < 10.
e :- rv(y,v7), rv(b,v3), v7 =:= 2, not v3<5, v3 < 10.
f :- rv(y,v7), rv(b,v3), v7 =:= 3, not v3<5, v3 < 10.
d :- rv(y,v7), rv(b,v4), v7 =:= 1, not v4<5, v4 < 10.
e :- rv(y,v7), rv(b,v4), v7 =:= 2, not v4<5, v4 < 10.
f :- rv(y,v7), rv(b,v4), v7 =:= 3, not v4<5, v4 < 10.
g :- rv(b,v3), rv(c,v5), a, not f, v3+v5<15.
g :- rv(b,v3), rv(c,v6), a, not f, v3+v6<15.
g :- rv(b,v4), rv(c,v5), a, not f, v4+v5<15.
g :- rv(b,v4), rv(c,v6), a, not f, v4+v6<15.
	\end{problog*}
	Together with the distributional facts and rules obtained in Example~\ref{ex:running-sugar-core-encoding}, this last block of rules forms the \dcplpsty program that specifies the semantics of the AD-free \dcproblogsty program, and thus the semantics of the \dcproblogsty program in Example~\ref{ex:running-sugar-full}.
\end{example}

We note that the mapping from an AD-free program to a set of distributional facts and contextualized rules as defined here is purely syntactical, and written to avoid case distinctions.
Therefore, it usually produces overly verbose programs.
For instance, for random terms introduced by a distributional fact, the indirection via \probloginline{rv} is only needed if there is a parent term in the distribution that has context-specific interpretations.
The grounding step may introduce rule instances whose conjunction of \probloginline{rv}-atoms is inconsistent.
This is for example the case for the last three rules for \probloginline{g} in the Example~\ref{example:contextualization}, which we illustrate in the example below.

\begin{example}\label{ex:running-sugar-simplified}
	The following is a (manually) simplified version of the \dfplpsty program for the running example, where we propagated definitions of  \probloginline{rv}-atoms:
	\begin{problog*}{linenos}
v1 ~ beta(1,1).
v2 ~ flip(v1).
v3 ~ normal(3,1).
v4 ~ normal(10,1).
v5 ~ normal(v3,5).
v6 ~ normal(v4,5).
v7 ~ finite([0.2:1,0.5:2,0.3:3]).

a :- v2 =:= 1.
d :- a,     v7 =:= 1, not v3<5, v3 < 10.
e :- a,     v7 =:= 2, not v3<5, v3 < 10.
f :- a,     v7 =:= 3, not v3<5, v3 < 10.
d :- not a, v7 =:= 1, not v4<5, v4 < 10.
e :- not a, v7 =:= 2, not v4<5, v4 < 10.
f :- not a, v7 =:= 3, not v4<5, v4 < 10.
g :- a,     a,     a, not f, v3+v5<15.
g :- a,     not a, a, not f, v3+v6<15.  
g :- not a, a,     a, not f, v4+v5<15.  
g :- not a, not a, a, not f, v4+v6<15.  
	\end{problog*}
	In the bodies of the last three rules we have, inter alia, conjunctions of \probloginline{a} and \probloginline{not a}. This can never be satisfied and renders the bodies of these rules inconsistent.
\end{example}

\begin{definition}[Semantics of AD-free \dcproblogsty Programs]
	\label{def:semantics_adfree}
	The semantics of an AD-free \dcproblogsty program \adfreeprogram is the semantics of the \dfplpsty program $ \dfadfreeprogram = \distdb \cup \logicprogram$. We call $\adfreeprogram$ valid if and only if $\dfadfreeprogram$ is valid.
\end{definition}

\begin{definition}[Semantics of $\dcproblogsty$ Programs]
	The semantics of a $\dcproblogsty$ program $\program$ is the semantics of the AD-free \dcproblogsty program $\adfreeprogram$. We call $\program$ valid if and only if $\adfreeprogram$ is valid.
\end{definition}

Programs with distributional clauses can make programs with combinatorial structures more readable by grouping random variables with the same role  under the same random term. However, the programmer needs to be aware of the fact that distributional clauses have non-local effects on the program, as they affect the interpretation of their random terms also outside the distributional clause itself. This can be rather subtle, especially  if the bodies of the distributional clauses with the same random term are not exhaustive.
We discuss this issue in more detail in Appendix~\ref{sec:non-mixture-dc}.

\subsection{Syntactic Sugar: Validity}
\label{sec:dcvalidity}

As stated above, a DC-ProbLog program $\program$  is syntactic sugar for an AD-free program $\adfreeprogram$ (Definition~\ref{def:elim-ad}),
and is valid if $\dfadfreeprogram$ as specified in Definition~\ref{def:semantics_adfree} is a valid \dfplpsty program, \ie the distributional database is well-defined,
the comparison literals are measurable,
and each consistent fact database results in a two-valued  well-founded  model if added to the logic program (Definition~\ref{def:core-valid}).
For the distributional database to be well-defined (Definition~\ref{def:well-defd-facts}), it suffices to have $\dclauses_\adfreeprogram$ well-defined  (Definition~\ref{def:dc-df-well-def}), as can be verified by comparing the relevant definitions. Indeed, a well-defined $\dclauses_\adfreeprogram$ is a precondition for the transformation as stated in the definition.

The transformation changes neither distribution terms nor comparison literals, and thus maintains the measurability of the latter.
As far as the logic program structure is concerned, the transformation to a \dfplpsty adds rules for \probloginline{rv} based on the bodies of all distributional clauses, and uses positive \probloginline{rv} atoms in the bodies of all clauses  that interpret random terms to ensure that all interpretations of random variables are anchored in the appropriate parts of the distributional database. This level of indirection does not affect the logical reasoning for programs that only interpret random terms in appropriate contexts. It is the responsibility of the programmer to ensure that this is the case and indeed results in appropriately defined models.

\subsection{Syntactic Sugar: Additional Constructs}
\label{sec:Additionalsyntacticsugar}

\subsubsection{User-Defined Sample Spaces}
The semantics of \dcproblogsty as presented in the previous sections only alllows for random variables with numerical sample spaces, \eg normal distributions, or Poisson distributions. For categorical random variables, however, one might like to give a specific meaning to the elements in the sample space instead of a numerical value.
\begin{example}
	Consider the following program:
	\begin{problog*}{linenos}
color ~ uniform([r,g,b]).
red:- color=:=r.
	\end{problog*}
	Here we discribe a categorical random variable (uniformaly distributed) whose sample space is the set of expressions $\{\text{\probloginline{r}},\text{\probloginline{b}},\text{\probloginline{g}}\}$. By simply associating a natural number to each element of the sample space we can map the program back to a program whose semantics we already defined:
	\begin{problog*}{linenos}
color ~ uniform([1,2,3]).
r:- color=:=1,
red:- r.
	\end{problog*}
\end{example}

Swapping out the sample space of discrete random variables with natural numbers is always possible as the cardinality of such a sample space is either smaller (finite categorical) or equal (infinite) to the cardinality of the natural numbers.

\subsubsection{Multivariate Distributions}

Until now we have restricted the syntax and semantics of \dcproblogsty to univariate distributions, \eg the univariate normal distribution. At first this might seem to severely limit the expressivity of \dcproblogsty, as probabilistic modelling with multivariate random variables is a common task in modern statistics and probabilistic programming. However, this concern is voided by realizing that multivariate random variables can be decomposed into {\em combinations} of independent univariate random variables. We will illustrate this on the case of the bivariate normal distribution.

\begin{example}[Constructing the Bivariate Normal Distribution]
	Assume we would like to construct a random variable distributed according to a bivariate normal distribution:
	\begin{align*}
		\begin{pmatrix} \nu_1 \\ \nu_2 \end{pmatrix}
		\sim
		\mathcal{N}\left(
		\begin{pmatrix}
				\mu_1 \\  \mu_2
			\end{pmatrix}
		,
		\begin{pmatrix}
				\sigma_{11} & \sigma_{12} \\
				\sigma_{21} & \sigma_{22} \\
			\end{pmatrix}
		\right)
	\end{align*}
	The equation above can be rewritten as:
	\begin{align*}
		\begin{pmatrix} \nu_1 \\ \nu_2 \end{pmatrix}
		\sim
		\begin{pmatrix}
			\mu_1 \\  \mu_2
		\end{pmatrix}
		+
		\begin{pmatrix}
			\eta_{11} & \eta_{12} \\
			\eta_{21} & \eta_{22} \\
		\end{pmatrix}
		\begin{pmatrix}
			\mathcal{N}(0, \lambda_1) \\
			\mathcal{N}(0, \lambda_2)
		\end{pmatrix}
	\end{align*}
	where it holds that
	\begin{align*}
		\begin{pmatrix}
			\sigma_{11} & \sigma_{12} \\
			\sigma_{21} & \sigma_{22} \\
		\end{pmatrix}
		=
		\begin{pmatrix}
			\eta_{11} & \eta_{12} \\
			\eta_{21} & \eta_{22} \\
		\end{pmatrix}
		\begin{pmatrix}
			\lambda_1 & 0         \\
			0         & \lambda_2
		\end{pmatrix}
		\begin{pmatrix}
			\eta_{11} & \eta_{21} \\
			\eta_{12} & \eta_{22} \\
		\end{pmatrix}
	\end{align*}
	It can now be shown that the bivariate distributions can be expressed as:
	\begin{align*}
		\begin{pmatrix} \nu_1 \\ \nu_2 \end{pmatrix}
		\sim
		\begin{pmatrix}
			\mathcal{N}(\mu_{\nu_1}, \sigma_{\nu_1}) \\
			\mathcal{N}(\mu_{\nu_2}, \sigma_{\nu_2}) \\
		\end{pmatrix}
	\end{align*}
	where $\mu_{\nu_1}$, $\mu_{\nu_2}$, $\sigma_{\nu_1}$ and $\sigma_{\nu_2}$ can be expressed as:
	\begin{align*}
		\mu_{\nu_1} & = \mu_1 \qquad
		            & \sigma_{\nu_1} = \sqrt{ \eta_{11}\lambda_1^2 + \eta_{12}\lambda_2^2 } \\
		\mu_{\nu_2} & = \mu_2 \qquad
		            & \sigma_{\nu_2} = \sqrt{ \eta_{21}\lambda_1^2 + \eta_{22}\lambda_2^2 }
	\end{align*}
	We conclude from this that a bivariate normal distribution can be modeled using two univariate normal distributions that have a shared set of parameters and is thereby semantically defined in \dcproblogsty.

\end{example}

Expressing multivariate random variables in a user-friendly fashion in a probabilistic programming language is simply a matter of adding syntactic sugar for combinations of univariate random variables once the semantics are defined for the latter.

\begin{example}[Bivariate Normal Distribution]
	Possible syntactic sugar to declare a bivariate normal distribution in \dcproblogsty, where the mean of the distribution in the two dimensions is $0.5$ and $2$, and the covariance matrix is
	$
		\begin{bmatrix}
			2   & 0.5 \\
			0.5 & 1
		\end{bmatrix}
	$.
	\begin{problog*}{linenos}
(x1,x2) ~ normal2D([0.5,2], [[2, 0.5],[0.5,1]])
q:- x1<0.4, x2>1.9.
	\end{problog*}
\end{example}

On the inference side, the special syntax might then additionally be used to deploy dedicated inference algorithms. This is usually done in probabilistic programming languages that cater towards inference with multivariate (and often continuous) random variables~\citep{carpenter2017stan,bingham2019pyro}. Note that probability distributions are usually constructed by applying transformations to a set of independent uniform distribution. From this viewpoint the builtin-in \probloginline{normal/2}, denoting the univariate normal distribution, is syntactic sugar for such a transformation as well.

\subsection{Beyond Mixtures}\label{sec:non-mixture-dc}
\label{sec:beyondmixtures}

By definition, we impose on distributional clauses mutual exclusivity of their bodies when they share a random term (\cf Definition~\ref{def:dc-df-well-def}). That is, if we have a set of distributional clauses: $\{ \tau \sim \delta_i\lpif \beta_1 \dots, \tau \sim \delta_n\lpif \beta_n  \}$ we impose that the conjunction of two distinct bodies $\beta_i$ and $\beta_j$ ($i \neq j$) is false.

A further condition that we might impose, which is, however, not necessary to define a valid distributional clause, is exhaustiveness. Let us consider again the set of distributional clauses  $\{ \tau \sim \delta_i\lpif \beta_1 \dots, \tau \sim \delta_n\lpif \beta_n  \}$. We call this set exhaustive if the disjunction of all the $\beta_i$'s is equivalent to true.

A set of exhaustive distributional clauses can be interpreted as a mixture models as they assign a unique distribution to the random term in any possible context. When, the bodies of such distributional clauses are not exhaustive, however, they may
interact with the logic program in rather subtle ways, especially if negation is involved. We demonstrate this in the examples below.

\begin{example}
	Consider the  following program fragments
	\begin{problog}
q :- not (x=:=1).
	\end{problog}
	and
	\begin{problog}
aux :- x=:=1.
q :- not aux.
	\end{problog}
	and now assume \probloginline{x} follows a mixture distribution, \eg,
	\begin{problog}
0.2::b.
x~flip(0.5) :- b.
x~flip(0.9) :- not b.
	\end{problog}
	With such a mixture model, as in the case of a distributional fact, "\probloginline{x} has an associated distribution" is always true, and both fragments agree on the truth value of \probloginline{q}.
\end{example}
In general, however, only the first of these two conditions is necessary, and it is thus possible to  associate a distribution with a random term in \emph{some} contexts only.

\begin{example}
	Consider again the two program fragments above with the following non-exhaustive definition of \probloginline{x}:
	\begin{problog}
0.2::b.
x~flip(0.5) :- b.
	\end{problog}
	With this definition, "\probloginline{x} has an associated distribution" is true if and only if \probloginline{b} is true, and the two fragments therefore no longer agree on the truth values of \probloginline{q}, as we more easily see after eliminating the distributional clause. We omit the transformation of the probabilistic fact for brevity. The fragment defining \probloginline{x} transforms to
	\begin{problog}
v1~flip(0.5).
rv(x,v1) :- b.
	\end{problog}
	The first program fragment maps to
	\begin{problog}
q :- rv(x,v1), not (v1=:=1).
	\end{problog}
	and the second one to
	\begin{problog}
aux :- rv(x,v1), v1=:=1.
q :- not aux.
	\end{problog}
	which clearly exposes the difference in how the negation is interpreted.
\end{example}
As this example illustrates, if random variables are defined through non-exhaustive sets of DCs, we can no longer refactor the logic program independently of the definition of the random variables in general, as it interacts with the context structure. The reason is that \dcproblogsty's declarative semantics builds upon the principle that the distributional database is declared \emph{independently} of the logic program, and can thus be combined modularly and declaratively with \emph{any} logic program over its comparison atoms. This is no longer the case with such arbitrary sets of DCs, which intertwine the definition of the two parts of a \dfplpsty program. We note that this differs from the procedural view on the existence of random variables taken in the \dcsty language~\citep{nitti2016probabilistic}, as we
discuss in more detail in Appendix~\ref{sec:dcproblog-dc}.

%% file: files_appendix2/detailed_comp_to_GutmanNittiDC.tex
\section{Relation to the DC language}\label{sec:dcproblog-dc}

Distributional clauses were first introduced in the language of the same name  by \cite{gutmann2011magic}, which at that point did not support negation. For negation-free programs, our interpretation of distributional clauses exactly corresponds to theirs, and \dcproblogsty thus generalizes both ProbLog (with negation) and the original (definite) distributional clause language.


In the following, we first discuss how the semantics of \dcproblogsty differs from \cite{nitti2016probabilistic}'s procedural view on negated comparison atoms, and then how \dcproblogsty's acyclicity conditions imposed on valid programs differ from those of \cite{gutmann2011magic}.

\subsection{Non-exhaustive sets of DCs}\label{sec:proc-nitti}
\cite{nitti2016probabilistic}  have extended the procedural view of the stochastic $T_P$ operator to locally stratified programs with negation under the perfect models semantics.\footnote{Local stratification is a necessary condition for perfect models semantics, and a sufficient one for well-founded semantics. On this class of programs, both semantics agree~\citep{van1991well}} In their view, a distributional clause \probloginline{x~d :- body} is informally interpreted  as "if \probloginline{body} is true, define a random variable \probloginline{x} with distribution \probloginline{d}". They then use the principle that "any comparison involving a non-defined variable will fail; therefore, its negation will succeed", \ie, they apply negation as failure to comparison atoms. 
In contrast, as already illustrated in Section~\ref{sec:non-mixture-dc}, we take a purely declarative view here, where all random variables are defined up front, independently of logical reasoning, and distributional clauses serve as syntactic sugar to compactly talk about a group of random variables. Then, truth values of comparison atoms are fully determined by their external interpretation, and do not involve reasoning about whether a random variable is defined or not. That is, we apply classical negation to comparison atoms, and restrict negation as failure to atoms defined by the logic program itself.

The following example adapted from \cite{nitti2016probabilistic} illustrates the difference.
\begin{example}\label{ex:dc-vs-dcproblog}
	Consider the following program about the color of certain objects, where the number of objects is given by the random variable \probloginline{n}:
	\begin{problog}
n ~ uniform([1,2,3]).
color(1) ~ uniform([red,green,blue]) :- 1=<n .
color(2) ~ uniform([red,green,blue]) :- 2=<n .
color(3) ~ uniform([red,green,blue]) :- 3=<n .
not_red :- not color(2)=:=red .
not_red_either :- color(2)=\=red.
	\end{problog}
	The \dcproblogsty semantics is given by the transformed program:
	\begin{problog}
v0 ~ uniform([1,2,3]).
v1 ~ uniform([red,green,blue]).
v2 ~ uniform([red,green,blue]).
v3 ~ uniform([red,green,blue]).

rv(n,v0).
rv(color(1),v1) :- rv(n,v0), 1=<v0.
rv(color(2),v2) :- rv(n,v0), 2=<v0.
rv(color(3),v3) :- rv(n,v0), 3=<v0.

not_red :- rv(color(2),v2), not v2=:=red.
not_red_either :- rv(color(2),v2), v2=\=red.
	\end{problog}
	If $\mathprobloginline{n}=1$ (i.e., $\mathprobloginline{v0}=1$), neither \probloginline{color(2)} nor \probloginline{color(3)} are associated with a distribution.  Thus, \probloginline{rv(color(2),v2)} fails, and both \probloginline{not_red} and \probloginline{not_red_either} therefore fail as well, independently of the values of the comparison literals.
	In contrast, under 
	the procedural semantics of \cite{nitti2016probabilistic}, \probloginline{color(2)=:=red} fails in this case, and \probloginline{not_red} thus succeeds. Similarly, \probloginline{color(2)=\=red} fails, and  \probloginline{not_red_either} thus fails.  
	Both views agree for $n>1$.
\end{example}

This example again illustrates that DC-ProbLog's semantics clearly follows the spirit of the distribution   semantics of defining a distribution over interpretations of basic facts (comparison atoms in this case) independently of the logic program rules. 
We note that the expressive power of logic programs allows the programmer to explicitly model the procedural view of "failure through undefined variable" in the program if desired, as illustrated in the following example. 

\begin{example}
	The following DC-ProbLog program is equivalent to the procedural interpretation of the program in  Example~\ref{ex:dc-vs-dcproblog}:
	\begin{problog*}{linenos}
n ~ uniform([1,2,3]).
color(1) ~ uniform([red,green,blue]) :- 1=<n .
color(2) ~ uniform([red,green,blue]) :- 2=<n .
color(3) ~ uniform([red,green,blue]) :- 3=<n .
not_red :- not color(2)=:=red.     
not_red :- not 2=<n.               
not_red_either :- 2=<n, color(2)=\=red.
	\end{problog*}
	We explicitly model that \probloginline{not_red} is true if either \probloginline{color(2)} can be interpreted and \probloginline{color(2)=:=red} is false (line 5, which is how \dcproblogsty interprets the first clause), or \probloginline{color(2)} cannot be interpreted (line 6, negating the body of the DC in line 3). Similarly, \probloginline{not_red_either} is true if and only if \probloginline{color(2)} can be resolved and \probloginline{color(2)=\=red} is true (line 7, repeating the body of the DC in line 3).
\end{example}

\subsection{Program validity}
To define valid programs, \citet{gutmann2011magic} impose acyclicity criteria based on the structure of the clauses in the program, whereas we use the ancestor relation between random variables in \dcproblogsty. This means that \dcproblogsty accepts certain cycles in the logic program structure that are rejected by \dcsty, as illustrated in the following example. 

\begin{example}
	We model a scenario where a property of a node in a network is either initiated locally with probability $0.1$, or propagated from a neighboring node that has the property with probability $0.3$. We consider a two node network with directed edges from each of the nodes to the other one, and directly ground the program for this situation.
	\begin{problog}
local(n1) ~ flip(0.1).
local(n2) ~ flip(0.1).
transmit(n1,n2) ~ flip(0.3) :- active(n1).
transmit(n2,n1) ~ flip(0.3) :- active(n2).
active(n1) :- local(n1)=:=1.
active(n2) :- local(n2)=:=1.
active(n1) :- transmit(n2,n1)=:=1.
active(n2) :- transmit(n1,n2)=:=1.
	\end{problog}
	This program is not distribution-stratified based on \citep{gutmann2011magic}, where in order to avoid cyclic probabilistic dependencies 1) DC heads have to be of strictly higher rank than any of their body atoms, 2) heads of regular clauses have to have at least the same rank as each body atom, and 3) atoms involving random terms have to have at least the same rank as the head of the DC introducing the random term. This is impossible with our program due to the cyclic dependency between \probloginline{active}-atoms and \probloginline{transmit}-random terms. The \dcproblogsty semantics, in contrast,  is clearly specified through the mapping:
	\begin{problog}
x1 ~ flip(0.1).
x2 ~ flip(0.1).
x3 ~ flip(0.3).
x4 ~ flip(0.3).
rv(local(n1),x1).
rv(local(n2),x2).
rv(transmit(n1,n2),x3) :- active(n1).
rv(transmit(n2,n1),x4) :- active(n2).
active(n1) :- rv(local(n1),x1), x1=:=1.
active(n2) :- rv(local(n2),x2), x2=:=1.
active(n1) :- rv(transmit(n2,n1),x4), x4=:=1.
active(n2) :- rv(transmit(n1,n2),x3), x3=:=1.
	\end{problog}
	We have four independent random variables, and a definite clause program whose meaning is well-defined despite of the cyclic dependencies between derived atoms. We can equivalently rewrite the logic program part to avoid deterministic auxiliaries:
	\begin{problog}
x1 ~ flip(0.1).
x2 ~ flip(0.1).
x3 ~ flip(0.3).
x4 ~ flip(0.3).
active(n1) :- x1=:=1.
active(n2) :- x2=:=1.
active(n1) :- active(n2), x4=:=1.
active(n2) :- active(n1), x3=:=1.
	\end{problog}
	Furthermore, \dcproblogsty agrees with the ProbLog formulation of the original program, \ie,
	\begin{problog}
0.1::local_cause(n1).
0.1::local_cause(n2).
0.3::transmit_cause(n1,n2) :- active(n1).
0.3::transmit_cause(n2,n1) :- active(n2).
active(n1) :- local_cause(n1).
active(n2) :- local_cause(n2).
active(n1) :- transmit_cause(n2,n1).
active(n2) :- transmit_cause(n1,n2).
	\end{problog}
	The AD-free program is
	\begin{problog}
v1 ~ flip(0.1).
local_cause(n1) :- v1=:=1.
v2 ~ flip(0.1).
local_cause(n2) :- v2=:=1.
v3 ~ flip(0.3).
transmit_cause(n1,n2) :- v3=:=1, active(n1).
v4 ~ finite(0.3).
transmit_cause(n2,n1) :- v4=:=1, active(n2).
active(n1) :- local_cause(n1).
active(n2) :- local_cause(n2).
active(n1) :- transmit_cause(n2,n1).
active(n2) :- transmit_cause(n1,n2).
	\end{problog}
	As this already is a \dfplpsty program, we can skip the further rewrites. While the definite clauses are factored differently compared to the earlier variants, their meaning is the same. 
\end{example}

%% file: files_appendix2/transformation_proofs.tex
\section{Proofs of Theorems and Propositions in Section~\ref{sec:dc2smt} and Section~\ref{sec:alw}}
\label{app:transformation_proofs}

\subsection{Proof of Theorem~\ref{theo:rgp}} \label{app:proof:rgp}

\theorgp*

\begin{proof}
    The semantics of $\dcpprogram$ is given by the ground program that is obtained by first grounding $\dcpprogram$ with respect to its Herbrand base and reducing it to a \dfplpsty program as specified in Section~\ref{sec:dcproblog}. The resulting program consists of distributional facts and ground normal clauses only, and includes clauses defining \probloginline{rv}-atoms as well as calls to those atoms in clause bodies. However, as the definitions of these atoms are acyclic, and each ground instance is defined by a single rule (with the body of the DC that introduced the new random variable), we can eliminate all references to such atoms by recursively applying the well-known \emph{unfolding} transformation, which replaces atoms in clause bodies by their definition. The result is an equivalent ground program using only predicates from \dcpprogram, but where rule bodies have been expanded with the contexts of the random variables they interpret. We know from Theorem~1 in \citep{fierens2015inference} that for given query and evidence, it is sufficient to use the part of this logic program that is encountered during backward chaining from those atoms. We note that in our case, this also includes the distributional facts providing the distributions for relevant random variables, \ie, random variables in relevant comparison atoms as well as their ancestors.
\end{proof}

\subsection{Proof of Theorem~\ref{theo:label_equivalence}}
\label{app:proof:label_equivalence}

\Labelequivalence*

\begin{proof} The probability of a model $\varphi$ of the relevant ground program $\dcpprogram_g$ is, according to the distribution semantics (cf. Appendix~\ref{app:proof:pf}), given by:
    \begin{align}
        P_{\dcpprogram_g}(\varphi) = \int \mathbf{1}_{[\mu_1=\boolval_1\wedge\ldots\wedge \mu_n=\boolval_n]}(\samplefunction(\randomvariableset)) \differential{\probabilitymeasure_{\randomvariableset}}
    \end{align}
    where the $\mu_i$ are the comparison atoms that appear (positively or negatively) in $\dcpprogram_g$ and the $b_i$ the truth values these atoms take in $\varphi$.
    We can manipulate the probability into:
    \begin{align}
        P_{\dcpprogram_g}
         & =
        \int \left( \prod_{i =1}^n \mathbf{1}_{[\mu_i=b_i]}(\samplefunction(\randomvariableset)) \right) \differential{\probabilitymeasure_{\randomvariableset}}                                                                                                         \\
         & =
        \int \left( \prod_{i: b_i=\bot} \mathbf{1}_{[\mu_i=b_i]}(\samplefunction(\randomvariableset)) \right)  \left( \prod_{i: b_i=\top} \mathbf{1}_{[\mu_i=b_i]}(\samplefunction(\randomvariableset)) \right)  \differential{\probabilitymeasure_{\randomvariableset}} \\
         & =
        \int \left( \prod_{i: b_i=\bot} \ive{\neg c_i(vars(\mu_i))} \right)  \left( \prod_{i: b_i=\top}  \ive{ c_i(vars(\mu_i))}  \right)  \differential{\probabilitymeasure_{\randomvariableset}}
        \label{eq:program_prob_label_equi}
    \end{align}
    Turning our attention now to the expected value of $\alpha(\varphi)$ we have:
    \begin{align}
        \E_{\randomvariableset \sim  \dcpprogram_g} [\alpha( \varphi )]
        =
        \int \alpha \left( \bigwedge_{\ell_i \in \varphi} \ell_i \right) \differential{P_\randomvariableset}
        =
        \int  \left( \prod_{\ell_i \in \varphi} \alpha \left(  \ell_i \right) \right)  \differential{P_\randomvariableset}
    \end{align}
    The literals $\ell_i \in \varphi$ fall into four groups: atoms whose predicate is a comparison and that are true in $\varphi$ (denoted by $CA^+(\varphi)$), non-comparison atoms that are true in $\varphi$ (denoted $NA^+(\varphi)$), and similarly the atoms that are false in $\varphi$ (denoted by $CA^-(\varphi)$ and $NA^-(\varphi)$). This yields:
    \begin{align}
         & \E_{\randomvariableset \sim  \dcpprogram_g} [\alpha( \varphi )] \\
         & =
        \int
        \left( \prod_{\ell_i \in CA^+(\varphi)} \alpha (  \ell_i ) \right)
        \left( \prod_{\ell_i \in CA^-(\varphi)} \alpha (  \neg \ell_i ) \right)
        \left( \prod_{\ell_i \in NA^+(\varphi)} \alpha (  \ell_i ) \right)
        \left( \prod_{\ell_i \in NA^-(\varphi)} \alpha (  \neg \ell_i ) \right)
        \nonumber
        \differential{P_\randomvariableset}
    \end{align}
    Plugging in the definition of the labeling function the last two products reduce to $1$ and we obtain for the remaining expression:
    \begin{align}
        \E_{\randomvariableset \sim  \dcpprogram_g} [\alpha( \varphi )]
        =
        \int
        \left( \prod_{i: \ell_i \in CA^+(\varphi)} \ive{  c_i(vars(\ell_i)) } \right)
        \left( \prod_{i: \ell_i \in CA^-(\varphi)} \ive{  \neg c_i(vars(\ell_i)) } \right)
        \differential{P_\randomvariableset}
        \label{eq:formula_prob_lable_equi}
    \end{align}
    Identifying now the set $\{ i: \ell_i \in CA^+(\varphi) \}$ with the set $\{i : \mu_i=\top   \} $ and the set  $\{ i: \ell_i \in CA^-(\varphi) \}$ with the set $\{i : \mu_i=\bot   \} $ proves the theorem, as this equates Equation~\ref{eq:program_prob_label_equi} and Equation~\ref{eq:formula_prob_lable_equi}
\end{proof}

\subsection{Proof of Proposition~\ref{prop:mcapproxconditional}}
\label{app:proof:mcapproxconditional}

\mcapproxconditional*

\begin{proof}
    First we write the conditional probability as a ratio of expected values invoking Theorem~\ref{thm:inference-by-expectation}, on which we then use Defintion~\ref{def:label_bool_formula}:
    \begin{align}
        P_\dcpprogram(\mu=q\mid\evidenceset = e)
         & = \frac{\E_{\randomvariableset \sim  \dcpprogram_g} [\alpha( \phi \wedge \phi_q)] }{\E_{\randomvariableset \sim  \dcpprogram_g} [\alpha( \phi )] }
        \\
         & =
        \frac{
            \E_{\randomvariableset \sim  \dcpprogram_g} \left[ \sum_{\varphi\in ENUM(\phi \land \phi_{q}) } \prod_{\ell \in \varphi} \alpha(\ell) \right]
        }
        {
            \E_{\randomvariableset \sim  \dcpprogram_g} \left[ \sum_{\varphi\in ENUM(\phi) } \prod_{\ell \in \varphi} \alpha(\ell)  \right]
        }
        \\
         & =
        \frac{
            \E_{\randomvariableset \sim  \dcpprogram_g} \left[ \sum_{\varphi\in ENUM(\phi \land \phi_{q}) } \alpha(\varphi) \right]
        }
        {
            \E_{\randomvariableset \sim  \dcpprogram_g} \left[ \sum_{\varphi\in ENUM(\phi) } \alpha(\varphi)  \right]
        }
    \end{align}

    We can now express the conditional probability in terms of the sampled values $\mathcal{S}$:
    \begin{align}
        P_\dcpprogram(\mu=q\mid\evidenceset=e)
           & =
        \frac{
        \lim_{ \lvert \mathcal{S} \rvert \rightarrow \infty} \nicefrac{1}{\lvert \mathcal{S} \rvert} \sum_{i=1}^{\lvert \mathcal{S} \rvert}  \sum_{\varphi\in ENUM(\phi \land \phi{q}) } \alpha^{(i)}(\varphi)
        }
        {
        \lim_{\lvert \mathcal{S} \rvert \rightarrow \infty} \nicefrac{1}{\lvert \mathcal{S} \rvert} \sum_{i=1}^{\lvert \mathcal{S} \rvert} \sum_{\varphi\in ENUM(\phi) } \alpha^{(i)}(\varphi)
        }                                                                   \\
           & \approx
        \frac{
        \sum_{i=1}^{\lvert \mathcal{S} \rvert}  \sum_{\varphi\in ENUM(\phi \land \phi_{q}) } \alpha^{(i)}(\varphi)
        }
        {
        \sum_{i=1}^{\lvert \mathcal{S} \rvert} \sum_{\varphi\in ENUM(\phi) } \alpha^{(i)}(\varphi)
        }, & \quad \lvert \mathcal{S} \rvert<\infty \label{eq:mc_prob_cond}
    \end{align}
\end{proof}

\subsection{Proof of Proposition~\ref{prop:alw_consistency}}
\label{app:proof:alw_consistency}

\alwconsistency*
\begin{proof}
    First we manipulate the expected value on the left hand side of Equation~\ref{eq:ALW}:
    \begin{align}
         & \E \left[ \sum_{\varphi \in ENUM(\phi)} \prod_{\ell \in \varphi}  \alpha \left(\ell \right) \bigg| \mathcal{S} \right]                                                           \\
         & =
        \lim_{ \lvert \mathcal{S} \rvert\rightarrow\infty} \sum_{i=1}^{ \lvert \mathcal{S} \rvert}  \sum_{\varphi \in ENUM(\phi)} \prod_{\ell \in \varphi}  \alpha^{(i)} \left(\ell \right) \\
         & =
        \lim_{ \lvert \mathcal{S} \rvert \rightarrow\infty} \sum_{i=1}^{ \lvert \mathcal{S} \rvert}  \sum_{\varphi \in ENUM(\phi)}
        \left(\prod_{\ell \in \varphi \setminus DI(\varphi)}  \alpha^{(i)} \left(\ell \right) \prod_{\ell \in DI(\varphi)}  \alpha^{(i)} \left(\ell \right)   \right)
    \end{align}
    As the samples are ancestral samples, they satisfy by construction the delta invervals appearing in the second product. This means that $\prod_{\ell \in DI(\varphi)}  \alpha^{(i)} \left(\ell \right) =1$ and that we can write the expected value in function of non delta interval atoms only:
    \begin{align}
        \E \left[ \sum_{\varphi\in ENUM(\phi)} \prod_{\ell \in \varphi}  \alpha \left(\ell \right) \bigg| \mathcal{S} \right]
        =
        \E \left[ \underbrace{\sum_{\varphi\in ENUM(\phi)} \prod_{\ell \in \varphi \setminus DI(\varphi)}  \alpha \left(\ell \right)}_{\coloneqq f(\phi)} \bigg| \mathcal{S} \right]  \label{eq:ALW_exp_simplified}
    \end{align}
    Let us now manipulate the expression in the numerator on the right hand side of Equation~\ref{eq:ALW}:
    \begin{align}
         & \bigoplus_{i=1}^{ \lvert \mathcal{S} \rvert}  \bigoplus_{\varphi \in ENUM(\phi)} \bigotimes_{\ell \in \varphi}  \alpha_{IALW}^{(i)} \left(\ell \right) \\
         & =
        \bigoplus_{i=1}^{ \lvert \mathcal{S} \rvert}  \bigoplus_{\varphi \in ENUM(\phi)}
        \underbrace{\left(  \bigotimes_{\ell \in \varphi \setminus DI(\varphi)}  \alpha_{IALW}^{(i)} \left(\ell \right) \right) }_{\left( r_\varphi^{(i)} , 0  \right)}  \otimes
        \underbrace{\left( \bigotimes_{\ell \in DI(\varphi)}  \alpha_{IALW}^{(i)} \left(\ell \right) \right) }_{\left( t_\varphi^{(i)} , m_\varphi^{(i)}  \right)} \label{eq:alw_proof_intermediate_1}
    \end{align}
    The expressions $\left( r_\varphi^{(i)} , 0  \right)$ and $\left( t_\varphi^{(i)} , m_\varphi^{(i)}  \right)$ denote infinitesimal numbers. Note how only the latter of the two picks up a non-zero second part.

    From the definition of the addition of two infinitesimal numbers we can see that only those infinitesimal numbers with the smallest integer in the second part {\em survive} the addition. This also means that in Equation~\ref{eq:alw_proof_intermediate_1} only those terms that have the smallest integer in their second part among all terms will contribute. We denote this smallest integer by:
    \begin{align}
        m^* =  \min_{ \substack{i\in \{ 1, \dots,  \lvert \mathcal{S} \rvert \} \\ \varphi\in ENUM(\phi) }} m^{(i)}_\varphi
    \end{align}
    We rewrite Equation~\ref{eq:alw_proof_intermediate_1} in function of $m^*$:
    \begin{align}
         & \bigoplus_{i=1}^{ \lvert \mathcal{S} \rvert}  \bigoplus_{\varphi \in ENUM(\phi)}
        \left(  \ive{m_{\varphi}^{(i)}{=} m^* } , 0  \right) \otimes
        \left( r_\varphi^{(i)} , 0  \right) \otimes
        \left( t_\varphi^{(i)} , m^*  \right)                                                                              \\
         & =
        \bigoplus_{i=1}^{ \lvert \mathcal{S} \rvert}  \bigoplus_{\varphi \in ENUM(\phi)}
        \left( \ive{m_{\varphi}^{(i)}{=} m^* } r_\varphi^{(i)}  t_\varphi^{(i)} , 0  \right) \otimes \left( 1, m^* \right) \\
         & =
        \left( \sum_{i=1}^{ \lvert \mathcal{S} \rvert}  \sum_{\varphi \in ENUM(\phi)}  \ive{m_{\varphi}^{(i)}{=} m^* } r_\varphi^{(i)}  t_\varphi^{(i)} , 0   \right)  \otimes \left( 1, m^* \right) \label{eq:ALW_num_simplified}
    \end{align}
    Similarly we get for the denominator in Equation~\ref{eq:ALW}:
    \begin{align}
         & \bigoplus_{i=1}^{ \lvert \mathcal{S} \rvert}  \bigoplus_{\varphi \in ENUM(\phi)} \bigotimes_{\ell \in  DI(\varphi)}  \alpha_{IALW}^{(i)} \left(\ell \right) \\
         & =
        \left( \sum_{i=1}^{ \lvert \mathcal{S} \rvert}  \sum_{\varphi\in ENUM(\phi)}  \ive{m_{\varphi}^{(i)}{=} m^* }  t_\varphi^{(i)} , 0   \right)  \otimes \left( 1, m^* \right)  \label{eq:ALW_den_simplified}
    \end{align}
    We can now plug Equation~\ref{eq:ALW_exp_simplified}, Equation~\ref{eq:ALW_num_simplified} and Equation~\ref{eq:ALW_den_simplified} back into Equation~\ref{eq:ALW} and obtain:

    \begin{align}
         & \phantom{{}\Leftrightarrow{}}  \left( \E \left[  f(\phi) | \mathcal{S} \right],  0 \right)
        =
        \left(
        \frac
        { \sum_{i=1}^{ \lvert \mathcal{S} \rvert} \sum_{\varphi \in ENUM(\phi)}  \ive{m_{\varphi}^{(i)}{=} m^* } r_\varphi^{(i)}  t_\varphi^{(i)} }
        { \sum_{i=1}^{ \lvert \mathcal{S} \rvert}  \sum_{\varphi\in ENUM(\phi)}  \ive{m_{\varphi}^{(i)}{=} m^* }  t_\varphi^{(i)}   }
        , 0
        \right) \quad ( \lvert \mathcal{S} \rvert\rightarrow \infty)                                  \\
         & \Leftrightarrow \E \left[  f(\phi) | \mathcal{S} \right]
        =
        \frac
        { \sum_{i=1}^{ \lvert \mathcal{S} \rvert}  \sum_{\varphi \in ENUM(\phi)}  \ive{m_{\varphi}^{(i)}{=} m^* } r_\varphi^{(i)}  t_\varphi^{(i)} }
        { \sum_{i=1}^{ \lvert \mathcal{S} \rvert}  \sum_{\varphi\in ENUM(\phi)}  \ive{m_{\varphi}^{(i)}{=} m^* }  t_\varphi^{(i)}   } \label{eq:llw}
    \end{align}
    We realize that $r_\varphi^{(i)}$ is actually $f(\phi)$ evaluated at the $i$-th sample at the instantiation $\varphi$ and evoke~\citep[Theorem 4.1]{wu2018discrete} to prove Equation~\ref{eq:llw}, which also finishes this proof.
\end{proof}

\subsection{Proof of Proposition~\ref{prop:alwapproximation}}
\label{app:proof:alwapproximation}

\alwapproximation*

\begin{proof}
    We start the proof by invoking Theorem~\ref{thm:inference-by-expectation}, which expresses the conditional probability as a ratio of expectations. In the numerator and the denominator we then write the label of the propositional logic formulas as the sum of the labels of the respective possible worlds.
    \begin{align}
        P_\dcpprogram(\mu=q|\evidenceset=e)  \label{eq:cond_amc_1}
         & =\frac{\E_{\randomvariableset \sim  \dcpprogram_g} [\alpha( \phi \wedge \phi_q)] }{\E_{\randomvariableset \sim  \dcpprogram_g} [\alpha( \phi )] } \\
         & =
        \frac
        {
            \E_{\randomvariableset \sim  \dcpprogram_g} \left[ \sum_{\varphi\in ENUM(\phi \land \phi_{q e}) } \alpha(\varphi) \right]
        }
        {
            \E_{\randomvariableset\sim  \dcpprogram_g} \left[ \sum_{\varphi\in ENUM(\phi) } \alpha(\varphi)  \right]
        }  \label{eq:cond_amc_2}
    \end{align}
    Next, we approximate the expectation using a set of ancestral samples $\mathcal{S}$, followed by pulling out the query from the summation index in the numerator:
    \begin{align}
         & \frac
        {
            \E_{\randomvariableset \sim  \dcpprogram_g} \left[ \sum_{\varphi\in ENUM(\phi \land \phi_{q}) } \alpha(\varphi) \right]
        }
        {
            \E_{\randomvariableset \sim  \dcpprogram_g} \left[ \sum_{\varphi\in ENUM(\phi) } \alpha(\varphi)  \right]
        }
        \\
         & \approx
        \frac
        {
            \E \left[ \sum_{\varphi\in ENUM(\phi \land \phi_{q})   } \alpha(\varphi) |  \mathcal{S}  \right]
        }
        {
            \E \left[ \sum_{\varphi\in ENUM(\phi) } \alpha(\varphi) |  \mathcal{S}   \right]
        }  \label{eq:cond_amc_approx}
        \\
         & \approx
        \frac{
            \E \left[ \sum_{\varphi\in ENUM(\phi)   } \ive{\varphi \models \phi_q} \alpha(\varphi) |  \mathcal{S}  \right]
        }
        {
            \E \left[ \sum_{\varphi\in ENUM(\phi) } \alpha(\varphi) |  \mathcal{S}   \right]
        }
        \label{eq:cond_amc_expanded}
    \end{align}
    We now rewrite the fraction of two real numbers in Equation~\ref{eq:cond_amc_expanded} as the fraction of two infinitesimal numbers and plug in the definition of the infinitesimal algebraic likelihood weight (cf. Definition~\ref{def:alw}):
    \begin{align}
         & \frac{
            \E \left[ \sum_{\varphi\in ENUM(\phi)   } \ive{\varphi \models \phi_q} \alpha(\varphi) |  \mathcal{S}  \right]
        }
        {
            \E \left[ \sum_{\varphi\in ENUM(\phi) } \alpha(\varphi) |  \mathcal{S}   \right]
        }
        \\
         & =
        \frac{
            \left(\E \left[ \sum_{\varphi\in ENUM(\phi)   } \ive{\varphi \models \phi_q} \alpha(\varphi) |  \mathcal{S}  \right], 0 \right)
        }
        {
            \left( \E \left[ \sum_{\varphi\in ENUM(\phi) } \alpha(\varphi) |  \mathcal{S}   \right],0 \right)
        }          \\
         & \approx
        \frac
        {
        \displaystyle \bigoplus_{i=1}^{\lvert \mathcal{S} \rvert}  \bigoplus_{\varphi\in ENUM(\phi)} \ive{\varphi \models \phi_q} \bigotimes_{\ell \in \varphi}   \alpha_{IALW}^{(i)} \left(\ell \right)}
        {
        \displaystyle  \cancel{ \bigoplus_{i=1}^{\lvert \mathcal{S} \rvert}  \bigoplus_{\varphi\in ENUM(\phi)} \bigotimes_{\ell \in  DI(\varphi)}  \alpha_{IALW}^{(i)} \left(\ell \right) }}
        \otimes
        \frac
        {
        \displaystyle \cancel{\bigoplus_{i=1}^{\lvert \mathcal{S} \rvert}  \bigoplus_{\varphi\in ENUM(\phi)} \bigotimes_{\ell \in  DI(\varphi)}  \alpha_{IALW}^{(i)} \left(\ell \right) } }
        {
        \displaystyle \bigoplus_{i=1}^{\lvert \mathcal{S} \rvert}  \bigoplus_{\varphi\in ENUM(\phi)} \bigotimes_{\ell \in \varphi}  \alpha_{IALW}^{(i)} \left(\ell \right)
        } \label{eq:cond_amc_canel}
    \end{align}

    In the last line the first factor corresponds to the numerator of the previous equation and the second factor corresponds to the reciprocal of the denominator.
    Note that the consistency of the infinitesimal algebraic likelihood weight of the numerator (first factor) is guaranteed by defining a new labeling function $\alpha^q(\varphi)\coloneqq \ive{\varphi \models \phi_q} \alpha(\varphi)$ and evoking Proposition~\ref{prop:alw_consistency} with $\alpha^q$.

    Finally, we push the expression $\ive{\varphi \models \phi_q}$ in the numerator back into the index of the summation ($\oplus$), which yields the following expression:
    \begin{align}
        \frac
        { \bigoplus_{i=1}^{\lvert \mathcal{S} \rvert}  \bigoplus_{\varphi\in ENUM(\phi \land \phi_{q})} \bigotimes_{\ell \in \varphi}  \alpha_{ALWI}^{(i)} \left(\ell \right)}
        { \bigoplus_{i=1}^{\lvert \mathcal{S} \rvert}  \bigoplus_{\varphi\in ENUM(\phi)} \bigotimes_{\ell \in \varphi}  \alpha_{IALW}^{(i)} \left(\ell \right)}
        \label{eq:cond_alw_last}
    \end{align}
    which proves the proposition. \end{proof}

\subsection{Proof of Proposition~\ref{prop:alwonddnnf}}
\label{app:proof:alwonddnnf}

\alwonddnnf*
\begin{proof}
    Algorithm~\ref{alg:prob_via_alw_kc} first compiles both propositional formulas into equivalent d-DNNF representations, cf. Lines~\ref{alg:prob_via_alw_kc:kc_qe} and~\ref{alg:prob_via_alw_kc:kc_e}.
    In Lines~\ref{alg:prob_via_alw_kc:alw_qe} and~\ref{alg:prob_via_alw_kc:alw_e} it then computes the (unnormalized) infinitesimal algebraic likelihood weight for both formulas by calling Algorithm~\ref{alg:unormalize_alw}. In other words, we compute the numerator and denominator in Equation~\ref{eq:cond_alw_last}. We observe that Algorithm~\ref{alg:unormalize_alw} evaluates a given d-DNNF formula for each conditioned topological sample using the \texttt{Eval} function, which evaluates a d-DDNF formula given a labeling function, cf. Algorithm~\ref{alg:eval}~\citep{kimmig2017algebraic}. The correctness of Algorithm~\ref{alg:prob_via_alw_kc} now hinges on the correctness of the \texttt{Eval} function, which was proven by~\citet{kimmig2017algebraic} for the evaluation of a d-DNNF formula using a semiring and labeling function pair that adheres to the properties described in Lemmas~\ref{lem:non_idem} to~\ref{lem:non_cons}. Effectively, Algorithm~\ref{alg:unormalize_alw} correctly computes the algebraic model count for each ancestral sample, adds up the results, and returns the unnormalized algebraic model count to Algorithm~\ref{alg:prob_via_alw_kc}. Line~\ref{alg:prob_via_alw_kc:conditional} finally return the ratio of the two unnormalized algebraic likelihood weights, which corresponds to the conditional probability $P_\dcpprogram(\mu=q|\evidenceset=e)$, as proven in Equations~\ref{eq:cond_amc_1} to~\ref{eq:cond_alw_last}.
\end{proof}